\documentclass{article}

\PassOptionsToPackage{numbers, compress}{natbib}
 \usepackage[preprint]{neurips_2026}

\usepackage[utf8]{inputenc} % allow utf-8 input
\usepackage[T1]{fontenc}    % use 8-bit T1 fonts
\usepackage{hyperref}       % hyperlinks
\usepackage{url}            % simple URL typesetting
\usepackage{booktabs}       % professional-quality tables
\usepackage{amsfonts}       % blackboard math symbols
\usepackage{nicefrac}       % compact symbols for 1/2, etc.
\usepackage{microtype}      % microtypography
\usepackage{xcolor}         % colors
\usepackage{amsmath}
\usepackage{amsthm}
\usepackage{mathtools}
\usepackage{bm}
\usepackage{tikz}
\usetikzlibrary{positioning, arrows.meta}
\usepackage{booktabs}
\usepackage{tabularx}
\usepackage{array}
\usepackage{amsmath}
\usepackage{float}
\usepackage{subcaption}
\usepackage[normalem]{ulem}
\newcolumntype{Y}{>{\raggedright\arraybackslash}X}

\newtheorem{definition}{Definition}

\newtheorem{proposition}{Proposition}
\newtheorem{theorem}{Theorem}

\newtheorem{remark}{Remark}

\title{A Hilbert-Valued Functional Decomposition Framework for Explaining Time-Dependent Outputs}

\author{%
  Sophie Hanna Langbein\textsuperscript{1,2} \qquad
  Niklas Koenen\textsuperscript{1} \qquad
  Marvin N. Wright\textsuperscript{1,2} \qquad
  Julia Herbinger\textsuperscript{1} \\[0.6em]
  \textsuperscript{1}Leibniz Institute for Prevention Research and Epidemiology -- BIPS, Bremen, Germany \\
  \textsuperscript{2}Faculty of Mathematics and Computer Science, University of Bremen, Germany \\
}

\begin{document}

\maketitle

\begin{abstract}
Feature-based explanations quantify features' influence on model predictions, but are primarily designed for scalar outputs. In many applications, however, outputs are functional or multivariate, such as time-dependent trajectories in demand forecasting. Consequently, existing approaches typically explain each output location independently, ignoring dependencies across the output components. We address this limitation by developing a unified framework for feature-based explanations of time-dependent outputs. Specifically, we generalize functional decomposition to Hilbert-valued prediction functions and extend an existing feature-based explanation framework to this setting. Our framework introduces kernel-based output representations that enable time-dependency-aware explanations at multiple levels of temporal granularity, including \emph{time-specific}, \emph{time-resolved}, and \emph{time-aggregated}, while providing a unified view in which existing methods arise as special cases. We validate our framework on synthetic and real-world data, including intraday financial market volatility prediction and energy demand forecasting.
\end{abstract}

\section{Introduction}\label{sec:intro}

% XAI is important --> a lot of development including unifying framework
Machine learning models achieve strong predictive performance across many applications, but are widely considered black boxes. However, in high-stakes settings such as healthcare, finance, or policy making, understanding how predictions are formed is crucial \citep{rudin2019stop,bracke2019machine,tjoa2021survey}. This has led to the development of explainable artificial intelligence (XAI), including a large class of feature-based methods that quantify how input features influence model predictions, either \emph{locally} (individual predictions) \citep{lundberg2017shap, goldstein2015peeking} or \emph{globally} (model behavior across the data distribution) \citep{breiman2001random, fisher2019all, covert2020sage}. 
Given the large number of such methods, recent work has focused on unifying perspectives that reveal connections between them \citep{covert2021explaining,lundstrom2023unifying, deng2024unifying, fumagalli2025unifying}. In particular, the framework proposed by \cite{fumagalli2025unifying} formalizes feature-based explanations through a common set of operators, providing a general view on existing approaches. 

% all just scalar
However, these methods and their unifying formulations are predominantly developed for \emph{scalar-valued outputs}, as encountered in standard regression or classification tasks. In many applications, models instead produce \emph{multivariate} or \emph{functional outputs}, such as time-dependent trajectories \citep{salinas2020deepar,choi2016doctor}. In practice, XAI methods are routinely applied pointwise to functional outputs, explaining each output location in isolation \citep{vanzyl2024harnessing, neubauer2025explainable, gutierrez2024multioutput, huang2024improving, yang2022multivariate, donizy2022machine}. This has two limitations: it ignores \emph{dependencies across the output domain}, treating each output location as if it were unrelated to its neighbors, and it offers no principled mechanism for \emph{aggregating} feature effects \emph{across time}. The result is a gap between the natural structure of time-dependent outputs and the levels of temporal granularity at which they can be meaningfully explained. 

% the following example shows why this might be problematic
To illustrate these limitations, consider a patient in intensive care where a biomarker (e.g., blood lactate) is monitored over time. A model $F$ predicts a 24h trajectory (i.e., $F(\mathbf{x})$ is a function over $t$) (Fig.~\ref{fig:intro}, top right), where $X_1$ controls the recovery trend, and $X_2, X_3$ represent clinically meaningful, early shock ($t=10$h) and late deterioration ($t=18$h) events, respectively. A natural goal is to understand how each feature influences the trajectory. Pointwise explanations attribute feature effects independently at each $t$. Yet, the influence of $X_2$ is not confined to $t=10$h, but spreads over a temporal neighborhood, and clinicians may care less about attribution at a single moment than about the aggregate influence over a clinically meaningful phase. Pointwise methods support neither: they ignore dependencies across the output domain and lack principled basis for aggregation.

This highlights the need for explanation methods that account for dependencies across time in the output and support different levels of aggregation. To address this, we propose the \emph{Hilbert-valued explanation framework}, which extends feature-based explanations to multivariate and functional outcomes while taking into account the outcomes' dependence structure through kernel-based aggregation. Fig.~\ref{fig:intro} illustrates this contrast: an identity kernel produces isolated pointwise attributions, while a time-dependency-aware kernel (correlation kernel) yields phase-aware effects examined at a specific time, across the trajectory, or aggregated into a single importance.

\textbf{Contributions.}

\textbf{1. Hilbert-valued functional decomposition:} We introduce the Hilbert-valued functional decomposition ($\mathcal{H}$-FD), which extends classical functional decomposition to time-dependent model outputs (multivariate or functional), yielding pure feature effects in a Hilbert space.

\textbf{2. Hilbert-valued explanation framework:} We extend the unifying framework of \cite{fumagalli2025unifying} to time-dependent outputs, establishing a unified class of feature-based explanations that capture both dependencies across the output domain and different levels of aggregation. Classical pointwise explanations and time-dependent Sobol indices arise as special cases.

\textbf{3. Practical guidance:} We provide guidance on selecting appropriate kernels based on the desired interpretation and demonstrate their impact on the resulting explanations in real-world applications.

\begin{figure}[ht]
  \centering
  \includegraphics[width=\textwidth]{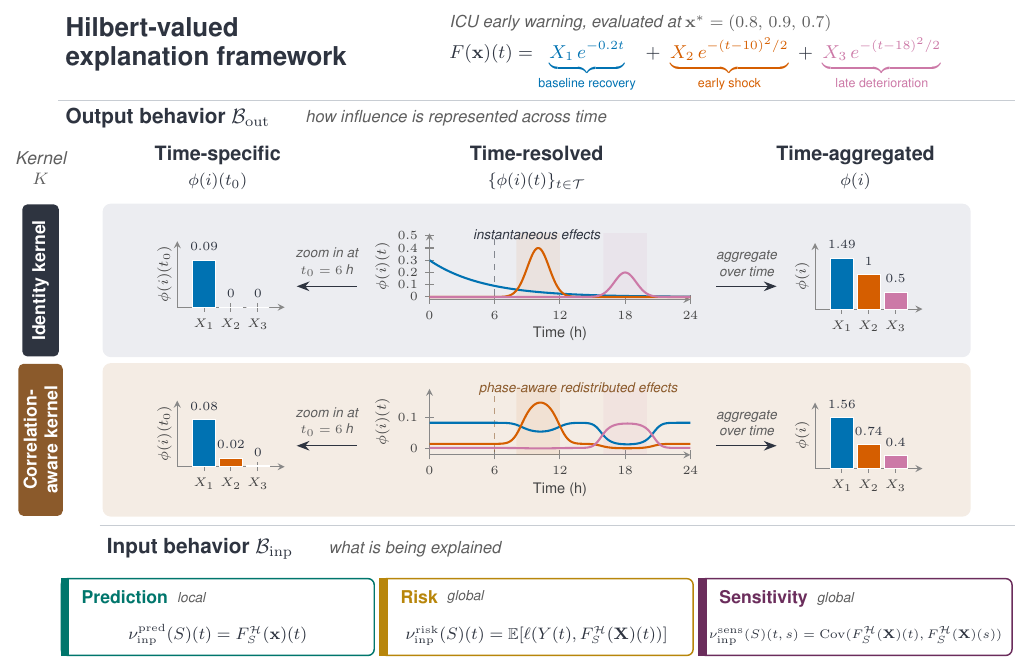}
  \caption{\textbf{Hilbert-valued explanation framework.} Illustration on an ICU early-warning toy model (\emph{top}). Two orthogonal axes structure time-dependent attribution. The \emph{output behavior} (middle) controls \emph{how} influence is represented across time, parameterized by a kernel $K$: the identity kernel produces instantaneous effects $\phi(i)(t)$ (\emph{middle, row~1}), a time-dependency-aware kernel (correlation kernel) redistributes them phase-awarely (\emph{middle, row~2}). \emph{Resolved} trajectories can be sampled at \emph{specific} time points $t_0$ or \emph{aggregated} into a scalar $\phi(i)$. The \emph{input behavior} (\emph{bottom}) specifies \emph{what} is being explained: prediction, risk, or sensitivity.}
  \label{fig:intro}
\end{figure}

\section{Related Work}

\textbf{Feature-based explanations for time-dependent outputs} largely adapt prediction-level methods to forecasting settings by treating multi-step outputs independently. TsSHAP \citep{raykar2023tsshap} explains individual forecast time points, ShapTime \citep{zhang2023shaptime} focuses on aggregated quantities, and PAX-TS \citep{kreuzer2025pax} combines both perspectives. Beyond forecasting, SurvSHAP(t) \citep{krzyzinski2023survshap} explains time-dependent survival functions pointwise, with SurvSHAP-IQ \citep{langbein2026functional} extending this to interaction effects. These methods remain confined to Shapley-style, prediction-level explanations operating on individual output locations or predefined aggregations, lacking a unified treatment of dependencies across the output domain.

%\textbf{Sensitivity analysis for multivariate and functional outputs.} 
\textbf{Outside the XAI literature, sensitivity analysis} has long addressed multivariate and functional outputs. Extensions of \emph{Sobol indices} define importance measures by aggregating variance across the output domain \citep{campbell2006sensitivity,lamboni2011multivariate,gamboa2014sensitivity}, while also accounting for temporal dependence through covariance-based formulations \citep{shi2018cross,alexanderian2020variance}. However, these approaches are restricted to variance-based notions of importance and have not been introduced from a feature-based XAI perspective. 
%Outside the XAI literature, sensitivity analysis has long addressed multivariate and functional outputs. \cite{campbell2006sensitivity} first extended Sobol indices to functional outputs via PCA decomposition, an approach later formalized by \cite{lamboni2011multivariate} through \emph{generalized sensitivity indices} aggregating variance across the output domain. \cite{gamboa2014sensitivity} established consistent estimators for multidimensional and functional Sobol indices. \cite{shi2018cross} and \cite{alexanderian2020variance} both incorporate the output's temporal dependence structure, through cross-covariance indices and Karhunen–Loève decomposition respectively. However, these methods are restricted to variance-based behavior and do not connect to the cooperative game-theoretic perspective underpinning modern feature-based XAI.

\section{Background}
Work on time-dependent outputs mainly arises in sensitivity analysis with limited links to feature-based explanations; we introduce the necessary concepts from both areas.

\textbf{General Notation.}
Consider a supervised learning setting with a function 
$F^\mathcal{H}: \mathcal{X} \to \mathcal{H}$, where $\mathcal{X} \subseteq \mathbb{R}^p$ and $\mathcal{H}$ is a Hilbert space with inner product $\langle \cdot, \cdot \rangle_\mathcal{H}$. This covers \emph{multivariate outputs} $F^\mathcal{H}(\mathbf{x}) \in \mathbb{R}^T$ ($T \in \mathbb{N}$) and \emph{functional outputs}, $F^\mathcal{H}(\mathbf{x}) \in L^2(\mathcal{T)}$. Scalar-valued outputs are written $F: \mathcal{X} \to \mathbb{R}$. For $t \in \mathcal{T}$, define $F_t(\mathbf{x}) := F^\mathcal{H}(\mathbf{x})(t) \in \mathbb{R}$.
Let $\mathbf{X} = (\mathbf{X}_1, \dots, \mathbf{X}_p)$ denote the input random vector.
For a subset $S \subseteq \{1,\dots,p\}$, let $\mathbf{X}_S$ denote subvectors and by $\mathbf{X}_{-S}$ complements. %Observations are $\mathbf{x}^{(i)} = (x_1^{(i)}, \dots, x_p^{(i)})$, $i=1,\dots,n$. 
While all results apply to general Hilbert-valued outputs, we focus on time-dependent trajectories.

\subsection{Sensitivity Analysis}

Sensitivity analysis attributes output variability to input features and their interactions via functional decomposition, inducing a variance decomposition that underlies closed and total Sobol indices.

\textbf{Functional and variance decomposition.}\label{sec:variance_decomposition} Let $F:\mathcal{X}\to\mathbb{R}$ denote a prediction function and let $P_{\mathbf X}$ be a reference distribution over $\mathcal{X}$. Functional decomposition represents $F$ as a sum of pure effects $f_S$ over all feature subsets $S\subseteq\{1,\dots,p\}$, where each $f_S$ captures effects beyond its subsets,
\begin{align}\label{eq:fun_decomposition}
    F(\mathbf{x}) = \sum_{S \subseteq \{1,\dots,p\}} f_S(\mathbf{x}), & \quad \text{with} & f_S(\mathbf{x}) := F_S(\mathbf{x}) - \sum_{L \subset S} f_L(\mathbf{x})
\end{align}
and $F_S(\mathbf{x}) = \mathbb{E}_{\mathbf{X}_{-S} \sim P_{\mathbf{X}_{-S}}}[F(\mathbf{X}) \mid \mathbf{X}_S = \mathbf{x}_S]$ denotes the marginal effect of features in $S$ \citep{hoeffding1948,sobol2001,hooker2004discovering}. The decomposition depends on a reference distribution, typically marginal or conditional \citep{fumagalli2025unifying}. Under feature independence and square-integrability of $F$, the components are orthogonal and induce the variance decomposition $\textstyle \mathbb{V}[F(\mathbf{X})] = \sum_{S \subseteq \{1,\dots,p\}} \mathbb{V}[f_S(\mathbf{X})]$. For time-dependent outputs, this decomposition applies pointwise via $F_t(\mathbf{x}) = F^\mathcal{H}(\mathbf{x})(t)$ with $f_L$ replaced by $f_{L,t}(\mathbf{x}) = f_L^\mathcal{H}(\mathbf{x})(t)$.

Based on this variance decomposition, Sobol indices quantify the contribution of feature subsets to the total variance. The \textbf{closed Sobol index} for a subset $S$ is defined as
\begin{equation}
    \xi_S
     =
    \frac{\mathbb{V}[F_S(\mathbf{X})]}
         {\mathbb{V}[F(\mathbf{X})]}
    =
    \frac{\sum_{L \subseteq S} \mathbb{V}[f_L(\mathbf{X})]}
         {\mathbb{V}[F(\mathbf{X})]},
\end{equation}
where the second equality holds under independent inputs \citep{sobol1993,sobol2001}. It captures all effects involving only features in $S$, while the \textbf{total Sobol index} additionally includes all interaction effects involving at least one feature in $S$ \citep{saltelli2008global}. %For time-dependent outputs, both indices are defined pointwise \citep{lamboni2011multivariate,gamboa2014sensitivity}.

\textbf{Time-dependent Sobol indices.} For multivariate or functional outputs, variance is replaced by the covariance kernel $\mathrm{Cov}(F_t, F_s)$, leading to Sobol-type indices that account for dependencies across the output domain. A \textbf{time-specific} \emph{closed} Sobol index can then be defined as \citep{shi2018cross}:
\begin{equation}\label{eq:sobol_time}
    \xi_S^{\mathrm{spec}}(t)
    =
    \frac{
    \int_{\mathcal{T}} 
    \mathrm{Cov}\!\big(F_{S,t}(\mathbf{X}), F_{S,s}(\mathbf{X})\big)\, ds
    }{
    \int_{\mathcal{T}} 
    \mathrm{Cov}\!\big(F_t(\mathbf{X}), F_s(\mathbf{X})\big)\, ds
    },
\end{equation}
accounting for time-dependencies relative to a fixed $t$. The corresponding \emph{total} index is obtained analogously by aggregating all effects involving at least one feature in $S$. Integrating these covariance contributions over time, yields scalar \textbf{time-aggregated} closed and total Sobol indices \citep{alexanderian2020variance}. Neglecting cross-time covariance reduces it to a sum of marginal variances.

\subsection{Feature-based Explanations}\label{sec:unifying_framework}
% define unifying framework
Feature-based explanations quantify how features influence model behavior. Following \cite{fumagalli2025unifying, herbinger2026granite}, they can be expressed as compositions of \emph{masking} ($\mathcal{M}$), \emph{behavior} ($\mathcal{B}$), and \emph{interaction} ($\mathcal{I}$) operators.

\textbf{Masking operator.} The masking operator  $\mathcal{M}$ constructs a predictor depending only on features in $S$,
$$
F_S(\mathbf{x}) := (\mathcal{M} F)(\mathbf{x},S),
$$
by removing (masking) features in $-S$ using a reference distribution $P_{X_{-S}}$ (e.g., marginal or conditional). Notably, $F_S$ corresponds to the marginal effect used in Eq.~\eqref{eq:fun_decomposition}.

\textbf{Behavior operator.} The behavior operator $\mathcal{B}$ maps masked predictors to a set function $\nu$
$$
\nu(S) := (\mathcal{B}(\mathcal{M}F))(S),
$$
which defines the quantity of interest, yielding local predictions $F_S(\mathbf{x})$ or global quantities such as prediction variance $\mathrm{Var}(F_S(\mathbf{X}))$ or expected loss $\mathbb{E}[\ell(Y, F_S(\mathbf{X}))]$.

\textbf{Interaction operator.} The interaction operator $\mathcal{I}$ maps $\nu$ to feature-based explanations $\phi$
$$
\phi(j) := (\mathcal I(\mathcal B(\mathcal{M}F)))(j), \qquad j\in \{1,\dots,p\}.
$$
This operator quantifies the influence of individual features by accounting for higher-order effects encoded in $\nu$.  
This yields different notions of explanations: \emph{pure effects}, $\phi(j)=\nu(j)-\nu(\emptyset)$, quantify the standalone contribution of feature $j$ and do not include interactions; 
\emph{partial effects}, $\phi(j)=\frac{1}{p}\sum_{S \subseteq -j}  \binom{p-1}{|S|}^{-1}[\nu(S \cup j)-\nu(S)]$, correspond to Shapley values and distribute interactions across features \citep{shapley1951}; 
and \emph{full effects}, $\phi(j)=\nu(\{1,\dots, p\})-\nu(-j)$, compare the full model to one where feature $j$ is excluded, thereby removing all contributions involving $j$ and capturing its total effect including interactions;
Extensions to feature group explanations are defined in \cite{fumagalli2025unifying}.

This operator view recovers many established methods as special cases. Prediction-based behavior with partial effects yields Shapley values \citep{shapley1951}, pure effects recover (conditional) partial dependence (PD) functions \citep{friedman2001greedy}, and risk-based behavior with full effects yields (conditional) permutation feature importance (PFI) \citep{breiman2001random, fisher2019all, strobl2008conditional}. An overview is provided in Tab.~\ref{tab:method_mapping} with more details in \cite{fumagalli2025unifying}. Closed and total Sobol indices arise analogously via marginal masking, variance-based behavior, and pure or full effects, linking the work to classical sensitivity analysis. However, this formulation is defined for scalar outputs $F$ and does not directly extend to Hilbert-valued outputs $F^{\mathcal H}$.

\section{Methodology}\label{sec:method}

We extend feature-based explanations to Hilbert-valued outputs via a functional decomposition in $\mathcal{H}$.

\subsection{Hilbert-valued Functional Decomposition ($\mathcal{H}$-FD)}\label{sec:hfd}
We generalize the functional decomposition of Eq.~\eqref{eq:fun_decomposition} to Hilbert-valued outputs $F^\mathcal{H}(\mathbf{x})$. The decomposition is performed over the input space as before, while the resulting effects take values in a Hilbert space $\mathcal{H}$. 
Intuitively, each effect contributes not a single value but a \emph{function} (or vector) describing its influence across the output domain.

\begin{definition}[Hilbert-Valued Functional Decomposition]
The Hilbert-valued prediction function $F^\mathcal{H}(\mathbf{x})$ admits a decomposition into \textbf{Hilbert-valued pure effects} $f^\mathcal{H}_S$:
\begin{equation}
    F^\mathcal{H}(\mathbf{x}) = \sum_{S \subseteq \{1,\dots,p\}} f^\mathcal{H}_S(\mathbf{x}), 
    \quad 
    f^\mathcal{H}_S(\mathbf{x}) := F^\mathcal{H}_S(\mathbf{x}_S) - \sum_{L \subset S} f^\mathcal{H}_L(\mathbf{x}).
\end{equation}
Here, $F^\mathcal{H}_S$ denotes the Hilbert-valued marginal effect with respect to a reference distribution $P_{\mathbf{X}_{-S}}$:
\begin{equation}
    F^\mathcal{H}_S(\mathbf{x}_S) := \mathbb{E}_{\mathbf{X}_{-S} \sim P_{\mathbf{X}_{-S}}}[F^\mathcal{H}(\mathbf{X}) \mid \mathbf{X}_S = \mathbf{x}_S],\quad \text{where} \quad F^\mathcal{H}_\emptyset := \mathbb{E}[F^\mathcal{H}(\mathbf{X})]
\end{equation}
\end{definition}

Thus, $\mathcal{H}$-FD follows the scalar construction, but the effects $f^\mathcal{H}_S(\mathbf{x})$ are functions or vectors in $\mathcal{H}$. 

\subsection{Hilbert-Valued Explanation Framework}

Now, we extend the framework introduced in  Sec.~\ref{sec:unifying_framework} to Hilbert-valued outputs $F^\mathcal{H}$ based on the $\mathcal{H}$-FD. A detailed overview of operator choices is provided in App.~\ref{app:operators}.

The \textbf{masking operator} $\mathcal{M}$ remains unchanged and connects directly to the $\mathcal{H}$-FD: the masked predictor $F_S^\mathcal{H}(\mathbf{x}) := (\mathcal M F^\mathcal H)(\mathbf{x},S)$ corresponds to the marginal effect and is itself a function (or vector) over the output domain.

Since $F_S^\mathcal{H}$ is defined over the output domain, the \textbf{behavior operator} $\mathcal{B}$ must specify both \emph{what} is explained and \emph{how} it is represented across that domain, in particular with respect to \emph{dependencies between output locations} and the \emph{level of aggregation}. We therefore decompose it as
\[
\mathcal B(F_S^\mathcal{H}) := \mathcal B_{\mathrm{out}}\big(\mathcal B_{\mathrm{inp}}(F_S^\mathcal{H})\big) = (\mathcal B_{\mathrm{out}} \circ \mathcal B_{\mathrm{inp}})(F_S^\mathcal{H}),
\]
where $\mathcal B_{\mathrm{inp}}$ defines the quantity of interest in the \textbf{input space} (as in the scalar case), and $\mathcal B_{\mathrm{out}}$ specifies its representation over the \textbf{output domain}.

This yields the extended pipeline
\[
F^\mathcal H
\;\xrightarrow{\mathcal M}\;
F_S^\mathcal H
\;\xrightarrow{\mathcal B_{\mathrm{inp}}}\;
\nu_{\mathrm{inp}}(S)
\;\xrightarrow{\mathcal B_{\mathrm{out}}}\;
\nu(S)
\;\xrightarrow{\mathcal I}\;
\phi(j).
\]

\textbf{Input behavior operator $\mathcal B_{\mathrm{inp}}$.}
The input behavior operator maps masked predictors to Hilbert-valued quantities of interest,
\[
\nu_{\mathrm{inp}}(S) := (\mathcal B_{\mathrm{inp}}(\mathcal M F^\mathcal H))(S),
\]
analogous to the scalar case. Typical choices include local prediction-based behavior $\nu_{\mathrm{inp}}^{\mathrm{pred}}(S)(t)=F_S^\mathcal H(\mathbf x)(t)$, 
global risk-based behavior $\nu_{\mathrm{inp}}^{\mathrm{risk}}(S)(t)=\mathbb E[\ell(Y(t),F_S^\mathcal H(\mathbf X)(t))]$, 
and sensitivity-based behavior $\nu_{\mathrm{inp}}^{\mathrm{sens}}(S)(t,s)=\mathrm{Cov}(F_S^\mathcal H(\mathbf X)(t),F_S^\mathcal H(\mathbf X)(s))$. 
Depending on the choice, $\nu_{\mathrm{inp}}(S)$ is a function over $\mathcal{T}$ or a covariance surface over $\mathcal{T}\times\mathcal{T}$.

\textbf{Output behavior operator $\mathcal B_{\mathrm{out}}$.} The output behavior operator determines how Hilbert-valued quantities are aggregated across the time domain,
\begin{equation}\label{eq:output_behavior}
    \nu(S) := (\mathcal B_{\mathrm{out}}(\nu_{\mathrm{inp}}))(S),
    \quad
    (\mathcal B_{\mathrm{out}}^{\mathrm{spec}} \nu_{\mathrm{inp}})(S,\cdot)
    := \int_{\mathcal T} K(t,s)\, \nu_{\mathrm{inp}}(S)(\cdot)\,ds,
\end{equation}
where $K$ is a symmetric positive semi-definite kernel encoding dependencies across time. The time-specific operator $\mathcal B_{\mathrm{out}}^{\mathrm{spec}}$ defines the base construction, yielding a \textbf{time-specific} explanation at each $t$, which may incorporate information from other time points through $K$. Evaluating it for all $t \in \mathcal{T}$ gives \textbf{time-resolved} explanations, while aggregating over $\mathcal{T}$, i.e.,
$(\mathcal B_{\mathrm{out}}^{\mathrm{aggr}} \nu_{\mathrm{inp}})(S)
= \int_{\mathcal T} (\mathcal B_{\mathrm{out}}^{\mathrm{spec}} \nu_{\mathrm{inp}})(S,\cdot)\,dt$,
yields \textbf{time-aggregated} explanations.

For time-aggregated explanations, kernel aggregation admits an explicit integral representation, characterizing when feature rankings are invariant to the choice of kernel.

\begin{theorem}\label{thm:kernel_effect}
Let $\mathcal B_{\mathrm{out}}^{\mathrm{spec}}$ be defined as in Eq.~\eqref{eq:output_behavior}. Then, for any subset $S\subseteq\{1,\dots,p\}$, the time-aggregated effect with prediction-based input behavior $\nu_{\mathrm{inp}}(S)(t) = F_S^\mathcal{H}(\mathbf{x})(t)$ admits
\[
(\mathcal B_{\mathrm{out}}^{\mathrm{aggr}} \nu_{\mathrm{inp}})(S)
=
\int_{\mathcal T} w_K(s)\,\nu_{\mathrm{inp}}(S)(s)\,ds,
\qquad
w_K(s):=\int_{\mathcal T}K(t,s)\,dt.
\]
If $\mathcal B_{\mathrm{inp}}$ is \textbf{linear} in $F^\mathcal H$ and $\nu_{\mathrm{inp}}(\{i\})(s)\ge \nu_{\mathrm{inp}}(\{j\})(s)
\text{ for all } s\in\mathcal T$,
the corresponding time-aggregated pure effects \textbf{preserve the ordering} of features $i$ and $j$ for any kernel $K$ satisfying $K(t,s)\ge 0$. Analogous results hold for partial and full effects; see App.~\ref{app:proof_linearity}.
\end{theorem}

\begin{remark}
For \textbf{nonlinear} input behaviors such as sensitivity or risk, the ordering invariance in Thm.~\ref{thm:kernel_effect} does not generally hold. Although the same weighted-integral representation applies, different kernels may induce different feature rankings.
\end{remark}\label{rmk:kernel_effect}

Finally, note that the proposed framework preserves the additive decomposition: combining masking with input and output behavior operators always yields a valid set function over feature subsets, ensuring that feature-based explanations remain well-defined.

\begin{proposition}\label{thm:moebius}
Let 
$\nu(S) := (\mathcal B_{\mathrm{out}}(\mathcal B_{\mathrm{inp}}(\mathcal M F^\mathcal H)))(S)$.
Then the Möbius decomposition 
\[
\nu(\{1,\dots,p\})
=
\sum_{S\subseteq\{1,\dots,p\}}\sum_{L\subseteq S}(-1)^{|S|-|L|}\nu(L)
\]
holds \citep{rota1964foundations, grabisch1999axiomatic}. Consequently, time-specific, time-resolved, and time-aggregated explanations all admit a unique additive decomposition for fixed choices of $\mathcal M$, $\mathcal B_{\mathrm{inp}}$, and $\mathcal B_{\mathrm{out}}$.
\end{proposition}

\textbf{Interaction operator and resulting explanations.}
The interaction operator $\mathcal I$ maps the set function $\nu$ to feature-based explanations $\phi(j) := (\mathcal I(\nu))(j)$. As illustrated in Fig. \ref{fig:intro}, the resulting explanation depends on both behavior operators: $\mathcal B_{\mathrm{inp}}$ determines \emph{what} is explained, while $\mathcal B_{\mathrm{out}}$ determines \emph{how} feature influence is represented across time. This yields scalar-valued time-specific and time-aggregated explanations, as well as time-resolved explanations as functions (or vectors) over $\mathcal{T}$, with the kernel controlling whether temporal dependencies are ignored or incorporated.

\subsection{Unifying Existing Methods}

\textbf{Time point feature-based explanations.}
If $K = I$ (identity kernel), the output behavior $\mathcal B_{\mathrm{out}}$ ignores cross-time dependencies: time-specific and time-resolved explanations reduce to pointwise explanations, and time-aggregated explanations to simple averages over the output domain.
Within this setting, existing methods arise as special cases combining $K=I$ with prediction-based input behavior (see App.~\ref{app:mapping_existing}). In particular, TsSHAP corresponds to time-specific Shapley values. ShapTime operates on aggregated output quantities, yielding time-aggregated explanations, while PAX-TS provides time-specific and time-aggregated explanations. SurvSHAP(t) produces time-resolved explanations over the trajectory; SurvSHAP-IQ extends this to interaction effects via $\mathcal I$.

In summary, existing approaches are confined to prediction-based explanations that are either pointwise (time-specific/time-resolved) or based on predefined aggregations (time-aggregated), without explicitly modeling dependencies across the output domain. Our framework removes both restrictions: general kernels $K$ enable time-dependency-aware explanations and risk and sensitivity-based input behaviors extend the framework to global notions of importance.

\textbf{Time-dependent Sobol indices.}
Beyond prediction-based methods, our framework also recovers Sobol measures for Hilbert-valued outputs via sensitivity-based input behavior and a constant kernel.

\begin{theorem}\label{thm:sobol_indices}
Let
$\nu_{\mathrm{inp}}(S)(t,s)
=
\mathrm{Cov}\!\big(F_S^\mathcal H(\mathbf X)(t),F_S^\mathcal H(\mathbf X)(s)\big)$
and choose a constant kernel $K(t,s)=1$. Then the output behavior yields
\[
\nu(S)(t)=\int_{\mathcal T}\nu_{\mathrm{inp}}(S)(t,s)\,ds,
\qquad
\nu(S)=\int_{\mathcal T}\int_{\mathcal T}\nu_{\mathrm{inp}}(S)(t,s)\,dt\,ds.
\]
Under feature independence, the corresponding pure and full interaction operators $\mathcal{I}$ recover the closed and total time-specific and time-aggregated Sobol indices.
\end{theorem}

\section{Experiments}\label{sec:experiments}

We empirically evaluate the proposed framework on synthetic and real-world datasets, with a \emph{focus on kernel choice} and its impact on resulting explanations. A Python implementation of our framework along with all experiments is available at \href{https://anonymous.4open.science/r/Hilbert_functional_decomposition-FD8F/README.md}{GitHub}, with full computational details in App.~\ref{app:experiments}. We additionally validate the \emph{recovery of feature effects and Sobol indices} (Thm.~\ref{thm:sobol_indices}), on a synthetic model with known ground truth across multiple model classes and varying sample sizes, confirming consistent convergence behavior (see App.~\ref{app:ground_truth}).

\subsection{Kernel Guidance}\label{sec:kernel_guidance}

The kernel $K$ encodes assumptions about how feature influence is distributed over the output domain; it should reflect the \emph{research question} rather than be tuned as a hyperparameter. We summarize guidance via three principles, illustrated on synthetic scenarios with known ground truth using pure first-order local effects $f_i(\mathbf{x})$ for a random observation (Fig.~\ref{fig:condensed_kernel_guidance}; see also App.~\ref{app:kernel_guidance}, Tab.~\ref{tab:kernel_guidance} for an extended kernel guidance reference).

\textbf{Principle 1: Match temporal scope to the notion of contribution.} If the question is \emph{which feature drives the output at an instant}, use the identity kernel ($K = I$). For \emph{sustained influence} or \emph{alignment with trajectory shape}, use smoothing or correlation-aware kernels. The ICU example (Fig.~\ref{fig:condensed_kernel_guidance}, col.~1) illustrates this: the identity kernel yields sharp peaks for $X_2$ and $X_3$ at $t=10$h and $t=18$h, tracking instantaneous events. The OU kernel $K_{\mathrm{OU}}(t,s) = \sigma^2\exp(-|t-s|/\ell)$ broadens these into windows of width $\ell$, reflecting that a shock event influences a patient's state for some time before and after it occurs. The correlation kernel reshapes attribution entirely, weighting a feature effect by alignment with output covariance. Since $X_1$'s slow exponential decay drives the trajectory between shocks, its attribution is amplified there, while $X_2$ and $X_3$ retain localized peaks at $t=10$h and $t=18$h. These are not competing estimates --- they answer \emph{different questions}, each valid in context.

\textbf{Principle 2: Match kernel symmetry to temporal causality.} Symmetric kernels (Gaussian, OU) implicitly assume that influence at time $t$ can be informed by past and future output values. This may produce temporally incoherent attributions when the underlying process is causal. Fig.~\ref{fig:condensed_kernel_guidance}, col.~2 considers a market price response driven by a pulse event ($X_1$), a transient temperature effect ($X_2$), and a constant baseline ($X_3$). The unprecetended and short-lived pulse $X_1$ active on $[0.5, 1.0)$h drives the response, but a Gaussian kernel assigns attribution to $X_1$ before the pulse occurs. A causal exponential kernel $K_{\mathrm{causal}}(t,s) = \exp(-(t-s)/\ell)\,\mathbf{1}_{t \geq s}$ correctly restricts attribution to $t \geq 0.5$h. In general, kernel structure should mirror the data-generating process: symmetric for bidirectional dependence, causal for events, interventions, or autoregressive dynamics.

\textbf{Principle 3: Match kernel to feature heterogeneity.} Using a single kernel imposes the same temporal structure across features, which can misattribute effects when feature behaviors differ. A pharmacokinetic response curve for some medication over 3 days demonstrates both the value and the risk of the periodic kernel (Fig.~\ref{fig:condensed_kernel_guidance}, col.~3): it correctly captures the daily recurrence of $X_1$ (8am dosing) but spuriously induces periodicity for $X_2$, a one-off event. Kernel choice should reflect feature-specific temporal behavior; in mixed settings, selecting kernels per-feature may be appropriate. This matters in practice, since recurring covariates are routinely mixed with event-driven ones.

\textbf{When kernel choice matters less.} Time-aggregated attributions yield a single importance score per feature (Fig.~\ref{fig:condensed_kernel_guidance}, row~3). For prediction-based input behavior, feature rankings are often stable across kernels when per-feature attribution trajectories maintain a consistent ordering over time (Thm.~\ref{thm:kernel_effect}), as observed in all three examples. Kernels then mainly affect the temporal \emph{shape} of attribution and can be chosen on interpretive grounds. This stability does not generally hold for sensitivity- or risk-based input behavior, where kernel choice can reorder features (App.~\ref{app:ground_truth}, Fig.~\ref{fig:ranking_games}).

\begin{figure}[ht]
    \centering
    \includegraphics[width=\textwidth]{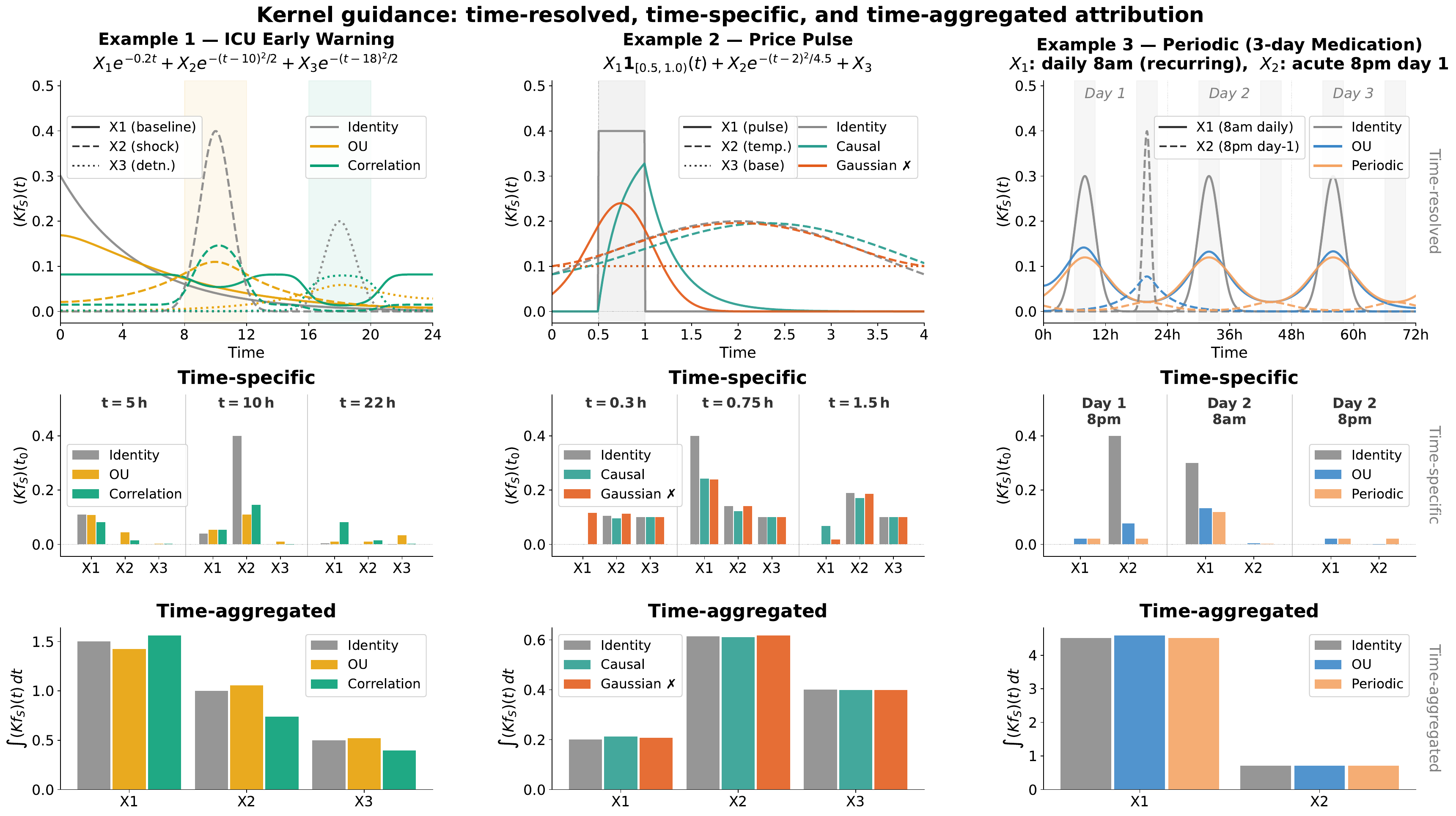}
    \caption{Kernel guidance for three synthetic examples with distinct temporal structures: an ICU early warning model (\emph{col.~1}), a market price pulse (\emph{col.~2}), and a 3-day medication response (\emph{col.~3}). Each example shows \emph{pure local effects} at three explanation levels: time-resolved (\emph{row~1}), time-specific at selected time points (\emph{row~2}), and time-aggregated (\emph{row~3}) for marginal masking $\mathcal{M}$, prediction-based input behavior $\mathcal{B}_{\text{inp}}$, and pure interaction operator $\mathcal{I}$ of a randomly chosen observation.}
    \label{fig:condensed_kernel_guidance}
\end{figure}

\subsection{Real-world Examples}\label{sec:real-world}

\textbf{Intraday SPY volatility prediction.} 
We model intraday volatility for the SPDR S\&P 500 ETF (SPY), using five-minute absolute log-returns (78 time steps) as the target (a standard proxy for intraday volatility) and six pre-market features, including prior-day VIX \footnote{CBOE Volatility Index (VIX), the so-called ``fear index''.} (\texttt{vix\_prev})  and a macro announcement indicator (\texttt{ann\_indicator}). A multivariate random forest is trained on 560 trading days (Jan 2022–Apr 2024)~\citep{polygon2024}. We subtract the average diurnal pattern to focus on deviations from the typical U-shape. In Fig.~\ref{fig:SPY_main} we study July 13, 2022, a high-volatility day during the Fed’s rate-hiking cycle, when the Beige Book --- a qualitative indicator of the US economic situation -- was released at 14:00 ET; full results are in App.~\ref{app:intraday_spy_volatility}. For a trader or risk manager, the key question is not just \emph{which} features matter, but \emph{when} and \emph{how} they drive volatility over the day.

Time-aggregated pure local effects (Fig.~\ref{fig:SPY_main}, col.~1) identify \texttt{vix\_prev} as the dominant driver. We additionally analyze \texttt{ann\_indicator}, whose effect is of direct interest given the scheduled Beige Book release on this day. The temporal structure of their effects depends critically on the kernel. Under the identity kernel (row~1), pure effects (col.~2) are noisy and difficult to interpret. Using temporally structured kernels (row~2) reveals economically meaningful patterns: For \texttt{vix\_prev}, an OU kernel yields a smooth U-shape, indicating elevated volatility at the open and closing times. In contrast, a causal kernel localizes the effect of \texttt{ann\_indicator} sharply at 14:00 without pre-event leakage, whereas a temporally smoothing kernel such as OU would induce spurious attribution before the announcement. The identity kernel also recovers the spike but does not distribute influence post-event, discarding the economically meaningful post-announcement decay. Partial effects (col.~3) show that \texttt{ann\_indicator}’s time-aggregated contribution increases by $1.61\times$ relative to its pure effect, indicating strong interaction with other features. The interaction term with \texttt{vix\_prev} (col.~4) reveals a clear regime shift: before 14:00, the two features are mildly redundant, both reflecting elevated expected volatility; after the announcement, they become strongly synergistic, with the release amplifying the impact of prior market fear. A scalar interaction measure would average over this shift and miss the change in regime. Overall, the framework reveals when volatility drivers act and how their role shifts intraday, matching how traders interpret market events.

\begin{figure}[ht]
    \centering
    \includegraphics[width=\textwidth]{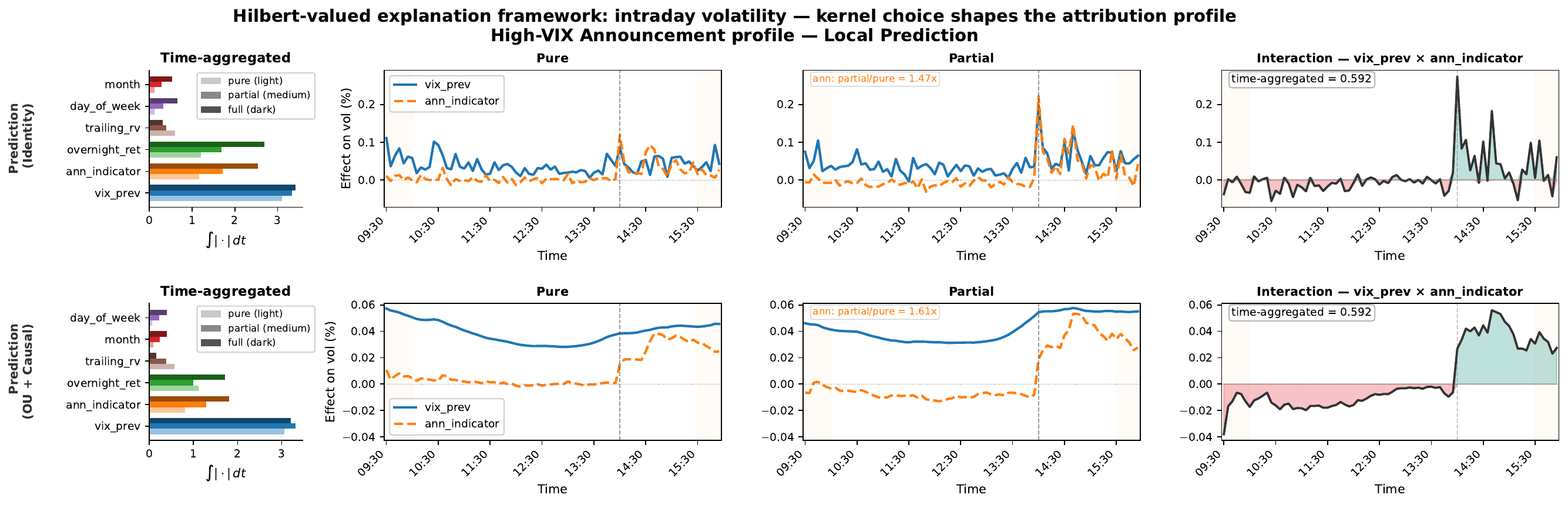}
    \caption{Local prediction effects for July 13, 2022 (\texttt{vix\_prev}~$= 27.3$, \texttt{ann\_indicator}~$= 1$), showing \texttt{vix\_prev} and \texttt{ann\_indicator}. \emph{Col.~1}: Time-aggregated pure, partial, and full importances across all features under the identity kernel (\emph{row~1}) and OU/causal kernel (\emph{row~2}). \emph{Col.~2}: Pure effects; \emph{Col.~3}: Partial effects; \emph{Col.~4}: Pure interaction effect between \texttt{vix\_prev} and \texttt{ann\_indicator}. The dashed vertical line marks 14:00 ET, the Beige Book release time. Note that pure prediction-based input behavior recovers the PD effect \citep{friedman2001greedy} and partial prediction-based input behavior recovers Shapley values \citep{shapley1951, lundberg2017shap} (see App. \ref{app:operators} for more details).}
    \label{fig:SPY_main}
\end{figure}

\textbf{Electricity demand comparison.}
Grid operators must manage forecast uncertainty, as high demand variance requires costly reserves. We ask which features drive this uncertainty and whether it is localized in time or spans the full day. This can be studied using the full sensitivity effect (time-dependent total Sobol index, see Thm.~\ref{thm:sobol_indices}). The kernel determines how prediction variance is attributed over time: the identity kernel treats each time step independently yielding pointwise variances, while the correlation kernel captures temporal dependence of covariances. We evaluate on the UCI Individual Household Electric Power Consumption dataset (IHEPC, $T=24$ hourly periods, ${\approx}1400$ days, one \emph{individual} household)~\citep{dua2019uci} and the GB National Electricity System Operator historic demand dataset (NESO, $T=48$ half-hourly periods, ${\approx}1800$ days, \emph{national} demand)~\citep{neso2024demand} . We follow the same modeling setup as in the SPY application; full details in App.~\ref{app:energy_comparison}).

The datasets differ markedly in temporal dependence (Fig.~\ref{fig:energy_main}, col.~1, 4). IHEPC shows strong correlation between nearby time points decaying with temporal distance, while NESO demand is strongly correlated across the entire day. The network plots (col.~2, 5) summarize full sensitivity effects, with node size encoding total Sobol effects and edge width full pairwise interaction strength. IHEPC is most strongly driven by \texttt{lag\_daily\_mean} (\texttt{LDM}), with diffuse interactions reflecting the idiosyncratic nature of single-household demand. NESO variance is dominated by \texttt{month} (\texttt{Mon}) and \texttt{season} (\texttt{Sea}), with stronger interactions with each other, \texttt{lag\_evening} (\texttt{LEv}) and \texttt{lag\_morning} (\texttt{LMo}). The two kernels then tell complementary stories. For IHEPC, identity kernel full sensitivity effects reveal sharp peaks in the AM and PM high-demand windows (red and blue shaded areas), whereas the correlation kernel smooths these into a shallower, broader elevation, since nearby hours are tightly coupled and attribution at any peak hour borrows strength from its neighbors (col.~3). For NESO, intraday variation under the identity kernel collapses to near-flat trajectories under the correlation kernel, consistent with a regime shift affecting the entire day (col.~6). In both cases, hourly reserve adjustments mis-target the structure of forecast variance: IHEPC calls for peak-window provisioning, NESO for daily capacity planning. Overall, the correlation kernel matches effect resolution to the dataset's dependence scale: local smoothing for IHEPC, full-day coherence for NESO.

\begin{figure}[ht]
    \centering
    \includegraphics[width=\textwidth]{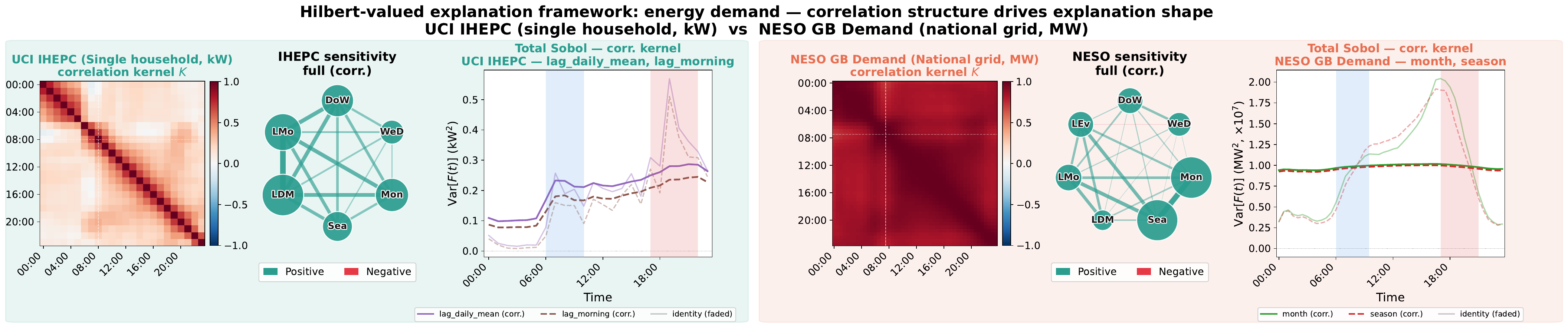}
    \caption{For each dataset (teal: IHEPC single household; orange: NESO national grid): empirical correlation heatmap (\emph{col.~1}); full sensitivity network under the correlation kernel (\emph{col.~2}); full sensitivity-based effect (total Sobol) for the two dominant features under the identity kernel (faded) and correlation kernel (solid/dashed) (\emph{col.~3}). Blue and red shading marks AM and PM peak windows.}
    \label{fig:energy_main}
\end{figure}

\section{Discussion and Limitations}\label{sec:limitations}

We introduced a Hilbert-valued explanation framework that generalizes feature-based explanations to time-dependent outputs and unifies existing methods under a common operator view. Our results show that this extension matters in practice: 
the kernel choice directly affects the explanations obtained, surfacing temporal structure that pointwise approaches cannot capture. Rather than treating this choice as a nuisance, we view it as a modeling decision that encodes the notion of \emph{contribution}—whether instantaneous, persistent, phase-dependent, or causal. Consequently, kernel selection should be guided by the scientific question and domain knowledge (see Tab.~\ref{tab:kernel_guidance}).
The same holds for hyperparameters: for example, the OU length-scale in the SPY experiment was chosen to reflect the empirical half-life of volatility shocks, but different applications may require different temporal scales.

A key limitation is computational cost, as Shapley-based interaction operators remain expensive \citep{lundberg2017shap,covert2021explaining}, although standard approximation methods apply \citep{strumbelj2014explaining} and fast implementations can be extended \citep{jethani2022fastshap,fumagalli2023shapiq}. In contrast, pure and full effects can be computed more efficiently. Beyond computation, an important limitation is the reliance on user-specified kernels, which introduces an additional modeling choice that may influence results if misspecified. While this flexibility is a strength, it also requires principled kernel choices aligned with the interpretive goal.

The framework extends naturally to general Hilbert-valued outputs, including spatial fields or images, or other structured outputs. A promising direction is to combine it with (functional) principal component analysis, to obtain low-dimensional and more interpretable representations of output variability before attribution. Another avenue is the use of operator-valued kernels \citep{kadri2016operator}, allowing output dependencies to vary with the input. Finally, the framework can accommodate time-dependent inputs through suitable masking choices, which we leave for future work.

\begin{ack}
Sophie Hanna Langbein and Niklas Koenen have been funded by the German Research Foundation (DFG) as part of the Research Unit ``Lifespan AI: From Longitudinal Data to Lifespan Inference in Health" (DFG FOR 5347), Grant 459360854. Julia Herbinger and Marvin N. Wright gratefully acknowledge funding by the German Research Foundation (DFG), Emmy Noether Grant 437611051.
\end{ack}

\newpage

\bibliographystyle{plainnat}
\bibliography{bib}
%References follow the acknowledgments in the camera-ready paper. Use unnumbered first-level heading for
%the references. Any choice of citation style is acceptable as long as you are
%consistent. It is permissible to reduce the font size to \verb+small+ (9 point)
%when listing the references.
%Note that the Reference section does not count towards the page limit.

%%%%%%%%%%%%%%%%%%%%%%%%%%%%%%%%%%%%%%%%%%%%%%%%%%%%%%%%%%%%
\newpage

\appendix

%\section{Technical Appendices and Supplementary Material}
%Technical appendices with additional results, figures, graphs and proofs may be submitted with the paper submission before the full submission deadline (see above), or as a separate PDF in the ZIP file below before the supplementary material deadline. There is no page limit for the technical appendices.

\section{Proofs}\label{app:proofs}

Here, we provide the proofs for the theoretical statements in Sec.\ref{sec:method}.

\subsection{Proof of Thm.~\ref{thm:kernel_effect}}\label{app:proof_linearity}

\begin{proof}
By definition of the time-specific output behavior operator for prediction-based input behavior $\nu_\mathrm{inp}(S)(t) = F^\mathcal{H}_S(\mathbf{x})(t)$,
\[
(\mathcal B_{\mathrm{out}}^{\mathrm{spec}} \nu_{\mathrm{inp}})(S,t)
=
\int_{\mathcal T} K(t,s)\,\nu_{\mathrm{inp}}(S)(s)\,ds.
\]
Assuming the integrals are well-defined, Fubini's theorem yields
\begin{align*}
 (\mathcal B_{\mathrm{out}}^{\mathrm{aggr}} \nu_{\mathrm{inp}})(S)
&=
\int_{\mathcal T} (\mathcal B_{\mathrm{out}}^{\mathrm{spec}} \nu_{\mathrm{inp}})(S,t)\,dt\\
&=
\int_{\mathcal T}\int_{\mathcal T} K(t,s)\,\nu_{\mathrm{inp}}(S)(s)\,ds\,dt\\
&=
\int_{\mathcal T}\left(\int_{\mathcal T} K(t,s)\,dt\right)\nu_{\mathrm{inp}}(S)(s)\,ds.   
\end{align*}

Defining
\[
w_K(s) := \int_{\mathcal T} K(t,s)\,dt,
\]
we obtain the stated representation.

For the ordering, consider individual pure feature effects and let $i,j \in P:=\{1,\dots,p\}$. If $\mathcal B_{\mathrm{inp}}$ is \textbf{linear} in $F^\mathcal H$ and
\[
\nu_{\mathrm{inp}}(\{i\})(s)\ge \nu_{\mathrm{inp}}(\{j\})(s)
\qquad \text{for all } s\in\mathcal T,
\]
then, since $K(t,s)\ge 0$, it follows that
\[
w_K(s)=\int_{\mathcal T} K(t,s)\,dt \ge 0
\qquad \text{for all } s\in\mathcal T.
\]
Hence,
\[
w_K(s)\big(\nu_{\mathrm{inp}}(\{i\})(s) - \nu_{\mathrm{inp}}(\{j\})(s)\big)\ge 0
\qquad \text{for all } s\in\mathcal T,
\]
and therefore
\[
\int_{\mathcal T} w_K(s)\big(\nu_{\mathrm{inp}}(\{i\})(s) - \nu_{\mathrm{inp}}(\{j\})(s)\big)\,ds \ge 0.
\]
Using the representation above, this yields
\[
(\mathcal B_{\mathrm{out}}^{\mathrm{aggr}} \nu_{\mathrm{inp}})(\{i\})
\ge
(\mathcal B_{\mathrm{out}}^{\mathrm{aggr}} \nu_{\mathrm{inp}})(\{j\}),
\]
which proves the result for individual pure feature effects.

\textbf{Extension to partial and full effects.}
For partial effects (Shapley values), the attribution of a feature $i$ is given by
\[
\phi_i^{\mathrm{Shap}}
=
\sum_{S \subseteq P \setminus \{i\}} \alpha_S
\big[\nu(S \cup \{i\}) - \nu(S)\big],
\]
where $\alpha_S > 0$ are the Shapley weights.

To compare two features $i$ and $j$, consider the difference
\[
\phi_i^{\mathrm{Shap}} - \phi_j^{\mathrm{Shap}}.
\]
Substituting the definition and collecting terms yields
\[
\phi_i^{\mathrm{Shap}} - \phi_j^{\mathrm{Shap}}
=
\sum_{S \subseteq P \setminus \{i\}} \alpha_S
\big[\nu(S \cup \{i\}) - \nu(S)\big]
-
\sum_{S \subseteq P \setminus \{j\}} \alpha_S
\big[\nu(S \cup \{j\}) - \nu(S)\big].
\]
By reindexing the sums and grouping terms over coalitions 
$A \subseteq P \setminus \{i,j\}$, this expression can be written as a sum of terms of the form
\[
\nu(A \cup \{i\}) - \nu(A \cup \{j\}),
\]
each with a positive weight.

Hence, if for all $A \subseteq P \setminus \{i,j\}$ and all $s \in \mathcal T$,
\[
\nu_{\mathrm{inp}}(A \cup \{i\})(s)
\ge
\nu_{\mathrm{inp}}(A \cup \{j\})(s),
\]
then, by the same argument as in the main proof,
\[
\nu(A \cup \{i\}) - \nu(A \cup \{j\})
=
\int_{\mathcal T} w_K(s)
\big[
\nu_{\mathrm{inp}}(A \cup \{i\})(s)
-
\nu_{\mathrm{inp}}(A \cup \{j\})(s)
\big] ds
\ge 0,
\]
and therefore $\phi_i^{\mathrm{Shap}} \ge \phi_j^{\mathrm{Shap}}$.

For full effects,
\[
\phi_i^{\mathrm{full}}
=
\nu(P) - \nu(P \setminus \{i\}),
\]
so that
\[
\phi_i^{\mathrm{full}} - \phi_j^{\mathrm{full}}
=
\nu(P \setminus \{j\}) - \nu(P \setminus \{i\}).
\]
This corresponds to the same condition with
$A = P \setminus \{i,j\}$, since
\[
P \setminus \{j\} = A \cup \{i\},
\quad
P \setminus \{i\} = A \cup \{j\}.
\]
Hence, if
\[
\nu_{\mathrm{inp}}(A \cup \{i\})(s)
\ge
\nu_{\mathrm{inp}}(A \cup \{j\})(s)
\quad \forall s \in \mathcal T,
\]
then $\phi_i^{\mathrm{full}} \ge \phi_j^{\mathrm{full}}$.

\end{proof}

\subsection{Proof of Proposition~\ref{thm:moebius}}\label{app:proof_moebius}

\begin{proof}
By construction,
\[
\nu:2^{\{1,\dots,p\}}\to\mathbb R,
\qquad
S\mapsto (\mathcal B_{\mathrm{out}}(\mathcal B_{\mathrm{inp}}(\mathcal M F^\mathcal H)))(S),
\]
is a well-defined set function on the power set $2^{\{1,\dots,p\}}$.

Let $m_\nu$ denote its Möbius transform, defined by
\[
m_\nu(S):=\sum_{L\subseteq S}(-1)^{|S|-|L|}\nu(L),
\]
which captures the contributions of individual subsets via inclusion–exclusion.

By Möbius inversion \citep{grabisch1997mobius}, $\nu$ admits the representation
\[
\nu(T)=\sum_{S\subseteq T} m_\nu(S)
\qquad\text{for all }T\subseteq\{1,\dots,p\}.
\]
Applying this to $T=\{1,\dots,p\}$ and substituting the definition of $m_\nu$ yields
\[
\nu(\{1,\dots,p\})
=
\sum_{S\subseteq\{1,\dots,p\}}\sum_{L\subseteq S}(-1)^{|S|-|L|}\nu(L),
\]
which is the claimed decomposition.

Since the Möbius transform is unique, this representation defines a unique additive decomposition of $\nu$ into contributions associated with feature subsets. As the construction of $\nu$ is independent of the specific choices of $\mathcal M$, $\mathcal B_{\mathrm{inp}}$, and $\mathcal B_{\mathrm{out}}$, the result holds for all corresponding time-specific, time-resolved, and time-aggregated explanations.
\end{proof}

\subsection{Proof of Theorem \ref{thm:sobol_indices}}\label{app:proof_sobol}
\begin{proof}
By definition of the sensitivity input behavior,
\[
\nu_{\mathrm{inp}}(S)(t,s)
=
\mathrm{Cov}\!\big(F_S^\mathcal H(\mathbf X)(t),
F_S^\mathcal H(\mathbf X)(s)\big).
\]
For the constant kernel \(K(t,s)=1\), the output behavior in Eq.~\eqref{eq:output_behavior} reduces to integration over the second output argument:
\[
(\mathcal B_{\mathrm{out}}\nu_{\mathrm{inp}}(S))(t)
=
\int_{\mathcal T}\nu_{\mathrm{inp}}(S)(t,s)\,ds.
\]
Thus,
\[
\nu(S)(t)
=
\int_{\mathcal T}
\mathrm{Cov}\!\big(F_S^\mathcal H(\mathbf X)(t),
F_S^\mathcal H(\mathbf X)(s)\big)\,ds,
\]
and integrating once more over \(t\) yields
\[
\nu(S)
=
\int_{\mathcal T}\nu(S)(t)\,dt
=
\int_{\mathcal T}\int_{\mathcal T}
\nu_{\mathrm{inp}}(S)(t,s)\,dt\,ds.
\]

Under marginal masking,
\[
F_S^\mathcal H(\mathbf X)
=
\mathbb E_{\mathbf{X}_{-S}\sim P_{\mathbf{X}_{-S}}}\!\left[F^\mathcal H(\mathbf X)\mid \mathbf X_S = \mathbf{x}_S\right],
\]
so $\nu(S)(t)$ and $\nu(S)$ are exactly the covariance-based Sobol value functions for time-specific and time-aggregated trajectory outputs.

Under the feature independence assumption, the Hilbert-valued functional decomposition yields orthogonal pure effects $f_L^\mathcal H$,
and therefore
\[
\nu(S)(t)
=
\sum_{L\subseteq S}
\int_{\mathcal T}
\mathrm{Cov}\!\big(f_{L,t}^\mathcal H(\mathbf X),
f_{L,s}^\mathcal H(\mathbf X)\big)\,ds.
\]
Thus, the pure interaction operator, which aggregates contributions over $L\subseteq S$, recovers the closed Sobol contribution. Analogously, the full interaction operator aggregates all contributions involving at least one feature in $S$, i.e. all $L$ with $L\cap S\neq\emptyset$, and therefore recovers the total Sobol contribution.

Normalizing these quantities by the corresponding total covariance quantity,
\[
\nu(\{1,\dots,p\})(t)
=
\int_{\mathcal T}
\mathrm{Cov}\!\big(F_t^\mathcal H(\mathbf X),
F_s^\mathcal H(\mathbf X)\big)\,ds,
\]
gives the closed and total time-specific Sobol indices. Integrating the same covariance contributions over \(t\in\mathcal T\) yields the corresponding time-aggregated closed and total Sobol indices. Hence, the proposed framework recovers these Sobol indices as a special case.
\end{proof}

\section{Operator Choices in the Hilbert-Valued Explanation Framework}
\label{app:operators}

We provide additional details on the operator choices underlying the Hilbert-valued explanation framework. Building on \cite{fumagalli2025unifying,herbinger2026granite}, we extend masking, behavior, and interaction operators to Hilbert-valued outputs by decomposing the behavior operator into input and output components, $\mathcal B_{\mathrm{inp}}$ and $\mathcal B_{\mathrm{out}}$.

\begin{table}[h]
\centering
\caption{Operator choices in the Hilbert-valued explanation framework.}
\begin{tabular}{lll}
\toprule
\textbf{Concept} & \textbf{Operator} & \textbf{Definition} \\
\midrule

\multicolumn{3}{c}{\textbf{Masking Operators} ($\mathcal{M}^{(\cdot)}$) } \\\midrule
baseline & $\mathcal{M}^b$ & $F_S^\mathcal H(\mathbf{x}) = F^\mathcal H(\mathbf{x}_S, b_{-S})$ \\
marginal & $\mathcal{M}^m$ & $F_S^\mathcal H(\mathbf{x}) = \mathbb{E}_{\mathbf{X}_{-S}}[F^\mathcal H(\mathbf{x}_S, \mathbf{X}_{-S})]$ \\
conditional & $\mathcal{M}^c$ & $F_S^\mathcal H(\mathbf{x}) = \mathbb{E}_{\mathbf{X}_{-S}\mid \mathbf{X}_S=\mathbf{x}_S}[F^\mathcal H(\mathbf{x}_S, \mathbf{X}_{-S})]$ \\

\midrule
\multicolumn{3}{c}{\textbf{Input Behavior Operators} ($\mathcal B_{\mathrm{inp}}^{(\cdot)}$)} \\\midrule
prediction & $\mathcal B_{\mathrm{inp}}^{\mathrm{pred}}$ 
& $\nu_\mathrm{inp}^\mathrm{pred}(S)(t) = F_S^\mathcal H(\mathbf{x})(t)$ \\
sensitivity & $\mathcal B_{\mathrm{inp}}^{\mathrm{sens}}$ 
& $\nu_\mathrm{inp}^\mathrm{sens}(S)(t, s) = \mathrm{Cov}(F_S^\mathcal H(\mathbf{X})(t),F_S^\mathcal H(\mathbf{X})(s))$ \\
risk & $\mathcal B_{\mathrm{inp}}^{\mathrm{risk}}$ 
& $\nu_\mathrm{inp}^\mathrm{risk}(S)(t) =  \mathbb{E}[\ell(\mathbf{Y}(t),F_S^\mathcal H(\mathbf{X})(t))]$ \\

\midrule
\multicolumn{3}{c}{\textbf{Output Behavior Operators} ($\mathcal B_{\mathrm{out}}^{(\cdot)}$)} \\\midrule
time-specific & $\mathcal B_{\mathrm{out}}^{\mathrm{spec}}$ 
& $\nu^{\mathrm{spec}}(S,t) = \int K(t,s)\, g_S(s)\,ds$ \\

time-resolved & $\mathcal B_{\mathrm{out}}^{\mathrm{res}}$ 
& $\nu^{\mathrm{res}}(S) = \{\nu^{\mathrm{spec}}(S,t)\}_{t\in\mathcal T}$ \\

time-aggregated & $\mathcal B_{\mathrm{out}}^{\mathrm{aggr}}$ 
& $\nu^{\mathrm{aggr}}(S) = \int \int K(t,s)\, g_S(s)\,ds\,dt$ \\

\midrule
\multicolumn{3}{c}{\textbf{Interaction Operators} ($\mathcal{I}^{(\cdot)}$)} \\\midrule
pure & $\mathcal I^{\mathrm{pure}}$ & $\phi^{\mathrm{pure}}(j) = \nu(j) - \nu(\emptyset)$ \\
partial & $\mathcal I^{\mathrm{partial}}$ &  $\phi^{\mathrm{partial}}(j) = \sum_{S \subseteq -j} \frac{1}{p \binom{p-1}{|S|}}[\nu(S \cup j)-\nu(S)]$\\
full & $\mathcal I^{\mathrm{full}}$ & $\phi^{\mathrm{full}}(j) = \nu(D) - \nu(-j)$ \\

\bottomrule
\end{tabular}
\end{table}

\textbf{Masking operators.}
Masking operators $\mathcal{M}$ define how features outside a subset $S$ are removed by different masking strategies. Baseline masking fixes features to a reference value $b_{-S}$, marginal masking integrates them out under the marginal distribution, and conditional masking preserves all feature dependencies. Under independence, conditional and marginal masking coincides with the functional decomposition used in sensitivity analysis (functional ANOVA) \citep{sobol2001, hooker2004discovering}.

\textbf{Input behavior operators.}
The input behavior operator $\mathcal B_{\mathrm{inp}}$ defines the quantity of interest prior to aggregation across the output domain. Prediction-based behavior captures local effects, sensitivity-based behavior captures variability via covariance, and risk-based behavior captures expected loss. Depending on the choice, $\nu_{\mathrm{inp}}(S)$ is defined either pointwise over $\mathcal T$ or over pairs $(t,s)$.

\textbf{Output behavior operators.}
The output behavior operator $\mathcal B_{\mathrm{out}}$ determines how Hilbert-valued quantities are aggregated across the output domain. Time-specific operators yield explanations at a fixed output location, time-resolved operators evaluate time-specific explanations for all $t \in \mathcal{T}$, and time-aggregated operators summarize contributions over the entire domain $\mathcal{T}$ into a scalar. All variants depend on a kernel $K$, which controls how dependencies across the output domain are incorporated.

\textbf{Interaction operators.}
Interaction operators map set functions to feature-based explanations. Pure effects ignore interactions, partial effects (Shapley values) fairly distribute interaction effects across involved features, and full effects assign the entire interaction effect to each involved feature. Under marginal masking and sensitivity behavior, pure and full effects recover closed and total Sobol indices, respectively.

\textbf{Method mapping.}
Many existing explanation methods arise as specific choices of masking, input behavior, and interaction operators (see Tab.~\ref{tab:method_mapping}). Prediction-based methods such as PDP and SHAP differ only in the interaction operator, while sensitivity-based methods correspond to Sobol indices with pure and full interactions yielding closed and total effects. Risk-based methods such as PFI and SAGE quantify performance degradation, differing again only in the interaction operator. The output behavior operator $\mathcal B_{\mathrm{out}}$ determines based on a kernel choice $K$ how these quantities are represented across the output domain (time-specific, time-resolved, or time-aggregated) and can be chosen independently, yielding valid explanations in all cases due to the additive decomposition established in Thm.~\ref{thm:moebius}. Detailed derivations of various feature-based explanation methods for the scalar case are given in \cite{fumagalli2025unifying}.

\begin{table}[h]
\centering
\caption{Mapping of common explanation methods mentioned in this work to operator choices in the Hilbert-valued explanation framework. Output behavior $\mathcal B_{\mathrm{out}}$ can be chosen independently (time-specific, time-resolved, or time-aggregated based on different kernel choices $K$), yielding valid explanations in all cases.}
\label{tab:method_mapping}
\begin{tabularx}{\textwidth}{l l l l X}
\toprule
\textbf{Method} & $\mathcal M$ & $\mathcal B_{\mathrm{inp}}$ & $\mathcal I$ & Description \\
\midrule

PDP & marginal & prediction & pure 
& average marginal effect\\

c-PDP & conditional & prediction & pure 
& average conditional marginal effect\\

SHAP (int.) & marginal & prediction & partial 
& (interventional) Shapley values \\

SHAP (obs.) & conditional & prediction & partial 
& (observational) Shapley values \\

Closed Sobol & marginal & sensitivity & pure 
& variance explained by features in $S$ \\

Total Sobol & marginal & sensitivity & full 
& variance explained by features in $S$ including all their interactions with $-S$ \\

PFI & marginal & risk & full 
& model performance drop when removing feature $j$ and all its interactions with $-j$ (marginal)\\

c-PFI & conditional & risk & full 
& model performance drop when removing feature $j$ and all its interactions with $-j$ (conditional) \\

SAGE & marginal / conditional & risk & partial 
& global Shapley values based on risk games \\

\bottomrule
\end{tabularx}
\end{table}

\section{Mapping Existing Methods to the Hilbert-valued Explanation Framework}
\label{app:mapping_existing}

This section formalizes how existing feature-based explanation methods for time-dependent outputs can be expressed within the proposed operator framework. Throughout, we consider explanations of a Hilbert-valued prediction function $F^\mathcal H$ based on masking $\mathcal M$, input behavior $\mathcal B_{\mathrm{inp}}$, output behavior $\mathcal B_{\mathrm{out}}$, and interaction operator $\mathcal I$.

\textbf{Time-specific and time-aggregated output behavior.}
For prediction-based explanations, the input behavior is given by
\[
\nu_{\mathrm{inp}}(S)(t) = F_S^\mathcal H(\mathbf X)(t).
\]
Time-specific explanations correspond to
\[
(\mathcal B_{\mathrm{out}}^{\mathrm{spec}} \nu_{\mathrm{inp}})(S,t)
= \nu_{\mathrm{inp}}(S)(t),
\]
while time-aggregated explanations are defined as
\[
(\mathcal B_{\mathrm{out}}^{\mathrm{aggr}} \nu_{\mathrm{inp}})(S)
= \int_{\mathcal T} \nu_{\mathrm{inp}}(S)(t)\,dt.
\]

All existing methods considered here implicitly assume $K = I$, i.e., they do not model dependencies across the output domain.

If $\mathcal B_{\mathrm{inp}}$ is linear in $F^\mathcal H$ and the aggregation operator defining $\mathcal B_{\mathrm{out}}^{\mathrm{aggr}}$ is linear (e.g., integration or averaging), then aggregation commutes with masking. Formally,
\[
\int_{\mathcal T} F_S^\mathcal H(\mathbf X)(t)\,dt
=
\mathbb E\!\left[\int_{\mathcal T} F^\mathcal H(\mathbf X)(t)\,dt \;\middle|\; \mathbf X_S \right].
\]

Hence, aggregating masked predictions is equivalent to computing explanations for an aggregated output (e.g., the mean prediction). 

In practice, existing methods often implement aggregation by first defining a scalar output $\widetilde F$ (e.g., the mean forecast) and then applying standard explanation methods. Within our framework, this corresponds to treating aggregation as part of the output behavior operator.

This equivalence holds for linear aggregation functionals but not in general for nonlinear ones (e.g., maximum or quantiles), where aggregation and masking do not commute.

\textbf{TsSHAP.}
TsSHAP corresponds to prediction-based explanations $\mathcal B_{\mathrm{inp}}^\mathrm{pred}$, partial interaction operator $\mathcal I^{\mathrm{partial}}$, and time-specific output behavior
\[
\mathcal B_{\mathrm{out}} = \mathcal B_{\mathrm{out}}^{\mathrm{spec}}.
\]
Explanations are computed independently for each $t \in \mathcal T$.

\textbf{ShapTime.}
ShapTime explains aggregated forecast quantities (e.g., the mean prediction over the horizon. In practice, this is achieved by defining an aggregated scalar output prior to masking and applying standard Shapley methods. Within our framework, this corresponds to prediction-based input behavior $\mathcal B_{\mathrm{inp}}^\mathrm{pred}$ and time-aggregated output behavior combined with partial interaction operator $\mathcal I^{\mathrm{partial}}$. For linear aggregation (e.g., the mean), both formulations are equivalent.

\textbf{PAX-TS.}
PAX-TS computes feature importance via localized perturbations, measuring changes in selected output quantities (e.g., individual time points or aggregated statistics) when features are modified. Within our framework, this corresponds to prediction-based input behavior $\mathcal B_{\mathrm{inp}}^\mathrm{pred}$ combined with either time-specific or time-aggregated output behavior,
\[
\mathcal B_{\mathrm{out}} \in \{\mathcal B_{\mathrm{out}}^{\mathrm{spec}},\, \mathcal B_{\mathrm{out}}^{\mathrm{aggr}}\},
\]
and the full effects interaction operator $\mathcal I^{\mathrm{full}}$. For linear aggregation functionals, this is equivalent to aggregating time-specific explanations.

\textbf{SurvSHAP(t).}
SurvSHAP(t) provides time-specific and time-resolved explanations for survival models by computing Shapley values independently for each $t \in \mathcal T$. Within our framework, this corresponds to prediction-based input behavior $\mathcal B_{\mathrm{inp}}^\mathrm{pred}$, partial interaction operator $\mathcal I^{\mathrm{partial}}$, and time-specific output behavior
\[
\mathcal B_{\mathrm{out}} = \mathcal B_{\mathrm{out}}^{\mathrm{spec}},
\]
evaluated for all $t$, yielding a time-resolved explanation over the output domain.

\textbf{SurvSHAP-IQ.}
SurvSHAP-IQ extends SurvSHAP(t) to interaction effects. In our framework, this corresponds to using higher-order interaction operators $\mathcal I$ (beyond partial effects), applied to the same time-specific output behavior (see \cite{fumagalli2025unifying} for definitions of higher-order interaction operators).

\textbf{Remarks on masking.}
Some of the above methods, such as PAX-TS, consider time-dependent inputs and employ localized or time-dependent perturbations. While our framework currently focuses on marginal or conditional feature masking, these approaches can be interpreted as alternative masking strategies acting on structured or time-dependent inputs. Extending the framework to explicitly account for time-dependent masking operators is a natural direction for future work.

\textbf{Summary.}
All considered methods correspond to prediction-based input behavior combined with either time-specific or time-aggregated output behavior under $K=I$, i.e., they treat time points independently and do not model dependencies across the output domain. In contrast, the proposed framework generalizes these approaches by allowing non-trivial kernels $K$ for correlation-aware explanations and by incorporating alternative input behaviors (e.g., risk or sensitivity) for global importance measures.

\section{Experiments}\label{app:experiments}

This appendix contains full experimental details for all synthetic and real-data experiments in the paper. All analysis code is available at \href{https://github.com/bips-hb/Hilbert_explanation_framework}{GitHub}.

\subsection{Ground-truth Recovery}\label{app:ground_truth}

\textbf{Experimental setup}

We validate the $\mathcal{H}$-FD framework on an extended version of the synthetic ICU biomarker model (Fig.~\ref{fig:intro}), augmented with a pairwise interaction term to enable interaction recovery assessment:
\begin{equation}
    F(\mathbf{x})(t) = X_1 e^{-0.2t} + X_2 e^{-(t-10)^2/2}
    + X_3 e^{-(t-18)^2/2}
    + \alpha(x_1 - \mu)(x_2 - \mu)\,e^{-(t-5)^2/2},
\end{equation}
with $X_i \sim \mathrm{Uniform}[0,1]$, $\mu = 0.5$, $\alpha = 1.0$, evaluated at $\bm{x}^* = (0.8,\,0.9,\,0.7)$. The interaction term accounts for approximately $2.4\%$ of total variance. Gaussian observation noise with signal-to-noise ratio of $5$ is added during training.

We estimate pure prediction effects via the Möbius transform applied to four black-box model classes fitted to $n$ noisy observations: Ridge regression (deliberately misspecified, without interaction term), random forest, NGBoost, and a multi-layer perceptron (MLP). An oracle
estimator applies the Möbius transform directly to the true model, isolating estimation variance from model approximation error. Experiments are run across $n \in \{50, 100, 250, 500, 1000, 2000, 5000, 10000\}$ training samples, averaged over $30$ Monte Carlo runs.

\textbf{Evaluation metrics}

Recovery is evaluated using two complementary metrics. Let $\hat{f}_S(t)$ denote the estimated pure effect and $f_S(t)$ the analytical ground truth.

\textit{Normalized $L^2$ error.}
\begin{equation}
    \varepsilon_{L^2}(S)
    = \frac{\|\hat{f}_S - f_S\|_{L^2}}{\|f_S\|_{L^2}}
    = \frac{\left(\int (\hat{f}_S(t) - f_S(t))^2\,\mathrm{d}t\right)^{1/2}}
           {\left(\int f_S(t)^2\,\mathrm{d}t\right)^{1/2}},
\end{equation}
measuring pointwise trajectory recovery relative to the signal magnitude. An error of $1.0$ corresponds to a null (zero) predictor; errors exceeding $1.0$ indicate the estimated effect is further from the truth than predicting zero, typically occurring when the wrong sign is learned at small~$n$.

\textit{Relative aggregated error.}
\begin{equation}
    \varepsilon_{\mathrm{agg}}(S) = \frac{|\hat{\Phi}_S - \Phi_S|}{\Phi_S},
    \qquad
    \Phi_S = \int |f_S(t)|\,\mathrm{d}t,
\end{equation}
measuring recovery of the time-aggregated importance $\Phi_S$, the quantity directly used for feature ranking in practice.

\textbf{Effect recovery.} Fig.~\ref{fig:n_recovery_agg_l2} shows normalized $L^2$ error and relative aggregated error as a function of $n$ for all four feature subsets. The oracle converges rapidly for all effects, confirming that the Möbius estimation procedure is consistent given the true model. All non-linear models recover the three main effects with $L^2$ errors below $5\%$ at $n = 2{,}000$ and below $2\%$ at $n = 10{,}000$; relative aggregated errors follow a similar pattern. The interaction term $f_{\{X_1,X_2\}}$ is harder to recover at all sample sizes due to its small signal contribution, but non-linear models consistently converge toward zero error with increasing data. Ridge regression recovers the main effects at a comparable rate but plateaus at $\varepsilon_{L^2} = 1.0$ and $\varepsilon_{\mathrm{agg}} = 1.0$ for the interaction at all $n$, correctly reflecting its structural inability to represent this term. Full numerical results are provided in Tabs.~\ref{tab:l2_full} and~\ref{tab:agg_full}.

\begin{figure}[h]
    \centering
    \includegraphics[width=\textwidth]{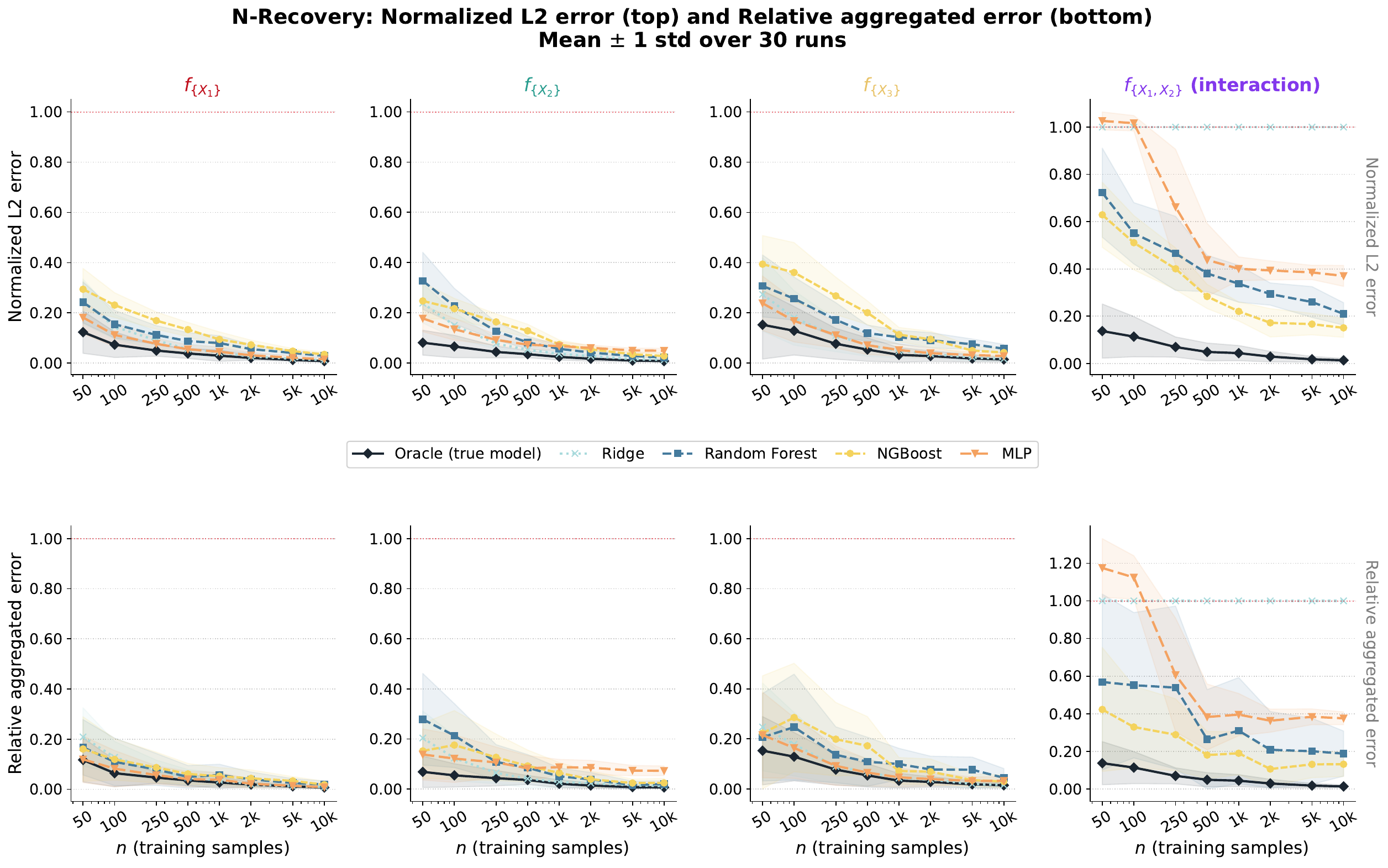}
    \caption{\textbf{N-Recovery}: Normalized $L^2$ error (\emph{row~1}) and relative aggregated error (\emph{row~2}) for the three main effects and the pairwise interaction $f_{\{X_1,X_2\}}$, as a function of training sample size $n$ (mean $\pm$ 1 standard deviation (std) over 30 runs). The oracle lower bound reflects finite background sample noise. Main effects converge reliably across all models; the interaction term is substantially harder to recover, with the MLP showing the slowest convergence.}
    \label{fig:n_recovery_agg_l2}
\end{figure}

\begin{table}[t]
\caption{Normalized $L^2$ error between estimated and oracle pure effects (mean $\pm$ 1 std over 30 independent runs). Each effect trajectory is normalized by the $L^2$ norm of the corresponding oracle effect, so a value of $1.0$ corresponds to a trivially zero prediction. Ridge regression fails to recover the interaction effect $f_{\{X_1,X_2\}}$ at all sample sizes, reflecting model misspecification under the additive linear assumption.}
\label{tab:l2_full}
\centering
\resizebox{\textwidth}{!}{%
\begin{tabular}{llllllllll}
\toprule
Model / Effect & $n=50$ & $n=100$ & $n=250$ & $n=500$ & $n=1$k & $n=2$k & $n=5$k & $n=10$k \\
\midrule
\multicolumn{9}{l}{\textit{Oracle (true model)}} \\
$f_{\{X_1\}}$ & $0.122 \pm 0.068$ & $0.062 \pm 0.030$ & $0.051 \pm 0.031$ & $0.043 \pm 0.020$ & $0.028 \pm 0.019$ & $0.025 \pm 0.007$ & $0.012 \pm 0.006$ & $0.008 \pm 0.005$ \\
$f_{\{X_2\}}$ & $0.093 \pm 0.045$ & $0.074 \pm 0.040$ & $0.046 \pm 0.022$ & $0.034 \pm 0.013$ & $0.022 \pm 0.010$ & $0.016 \pm 0.008$ & $0.009 \pm 0.005$ & $0.007 \pm 0.004$ \\
$f_{\{X_3\}}$ & $0.116 \pm 0.112$ & $0.124 \pm 0.066$ & $0.059 \pm 0.048$ & $0.049 \pm 0.038$ & $0.034 \pm 0.036$ & $0.024 \pm 0.025$ & $0.022 \pm 0.020$ & $0.014 \pm 0.011$ \\
$f_{\{X_1,X_2\}}$ & $0.138 \pm 0.101$ & $0.135 \pm 0.071$ & $0.075 \pm 0.043$ & $0.068 \pm 0.041$ & $0.045 \pm 0.029$ & $0.031 \pm 0.024$ & $0.013 \pm 0.011$ & $0.013 \pm 0.008$ \\
\midrule
\multicolumn{9}{l}{\textit{Ridge}} \\
$f_{\{X_1\}}$ & $0.220 \pm 0.092$ & $0.131 \pm 0.060$ & $0.088 \pm 0.035$ & $0.059 \pm 0.023$ & $0.039 \pm 0.021$ & $0.030 \pm 0.012$ & $0.018 \pm 0.006$ & $0.013 \pm 0.004$ \\
$f_{\{X_2\}}$ & $0.278 \pm 0.067$ & $0.169 \pm 0.054$ & $0.083 \pm 0.020$ & $0.056 \pm 0.022$ & $0.040 \pm 0.014$ & $0.026 \pm 0.009$ & $0.016 \pm 0.004$ & $0.013 \pm 0.005$ \\
$f_{\{X_3\}}$ & $0.281 \pm 0.126$ & $0.202 \pm 0.085$ & $0.095 \pm 0.047$ & $0.067 \pm 0.039$ & $0.053 \pm 0.036$ & $0.034 \pm 0.023$ & $0.028 \pm 0.020$ & $0.017 \pm 0.011$ \\
$f_{\{X_1,X_2\}}$ & $1.000 \pm 0.000$ & $1.000 \pm 0.000$ & $1.000 \pm 0.000$ & $1.000 \pm 0.000$ & $1.000 \pm 0.000$ & $1.000 \pm 0.000$ & $1.000 \pm 0.000$ & $1.000 \pm 0.000$ \\
\midrule
\multicolumn{9}{l}{\textit{Random Forest}} \\
$f_{\{X_1\}}$ & $0.248 \pm 0.067$ & $0.137 \pm 0.027$ & $0.105 \pm 0.026$ & $0.089 \pm 0.038$ & $0.070 \pm 0.020$ & $0.051 \pm 0.018$ & $0.042 \pm 0.013$ & $0.032 \pm 0.009$ \\
$f_{\{X_2\}}$ & $0.387 \pm 0.086$ & $0.266 \pm 0.071$ & $0.135 \pm 0.039$ & $0.099 \pm 0.029$ & $0.063 \pm 0.014$ & $0.041 \pm 0.010$ & $0.027 \pm 0.003$ & $0.022 \pm 0.003$ \\
$f_{\{X_3\}}$ & $0.317 \pm 0.121$ & $0.237 \pm 0.077$ & $0.142 \pm 0.027$ & $0.117 \pm 0.023$ & $0.105 \pm 0.033$ & $0.081 \pm 0.027$ & $0.083 \pm 0.024$ & $0.062 \pm 0.018$ \\
$f_{\{X_1,X_2\}}$ & $0.816 \pm 0.252$ & $0.525 \pm 0.095$ & $0.405 \pm 0.134$ & $0.390 \pm 0.058$ & $0.358 \pm 0.083$ & $0.270 \pm 0.047$ & $0.234 \pm 0.054$ & $0.213 \pm 0.057$ \\
\midrule
\multicolumn{9}{l}{\textit{NGBoost}} \\
$f_{\{X_1\}}$ & $0.309 \pm 0.116$ & $0.237 \pm 0.064$ & $0.170 \pm 0.027$ & $0.151 \pm 0.031$ & $0.090 \pm 0.019$ & $0.067 \pm 0.012$ & $0.048 \pm 0.011$ & $0.036 \pm 0.012$ \\
$f_{\{X_2\}}$ & $0.246 \pm 0.036$ & $0.214 \pm 0.041$ & $0.167 \pm 0.032$ & $0.133 \pm 0.023$ & $0.078 \pm 0.024$ & $0.054 \pm 0.015$ & $0.034 \pm 0.010$ & $0.031 \pm 0.011$ \\
$f_{\{X_3\}}$ & $0.373 \pm 0.118$ & $0.368 \pm 0.062$ & $0.258 \pm 0.052$ & $0.199 \pm 0.056$ & $0.112 \pm 0.022$ & $0.093 \pm 0.029$ & $0.057 \pm 0.022$ & $0.042 \pm 0.018$ \\
$f_{\{X_1,X_2\}}$ & $0.640 \pm 0.126$ & $0.506 \pm 0.098$ & $0.409 \pm 0.096$ & $0.274 \pm 0.061$ & $0.220 \pm 0.033$ & $0.184 \pm 0.044$ & $0.167 \pm 0.047$ & $0.158 \pm 0.038$ \\
\midrule
\multicolumn{9}{l}{\textit{MLP}} \\
$f_{\{X_1\}}$ & $0.186 \pm 0.047$ & $0.106 \pm 0.039$ & $0.080 \pm 0.022$ & $0.056 \pm 0.020$ & $0.042 \pm 0.019$ & $0.034 \pm 0.013$ & $0.023 \pm 0.007$ & $0.016 \pm 0.005$ \\
$f_{\{X_2\}}$ & $0.211 \pm 0.052$ & $0.150 \pm 0.048$ & $0.094 \pm 0.021$ & $0.081 \pm 0.014$ & $0.075 \pm 0.017$ & $0.063 \pm 0.011$ & $0.050 \pm 0.010$ & $0.050 \pm 0.005$ \\
$f_{\{X_3\}}$ & $0.203 \pm 0.075$ & $0.173 \pm 0.064$ & $0.095 \pm 0.028$ & $0.062 \pm 0.022$ & $0.052 \pm 0.027$ & $0.039 \pm 0.017$ & $0.032 \pm 0.016$ & $0.025 \pm 0.010$ \\
$f_{\{X_1,X_2\}}$ & $1.036 \pm 0.049$ & $1.030 \pm 0.031$ & $0.527 \pm 0.210$ & $0.414 \pm 0.177$ & $0.388 \pm 0.061$ & $0.403 \pm 0.036$ & $0.378 \pm 0.036$ & $0.368 \pm 0.044$ \\
\bottomrule
\end{tabular}}
\end{table}

\begin{table}[t]
\caption{Relative aggregated error between estimated and oracle integrated pure effect values $\Phi_S$ (mean $\pm$ 1 std over 30 independent runs), normalized analogously to Tab.~\ref{tab:l2_full}. For most estimator--subset combinations, aggregated errors tend to be lower than the corresponding $L^2$ errors due to partial cancellation of pointwise errors upon integration. For oracle estimates of $f_{\{X_3\}}$ and $f_{\{X_1,X_2\}}$, the two metrics coincide because the synthetic model is separable in those terms: the Möbius estimate inherits the true time-shape exactly, and finite-sample error reduces to a scalar coefficient that contributes identically to both metrics.}
\label{tab:agg_full}
\centering
\resizebox{\textwidth}{!}{%
\begin{tabular}{llllllllll}
\toprule
Model / Effect & $n=50$ & $n=100$ & $n=250$ & $n=500$ & $n=1$k & $n=2$k & $n=5$k & $n=10$k \\
\midrule
\multicolumn{9}{l}{\textit{Oracle (true model)}} \\
$f_{\{X_1\}}$ & $0.117 \pm 0.072$ & $0.049 \pm 0.033$ & $0.047 \pm 0.033$ & $0.039 \pm 0.021$ & $0.025 \pm 0.019$ & $0.023 \pm 0.007$ & $0.012 \pm 0.006$ & $0.007 \pm 0.005$ \\
$f_{\{X_2\}}$ & $0.083 \pm 0.061$ & $0.059 \pm 0.038$ & $0.049 \pm 0.036$ & $0.031 \pm 0.017$ & $0.018 \pm 0.010$ & $0.013 \pm 0.010$ & $0.008 \pm 0.005$ & $0.007 \pm 0.004$ \\
$f_{\{X_3\}}$ & $0.116 \pm 0.112$ & $0.124 \pm 0.066$ & $0.059 \pm 0.048$ & $0.049 \pm 0.038$ & $0.034 \pm 0.036$ & $0.024 \pm 0.025$ & $0.022 \pm 0.020$ & $0.014 \pm 0.011$ \\
$f_{\{X_1,X_2\}}$ & $0.138 \pm 0.101$ & $0.135 \pm 0.071$ & $0.075 \pm 0.043$ & $0.068 \pm 0.041$ & $0.045 \pm 0.029$ & $0.031 \pm 0.024$ & $0.013 \pm 0.011$ & $0.013 \pm 0.008$ \\
\midrule
\multicolumn{9}{l}{\textit{Ridge}} \\
$f_{\{X_1\}}$ & $0.193 \pm 0.118$ & $0.109 \pm 0.071$ & $0.075 \pm 0.039$ & $0.049 \pm 0.036$ & $0.029 \pm 0.027$ & $0.022 \pm 0.016$ & $0.012 \pm 0.010$ & $0.010 \pm 0.007$ \\
$f_{\{X_2\}}$ & $0.217 \pm 0.106$ & $0.128 \pm 0.073$ & $0.075 \pm 0.031$ & $0.036 \pm 0.030$ & $0.032 \pm 0.029$ & $0.026 \pm 0.012$ & $0.017 \pm 0.008$ & $0.015 \pm 0.013$ \\
$f_{\{X_3\}}$ & $0.326 \pm 0.121$ & $0.193 \pm 0.129$ & $0.066 \pm 0.064$ & $0.074 \pm 0.043$ & $0.059 \pm 0.046$ & $0.036 \pm 0.025$ & $0.030 \pm 0.025$ & $0.015 \pm 0.016$ \\
$f_{\{X_1,X_2\}}$ & $1.000 \pm 0.000$ & $1.000 \pm 0.000$ & $1.000 \pm 0.000$ & $1.000 \pm 0.000$ & $1.000 \pm 0.000$ & $1.000 \pm 0.000$ & $1.000 \pm 0.000$ & $1.000 \pm 0.000$ \\
\midrule
\multicolumn{9}{l}{\textit{Random Forest}} \\
$f_{\{X_1\}}$ & $0.174 \pm 0.100$ & $0.083 \pm 0.049$ & $0.048 \pm 0.047$ & $0.039 \pm 0.033$ & $0.038 \pm 0.025$ & $0.038 \pm 0.027$ & $0.025 \pm 0.017$ & $0.025 \pm 0.017$ \\
$f_{\{X_2\}}$ & $0.284 \pm 0.212$ & $0.240 \pm 0.138$ & $0.140 \pm 0.067$ & $0.113 \pm 0.047$ & $0.071 \pm 0.035$ & $0.043 \pm 0.016$ & $0.012 \pm 0.012$ & $0.025 \pm 0.013$ \\
$f_{\{X_3\}}$ & $0.209 \pm 0.205$ & $0.176 \pm 0.164$ & $0.111 \pm 0.063$ & $0.108 \pm 0.076$ & $0.112 \pm 0.057$ & $0.077 \pm 0.046$ & $0.083 \pm 0.050$ & $0.069 \pm 0.042$ \\
$f_{\{X_1,X_2\}}$ & $0.699 \pm 0.592$ & $0.551 \pm 0.306$ & $0.410 \pm 0.380$ & $0.286 \pm 0.202$ & $0.338 \pm 0.329$ & $0.191 \pm 0.207$ & $0.155 \pm 0.103$ & $0.207 \pm 0.141$ \\
\midrule
\multicolumn{9}{l}{\textit{NGBoost}} \\
$f_{\{X_1\}}$ & $0.180 \pm 0.192$ & $0.106 \pm 0.061$ & $0.072 \pm 0.062$ & $0.068 \pm 0.040$ & $0.054 \pm 0.035$ & $0.032 \pm 0.026$ & $0.032 \pm 0.017$ & $0.022 \pm 0.018$ \\
$f_{\{X_2\}}$ & $0.111 \pm 0.061$ & $0.186 \pm 0.152$ & $0.140 \pm 0.075$ & $0.095 \pm 0.061$ & $0.080 \pm 0.045$ & $0.038 \pm 0.027$ & $0.025 \pm 0.016$ & $0.032 \pm 0.022$ \\
$f_{\{X_3\}}$ & $0.246 \pm 0.238$ & $0.284 \pm 0.179$ & $0.224 \pm 0.115$ & $0.145 \pm 0.115$ & $0.079 \pm 0.041$ & $0.059 \pm 0.055$ & $0.044 \pm 0.030$ & $0.034 \pm 0.025$ \\
$f_{\{X_1,X_2\}}$ & $0.394 \pm 0.272$ & $0.425 \pm 0.224$ & $0.240 \pm 0.185$ & $0.204 \pm 0.105$ & $0.217 \pm 0.139$ & $0.114 \pm 0.086$ & $0.125 \pm 0.081$ & $0.131 \pm 0.068$ \\
\midrule
\multicolumn{9}{l}{\textit{MLP}} \\
$f_{\{X_1\}}$ & $0.141 \pm 0.096$ & $0.073 \pm 0.061$ & $0.057 \pm 0.038$ & $0.039 \pm 0.029$ & $0.030 \pm 0.027$ & $0.025 \pm 0.020$ & $0.016 \pm 0.010$ & $0.010 \pm 0.010$ \\
$f_{\{X_2\}}$ & $0.152 \pm 0.123$ & $0.131 \pm 0.081$ & $0.125 \pm 0.042$ & $0.091 \pm 0.051$ & $0.112 \pm 0.026$ & $0.096 \pm 0.024$ & $0.076 \pm 0.021$ & $0.072 \pm 0.017$ \\
$f_{\{X_3\}}$ & $0.137 \pm 0.125$ & $0.161 \pm 0.135$ & $0.069 \pm 0.054$ & $0.048 \pm 0.032$ & $0.053 \pm 0.044$ & $0.035 \pm 0.031$ & $0.033 \pm 0.018$ & $0.026 \pm 0.017$ \\
$f_{\{X_1,X_2\}}$ & $1.232 \pm 0.192$ & $1.162 \pm 0.132$ & $0.440 \pm 0.257$ & $0.325 \pm 0.221$ & $0.372 \pm 0.135$ & $0.361 \pm 0.048$ & $0.384 \pm 0.046$ & $0.374 \pm 0.035$ \\
\bottomrule
\end{tabular}}
\end{table}

\textbf{Time-resolved effect curves and time-aggregated effects.} The top row of Fig.~\ref{fig:effect_comparison} shows representative time-resolved pure effect trajectories at $n = 1{,}000$. All non-linear models closely track the analytical ground truth for the three main effects. For the interaction $f_{\{X_1,X_2\}}$, the peak at $t \approx 5$h is recovered in shape but with larger pointwise variance across models, consistent with the higher $L^2$ errors observed at this sample size. The bottom row of Fig.~\ref{fig:effect_comparison} shows time-aggregated effects $\Phi_S = \int |f_S(t)|\,\mathrm{d}t$ for a representative run. All models reproduce the correct ordering $\Phi_{\{X_1\}} > \Phi_{\{X_2\}} > \Phi_{\{X_3\}} > \Phi_{\{X_1,X_2\}}$ and closely match the analytical values, with the largest relative deviations for the interaction term.

\begin{figure}[ht]
    \centering
    \includegraphics[width=\textwidth]{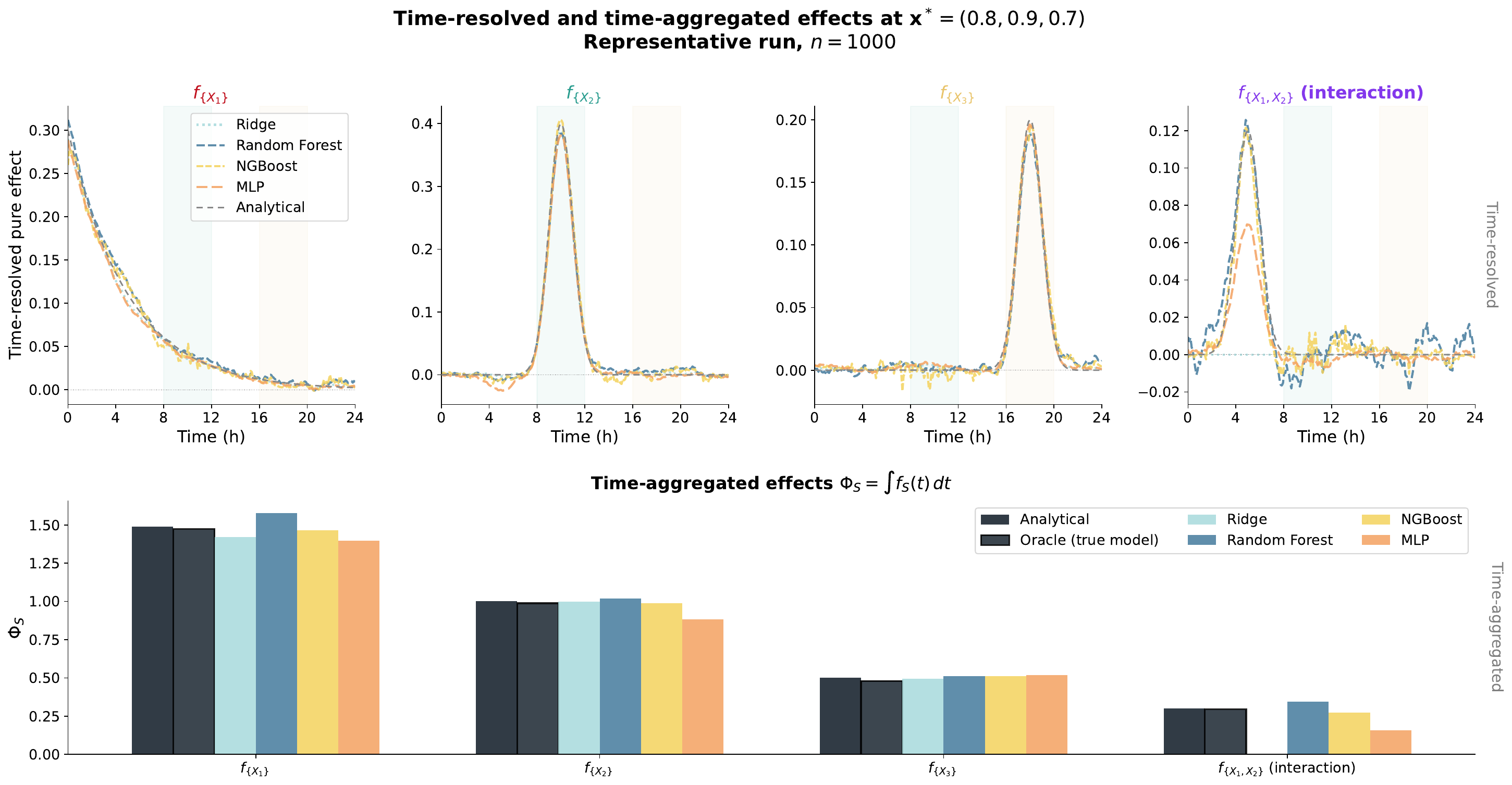}
    \caption{Time-resolved (\emph{row~1}) and time-aggregated (\emph{row~2}) $\mathcal{H}$-FD pure prediction effects at $\mathbf{x}^* = (0.8, 0.9, 0.7)$ for a representative run at $n = 1{,}000$. The time-resolved curves show estimated effects overlaid with the analytical ground truth (grey dashed). Main effects are closely recovered by all models; the interaction $f_{\{X_1,X_2\}}$ shows greater variability, and the MLP underestimates its time-aggregated value.}
    \label{fig:effect_comparison}
\end{figure}

\textbf{Sobol index recovery.} Fig.~\ref{fig:sobol_recovery} validates Thm.~\ref{thm:sobol_indices}: under the
constant kernel $K(t,s) = 1$ and sensitivity input behavior, $\mathcal{H}$-FD recovers the classical time-resolved and time-aggregated Sobol indices. The oracle estimator matches the analytical ground truth to within numerical precision across the full time domain. Non-linear models recover the indices accurately in the regions where each feature dominates, the exponential decay window for $X_1$, the peak at $t \approx 10$h for $X_2$, and the peak at $t \approx 18$h for $X_3$, but exhibit noisy deviations in near-zero regions where no feature dominates (row~1). These deviations arise from normalization: the time-resolved index divides by a near-zero denominator when all pure variance effects are small, amplifying estimation noise. This is an artifact of the ratio form of the Sobol index rather than a failure of effect recovery. Time-aggregated indices (row~2) are correctly recovered by all models, confirming that the integrated quantities are robust to pointwise normalization noise.

\begin{figure}[ht]
    \centering
    \includegraphics[width=\textwidth]{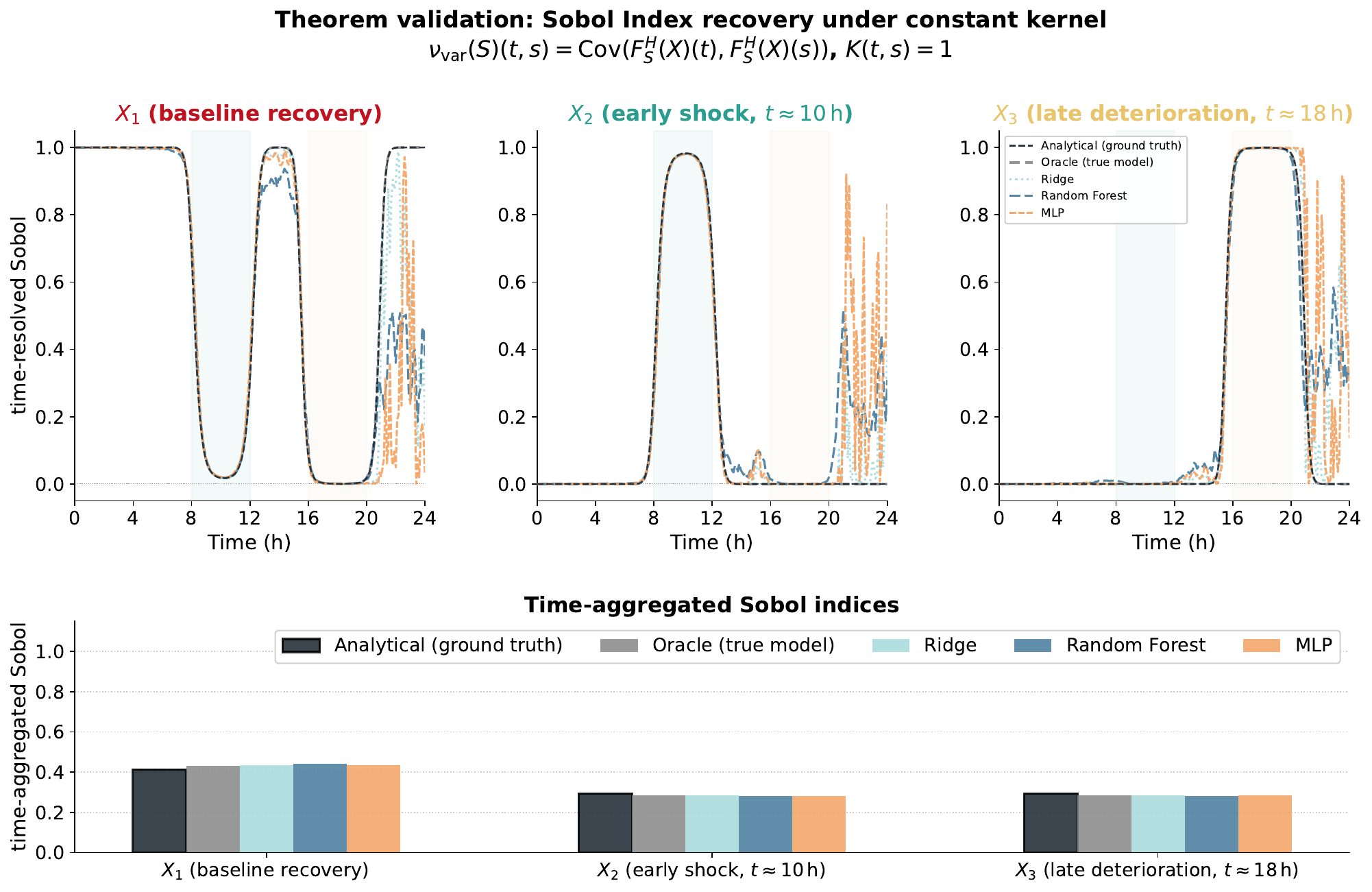}
    \caption{Time-resolved (\emph{row~1}) and time-aggregated (\emph{row~2}) Sobol indices recovered by $\mathcal{H}$-FD under the constant kernel and sensitivity input behavior, compared against the analytical ground truth.}
    \label{fig:sobol_recovery}
\end{figure}

\textbf{Aggregated ranking preservation.} Thm.~\ref{thm:kernel_effect} establishes that the time-aggregated effect under $\mathcal{B}_{\mathrm{out}}$ admits the weighted-integral representation
\[
\int_{\mathcal{T}}(\mathcal{B}_{\mathrm{out}}\,\nu_{\mathrm{inp}})(S)(t)\,\mathrm{d}t
= \int_{\mathcal{T}} w_K(s)\,\nu_{\mathrm{inp}}(S)(s)\,\mathrm{d}s,
\qquad
w_K(s) := \int_{\mathcal{T}} K(t,s)\,\mathrm{d}t.
\]
For the prediction input behavior, $\mathcal{B}_{\mathrm{inp}}$ is linear in
$F^{\mathcal{H}}$, so if
$\nu_{\mathrm{inp}}(\{i\})(s) \geq \nu_{\mathrm{inp}}(\{j\})(s)$ for all
$s \in \mathcal{T}$, the integrated ordering is preserved for any non-negative kernel $K$: the kernel re-weights the time axis without reversing a pointwise-dominant feature. This sufficient condition holds in the three examples of main-text Fig.~\ref{fig:condensed_kernel_guidance}, where the time-aggregated rankings $\Phi_{X_1} > \Phi_{X_2} > \Phi_{X_3}$ (ICU), $\Phi_{X_2} > \Phi_{X_3} > \Phi_{X_1}$ (price pulse), and $\Phi_{X_1} \gg \Phi_{X_2}$ (periodic medication) are preserved across all kernels under the prediction input behavior. By contrast, Remark 1 (\ref{rmk:kernel_effect}) establishes that ranking invariance does not generally hold for nonlinear input behaviors. Although the same weighted-integral representation applies, different kernels may induce different feature rankings under sensitivity or risk input behaviors. We illustrate this directly for the ICU model in Fig.~\ref{fig:ranking_games}: the prediction input behavior preserves the ordering $X_1 > X_2 > X_3$ across identity, OU, and correlation kernels, while the sensitivity and risk input behaviors produce a different ranking under different kernels $X_1 > X_3 > X_2$.

\begin{figure}[ht]
    \centering
    \includegraphics[width=\textwidth]{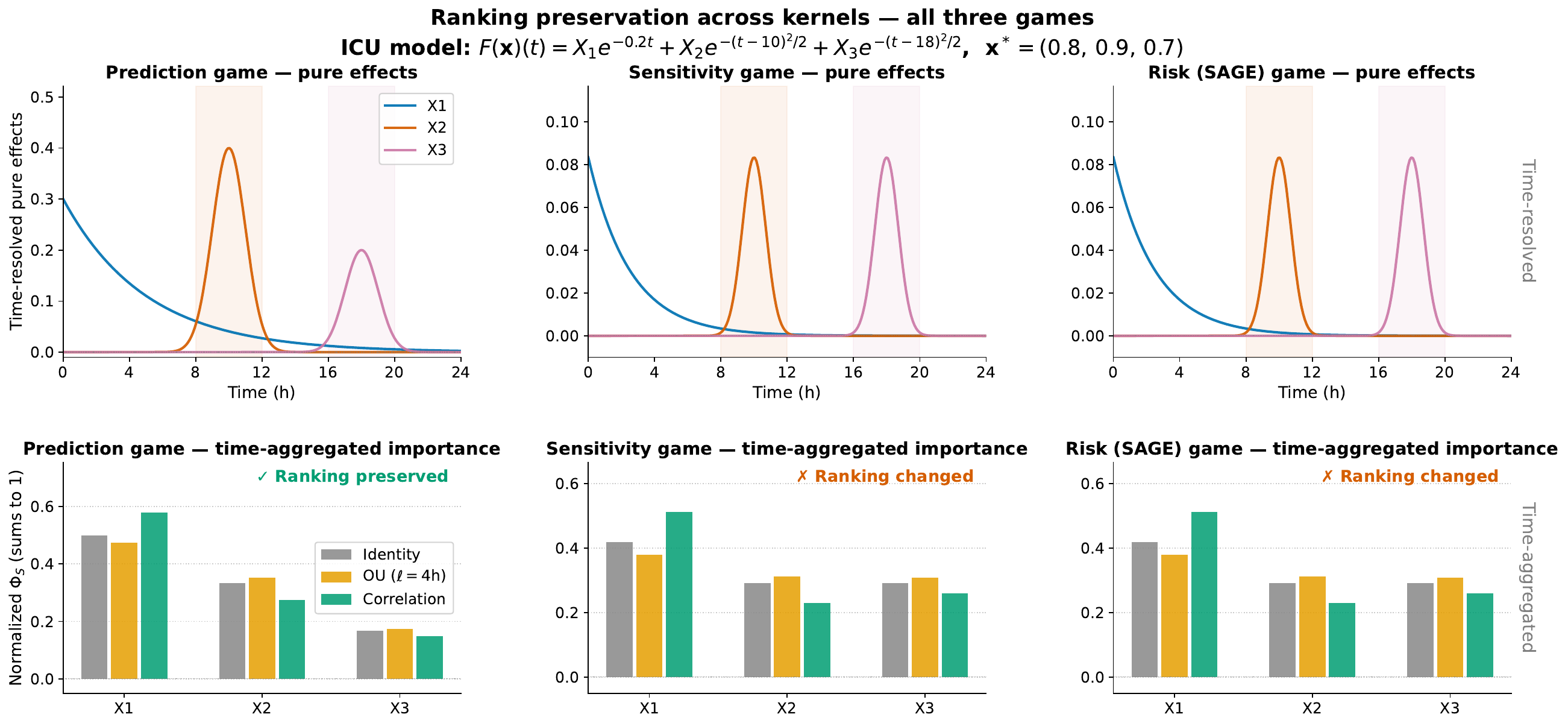}
    \caption{Ranking preservation across input behaviors for the ICU model, under identity, OU ($\ell = 4$h), and correlation kernels. (\emph{row~1}) Time-resolved pure effects; (\emph{row~2}) time-aggregated normalized importance. Under the prediction input behavior (\emph{col.~1}), $X_1 > X_2 > X_3$ is preserved across all kernels, consistent with Thm.~\ref{thm:kernel_effect}. Under the sensitivity and risk input behaviors (\emph{col.~2,3}), rankings change across kernels: the correlation kernel reweights the squared per-feature effects so that $X_3$ overtakes $X_2$, while identity and OU preserve $X_2 > X_3$. Pointwise dominance of individual effects is not preserved under the quadratic input behaviors of variance and MSE, as stated in Remark 1 (\ref{rmk:kernel_effect}).}
    \label{fig:ranking_games}
\end{figure}

\clearpage
\subsection{Kernel Guidance}\label{app:kernel_guidance}

This appendix provides supplementary details for the kernel guidance discussion in the main text (Sec.~\ref{sec:kernel_guidance}, Fig.~\ref{fig:condensed_kernel_guidance}).

\textbf{Kernel role across input behaviors.} For input behaviors defined pointwise over the output domain---prediction $\nu_{\mathrm{inp}}^{\mathrm{pred}}(S)(t) = F_S^{\mathcal{H}}(\mathbf{x})(t)$ and risk $\nu_{\mathrm{inp}}^{\mathrm{risk}}(S)(t)$---the kernel introduces dependencies between output locations: it specifies how feature effects at location $t$ incorporate information from other locations $s \in \mathcal{T}$. For variance-based explanations, the input behavior already captures dependencies between output locations through the covariance surface
\[
\nu_{\mathrm{inp}}^{\mathrm{sens}}(S)(t,s)
= \mathrm{Cov}\!\big(F_S^{\mathcal{H}}(\mathbf{X})(t),\,F_S^{\mathcal{H}}(\mathbf{X})(s)\big),
\]
and the kernel determines how these covariance contributions are aggregated across time pairs. Classical functional Sobol indices correspond to the constant kernel $K(t,s) = 1$, which assigns equal weight to all covariance terms \citep{gamboa2014sensitivity,alexanderian2020variance}.

\textbf{Kernel selection.}
Tab.~\ref{tab:kernel_guidance} summarizes the temporal assumption and intended use case for a selected set of useful kernels. The appropriate choice depends on the research question, output structure, and domain knowledge; there is no
universally correct kernel. For outputs mixing features with different temporal characteristics, kernels may be selected per-feature or combined.

\begin{table}[t]
\caption{\textbf{Kernel selection guide} for $\mathcal{H}$-FD. The appropriate kernel depends on the research question, output structure, and domain knowledge; there is no universally correct choice. For outputs mixing features with different temporal characteristics, kernels may be selected per-feature or combined.}
\label{tab:kernel_guidance}
\centering
\small
\begin{tabular}{p{2.8cm}p{3.2cm}p{4.0cm}p{2.6cm}}
\toprule
Kernel & Temporal assumption & Use when & Domains \\
\midrule
$\delta(t-s)$ (identity)
& Locations independent
& Instantaneous effects; no smoothing desired
& Any; default \\[4pt]
$1$ (constant)
& All interactions equally weighted
& Trajectory-level importance; Sobol indices
& Sensitivity analysis \\[4pt]
$\mathrm{Corr}(F_t, F_s)$ (correlation)
& Weighted by output co-movement
& Features drive distinct trajectory shapes
& Finance, electricity consumption \\[4pt]
$\exp(-|t-s|/\ell)$ (OU)
& Mean-reverting, local smoothing
& Noisy outputs; sustained influence over a window
& Biomedical, intraday volatility \\[4pt]
$\exp(-(t-s)^2/2\sigma^2)$ (Gaussian)
& Smooth symmetric neighbourhoods
& Smooth outputs; no causal ordering required
& Temperature, demand \\[4pt]
$\rho^{|t-s|}$ (AR-type)
& Discrete autoregressive dependence
& Strong lag structure; discrete-time predictions
& Economic time series \\[4pt]
$\exp(-|t-s|/\ell)\,\mathbf{1}_{t \geq s}$ (causal)
& Strict temporal ordering
& Event-driven; causal interpretability required
& Algorithmic trading \\[4pt]
$\exp(-2\sin^2(\pi(t-s)/\tau)/\ell^2)$ (periodic)
& Known periodicity; apply per-feature
& Multiple full periods; genuinely recurring influence
& Circadian, seasonal \\
\bottomrule
\end{tabular}
\end{table}

\clearpage
\subsection{SPY Intraday Volatility}
\label{app:intraday_spy_volatility}

This appendix provides full experimental details and supplementary figures for the SPY intraday volatility case study summarized in the main text (Sec.\ref{sec:real-world}, Fig.~\ref{fig:SPY_main}).

\textbf{Data.} We use one-minute OHLCV (open, high, low, close, volume) bar data for the SPDR S\&P~500 ETF Trust (SPY) sourced from Polygon.io (\url{https://polygon.io}) \citep{polygon2024}, covering 560 trading days from January 3, 2022 to April 30, 2024. The New York Stock Exchange regular session runs 09:30--16:00 ET; after removing the opening and closing auction bars we retain $T = 78$ five-minute intervals per day. The per-interval target is the \emph{absolute log-return}
\[
  y_t \;=\; \bigl|\log(c_t/c_{t-1})\bigr|, \qquad t = 1,\ldots,78,
\]
where $c_t$ denotes the close price of the $t$-th five-minute bar. Absolute log-returns are a standard proxy for intraday volatility at the five-minute frequency in the high-frequency literature \citep{andersen1998answering}. We subtract the cross-day mean at each interval (the \emph{diurnal pattern}) so that the model is trained to explain day-to-day deviations from the typical U-shaped volatility curve (elevated at the open due to overnight news digestion, low through midday, and elevated again at the close due to position squaring), rather than the U-shape itself. The diurnal pattern is estimated on the training split only and then subtracted from both splits.

\textbf{Pre-market features.} Six scalar covariates are constructed from information available before the 09:30 open:
 
\begin{center}
\small
\begin{tabularx}{\textwidth}{lXl}
\toprule
Feature & Description & Source \\
\midrule
\texttt{vix\_prev}      & Prior business day's closing VIX level
                        & CBOE via Yahoo Finance \\
\texttt{overnight\_ret} & Log-return from prior close to 09:30 open
                        & Polygon.io \\
\texttt{ann\_indicator} & Binary: 1 if a scheduled macro announcement
                          falls on this day$^\dagger$
                        & Federal Reserve / BLS / BEA \\
\texttt{day\_of\_week}  & Integer $\in\{0,\ldots,4\}$ (Mon--Fri)
                        & -- \\
\texttt{trailing\_rv}   & Mean daily absolute log-return over the prior
                          5 trading days
                        & Polygon.io \\
\texttt{month}          & Integer $\in\{1,\ldots,12\}$
                        & -- \\
\bottomrule
\end{tabularx}
\end{center}
\vspace{-\baselineskip}
{\footnotesize $^\dagger$\texttt{ann\_indicator} is set to 1 on FOMC meeting dates, CPI release dates, and Non-Farm Payroll (NFP) release dates; all dates are hardcoded from official Federal Reserve, BLS, and BEA schedules.}

\smallskip\noindent
All continuous features (\texttt{vix\_prev}, \texttt{overnight\_ret}, \texttt{trailing\_rv}) are standardized to zero mean and unit variance using statistics computed on the training split only.

\textbf{Model.} We use a single chronological train/test split to respect temporal ordering and prevent look-ahead leakage.
 
\begin{center}
\small
\begin{tabular}{lcrr}
\toprule
Split & Date range & \# Days & \% \\
\midrule
Train & 2022-01-03 -- 2023-09-29 & 448 & 80\% \\
Test  & 2023-10-02 -- 2024-04-30 & 112 & 20\% \\
\bottomrule
\end{tabular}
\end{center}
 
We fit a \emph{multivariate random forest}: a single
\texttt{RandomForestRegressor} from \texttt{scikit-learn}
\citep{scikitlearn} whose targets are the $T = 78$ diurnal-adjusted volatility values stacked into a single matrix $\mathbf{Y} \in \mathbb{R}^{N \times T}$.  Each tree predicts the full trajectory jointly, which naturally captures temporal correlations in the residuals without requiring an explicit sequence model.
 
\textbf{Hyperparameters.}
\begin{center}
\begin{tabular}{ll}
\toprule
Hyperparameter & Value \\
\midrule
\texttt{n\_estimators}      & 300 \\
\texttt{max\_features}      & \texttt{"sqrt"} (default) \\
\texttt{min\_samples\_leaf} & 1 (default) \\
\texttt{min\_samples\_split}& 2 (default) \\
\texttt{max\_depth}         & None (fully grown) \\
\texttt{bootstrap}          & True (default) \\
\texttt{random\_state}      & 42 \\
\texttt{n\_jobs}            & $-1$ (all available cores) \\
\bottomrule
\end{tabular}
\end{center}
 
Hyperparameters were fixed \emph{a priori}; no grid search or
cross-validation was performed.  This choice was deliberate: our goal is to demonstrate the framework's explanatory properties rather than optimize predictive accuracy, and heavy tuning on 560 days of financial data risks over-fitting to the specific market regime covered.  Random
forests with 300 trees and default splitting rules are a well-understood baseline that generalizes reliably in low-sample tabular settings\citep{breiman2001random}. The model achieves a trajectory-level $R^2 = 0.120$ on the held-out
test split ($1 -\mathrm{SS}_\mathrm{res}/\mathrm{SS}_\mathrm{tot}$ over all $112 \times 78$ predictions). This modest value is expected:
the diurnal pattern has been subtracted prior to modelling, leaving only day-to-day deviations from the typical U-shape as the target.

\textbf{Cooperative game set-up.} The full $2^p$-point Möbius transform is computed exactly over the Boolean lattice for all $S \subseteq [p]$. Pure main effects, partial (Shapley) effects, and full effects are then derived in closed form from the Möbius values. We use marginal (interventional) imputation throughout, sampling missing feature values from the training background set. All Monte Carlo estimators use a fixed random seed (seed $=42$) for reproducibility.

The three input behaviors are estimated as follows.

\emph{Local prediction.}
For the High-VIX Announcement profile $\mathbf{x}^*$ and a coalition $S$,
\[
  \nu_{\mathrm{inp}}^{\mathrm{pred}}(S)(t)
  = \mathbb{E}_{\mathbf{X}_{-S}}\bigl[F(\mathbf{x}^*_S, \mathbf{X}_{-S})(t)\bigr],
\]
estimated with $n_{\mathrm{sample}} = 200$ background draws per coalition. Pure effects yield per instance PDP equivalents, and partial effects are interventional SHAP values. The PDP-style plots in Fig.~\ref{fig:SPY_fig3} aggregate per-instance local prediction effects over $N_{\mathrm{PDP}} = 120$ profiles drawn by stratified sampling from the training set (one instance per month plus random fill).

\emph{Global sensitivity.}
\[
  \nu_{\mathrm{inp}}^{\mathrm{sens}}(S)(t,s)
  = \mathrm{Cov}\bigl(F_S^{\mathcal{H}}(\mathbf{X})(t),\,
                       F_S^{\mathcal{H}}(\mathbf{X})(s)\bigr),
\]
estimated by nested Monte Carlo: an outer sample of $n_{\mathrm{outer}} = 100$ draws of $\mathbf{X}_S$ from the training distribution, and for each outer draw an inner sample of $n_{\mathrm{inner}} = 100$ draws of $\mathbf{X}_{-S}$ for the conditional expectation. Pure effects yield closed Sobol indices and full effects yield total Sobol indices (Thm.~\ref{thm:sobol_indices}); partial effects are the corresponding Shapley sensitivity indices.

\emph{Global risk.}
The risk effect is defined as the loss \emph{reduction} relative to the marginal-only predictor $\bar{F}(t) = \mathbb{E}[F(\mathbf{X})](t)$:
\[
  \nu_{\mathrm{inp}}^{\mathrm{risk}}(S)(t)
  = \mathbb{E}\bigl[(Y(t) - \bar{F}(t))^2\bigr]
    - \mathbb{E}_{(\mathbf{X},Y)}\Bigl[\bigl(Y(t)
       - \mathbb{E}_{\mathbf{X}_{-S}}[F(\mathbf{X}_S, \mathbf{X}_{-S})(t)]
       \bigr)^2\Bigr],
\]
estimated by the same nested Monte Carlo scheme with $n_{\mathrm{outer}} = n_{\mathrm{inner}} = 100$. This sign convention matches the SAGE / PFI literature, in which higher importance corresponds to larger reduction in expected loss; pure effects are also set to be non-negative by construction. With $p = 6$ features there are $2^p = 64$ coalitions per input behavior.
 
\textbf{Kernels.} Two kernel choices are compared throughout.
 
\begin{itemize}
  \item \textbf{Identity kernel} $K = I$.  Treats each time step independently.
  \item \textbf{Feature-specific kernel.}  Rather than a single kernel for all features, we assign each feature a kernel that reflects its expected temporal structure.  For all features except \texttt{ann\_indicator}, we use an Ornstein--Uhlenbeck (OU) kernel,
    \[
      [K_{\mathrm{OU}}]_{st}
        = \exp\!\Bigl(-\tfrac{|t - s|}{\ell}\Bigr),
    \]
    with length-scale $\ell = 8$ (in units of five-minute bars, equivalent to 40 minutes), encoding smooth temporal decay. For \texttt{ann\_indicator}, whose effect is expected to be localized to a specific event time with no pre-event leakage, we use a causal kernel,
    \[
      [K_{\mathrm{causal}}]_{st}
        = \exp\!\Bigl(-\tfrac{t - s}{\ell}\Bigr)\,\mathbf{1}[t \geq s],
    \]
    with the same length-scale $\ell = 8$.  This enforces the
    interpretation that the announcement at time $s$ can only affect attributions at $t \geq s$.  The length-scale $\ell = 8$ was chosen to match the typical half-life of a volatility spike following a scheduled macro release (approximately 30--60 minutes in the high-frequency literature \citep{andersen1998answering}), while remaining interpretable on
    the 78-bar trading day grid.  Both kernels are row-normalized before application so that each time step receives a unit-mass weighted average of the effect trajectory.
\end{itemize}

\textbf{Profile selection.} The High-VIX Announcement profile studied in the main text (July 13, 2022) is selected as the median instance among all days satisfying \texttt{ann\_indicator}~$= 1$ and
\texttt{vix\_prev}~$\geq q_{75}$, where $q_{75}$ is the 75th percentile of \texttt{vix\_prev} over the training set (22 days match these criteria). The profile has feature values \texttt{vix\_prev}~$= 27.29$, \texttt{overnight\_ret}~$= -0.0153$, \texttt{ann\_indicator}~$= 1$,
\texttt{day\_of\_week}~$= 2$ (Wednesday), \texttt{trailing\_rv}~$= 8.18 \times 10^{-4}$, and \texttt{month}~$= 7$ (July).

\textbf{Supplementary figures.} Figs.~\ref{fig:SPY_fig1}--\ref{fig:SPY_fig4} provide the full set of results for the SPY experiment. All panels use the following color coding: \texttt{vix\_prev} (blue), \texttt{ann\_indicator} (orange), \texttt{overnight\_ret} (green), \texttt{trailing\_rv} (brown), \texttt{day\_of\_week} (purple), \texttt{month} (red). Solid lines correspond to the OU kernel; dashed lines to the causal kernel (used for \texttt{ann\_indicator} in the OU + causal kernel rows). The dashed vertical grey line marks 14:00 ET, the time of the Beige Book release on July~13,~2022.

Fig.~\ref{fig:SPY_fig1} shows global sensitivity and risk effects, capturing how \texttt{vix\_prev} contributes to output variance and forecast error across the full data distribution rather than at a single instance. \texttt{vix\_prev} dominates both quantities across all effect types and both kernels, consistent with the well-established role of implied volatility as a forward-looking risk signal at the population level. Under the identity kernel its effect trajectory is noisy and roughly flat across the trading day, while the OU kernel reveals a smooth U-shape, reflecting that variance and forecast error are systematically elevated at the open and close. The gap between pure and full importances is largest for \texttt{vix\_prev}, indicating that a substantial portion of its global contribution operates through interactions with other features rather than as a standalone main effect.
 
\begin{figure}[ht]
  \centering
  \includegraphics[width=\textwidth]{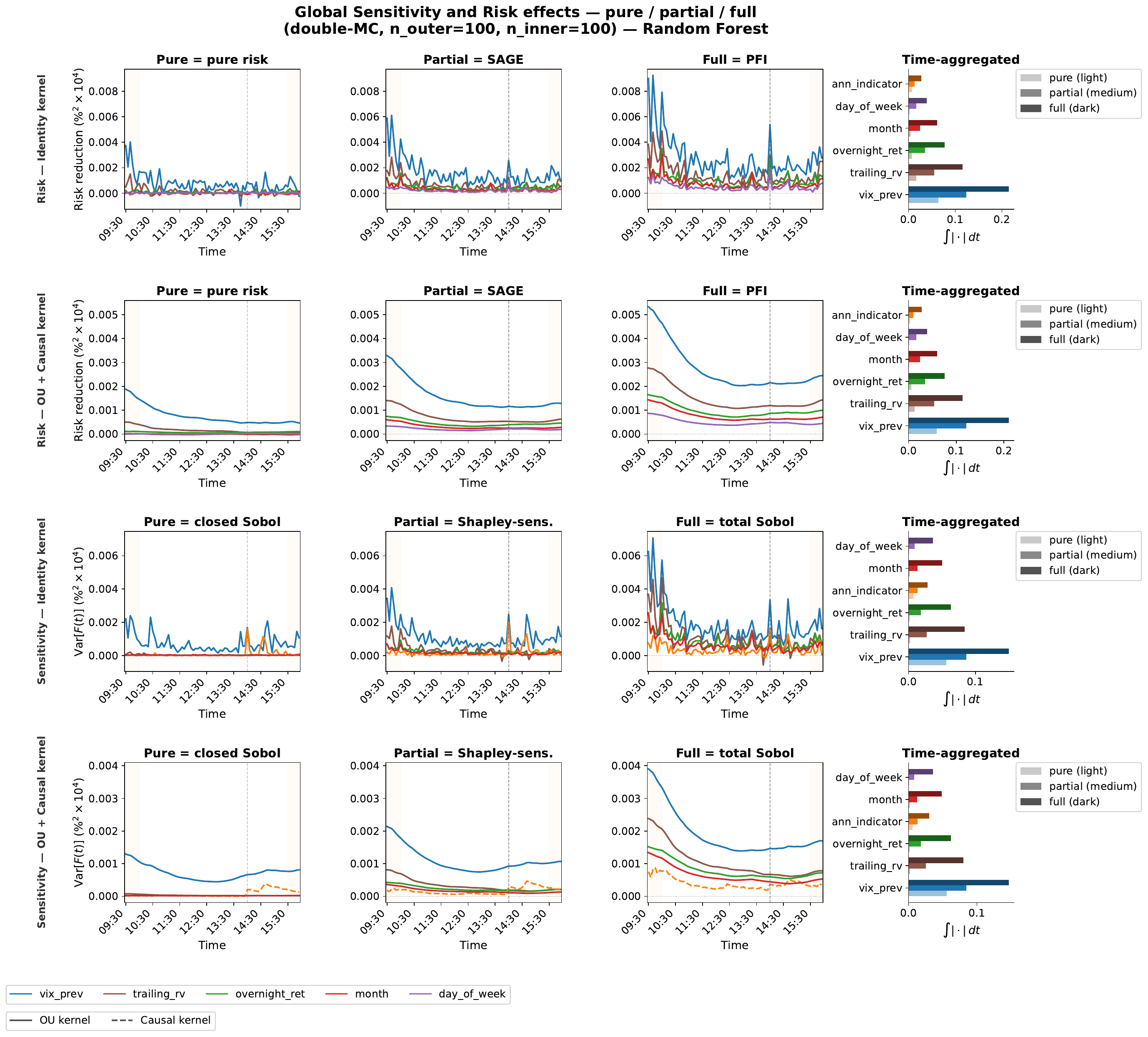}
    \caption{\textbf{Global sensitivity and risk effects --- pure / partial / full.} Effects are computed by nested Monte Carlo ($n_{\mathrm{outer}} = n_{\mathrm{inner}} = 100$) directly from the global input behaviors $\nu_{\mathrm{inp}}^{\mathrm{sens}}$ and $\nu_{\mathrm{inp}}^{\mathrm{risk}}$. (\emph{rows~1--2})  Risk behavior (loss reduction in $\%^2 \times 10^{-4}$) under the identity kernel (\emph{row~1}) and
    feature-specific kernel (\emph{row~2}); (\emph{rows~3--4}) sensitivity input behavior under the identity (\emph{row~3}) and feature-specific (\emph{row~4}) kernels. (\emph{cols.~1--3}) Pure $=$ closed Sobol / pure risk (col.~1), partial $=$ Shapley-sensitivity / SAGE (\emph{col.~2}),
    full $=$ total Sobol / PFI (\emph{col.~3}). (\emph{col.~4}) Time-aggregated bar chart for pure,
    partial, and full importances.}
  \label{fig:SPY_fig1}
\end{figure}

Fig.~\ref{fig:SPY_fig2} shows local prediction effects for the High-VIX Announcement profile. Under the identity kernel, pure effects are noisy and difficult to interpret. The feature-specific kernel clarifies the structure: \texttt{vix\_prev} produces a smooth positive baseline effect
throughout the day, while the causal kernel localizes the effect of \texttt{ann\_indicator} sharply after 14:00 with no pre-event leakage. The partial-to-pure ratio of $1.61\times$ for \texttt{ann\_indicator} indicates that interaction with other features substantially amplifies its attributed effect.

Fig.~\ref{fig:SPY_fig4} confirms that the \texttt{vix\_prev}~$\times$~\texttt{ann\_indicator} pair is the strongest pairwise interactions in terms of time-aggregated effects. Under the identity kernel this manifests as a sharp spike at 14:00. The causal kernel smooths and causally localizes this spike, revealing the regime shift described in the main text: mild redundancy between the two features before the announcement, where both reflect elevated expected volatility, followed by strong synergy afterwards, where the release amplifies the impact of prior market fear. A second interaction, \texttt{vix\_prev}~$\times$~\texttt{overnight\_ret}, is comparable in time-aggregated magnitude but lacks a sharp event-locked profile; remaining pairs are noticeably smaller. Together this confirms that the \texttt{vix\_prev}--\texttt{ann\_indicator} interaction is the structurally dominant \emph{event-driven} higher-order effect on this profile.

\begin{figure}[p]
  \centering
  \includegraphics[width=\textwidth, height=0.44\textheight, keepaspectratio]{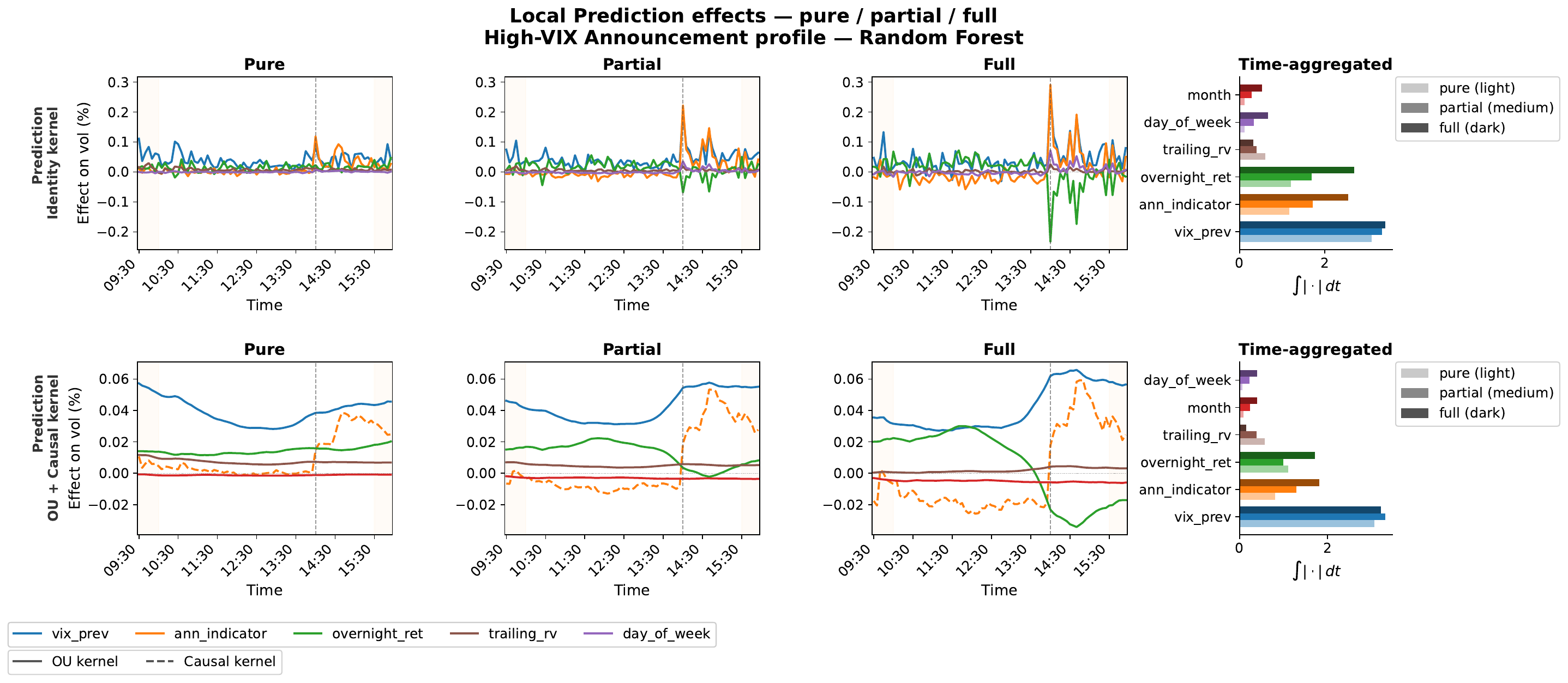}
  \caption{\textbf{Local prediction effects --- pure / partial / full.} Instance: July~13,~2022 (\texttt{vix\_prev}~$=27.3$, \texttt{ann\_indicator}~$=1$), selected as the median among days with \texttt{ann\_indicator}~$=1$ and \texttt{vix\_prev}~$\geq q_{75}$. (\emph{Row~1}) Identity kernel; (\emph{Row~2}) feature-specific kernel (OU for all features except \texttt{ann\_indicator}; causal for \texttt{ann\_indicator}, shown dashed). (\emph{cols.~1--3}) pure, partial, full. (\emph{col.~4}) Time-aggregated bar chart for pure, partial, and full effects.}
  \label{fig:SPY_fig2}
\end{figure}

\begin{figure}[p]
  \centering
  \includegraphics[width=\textwidth, height=0.44\textheight, keepaspectratio]{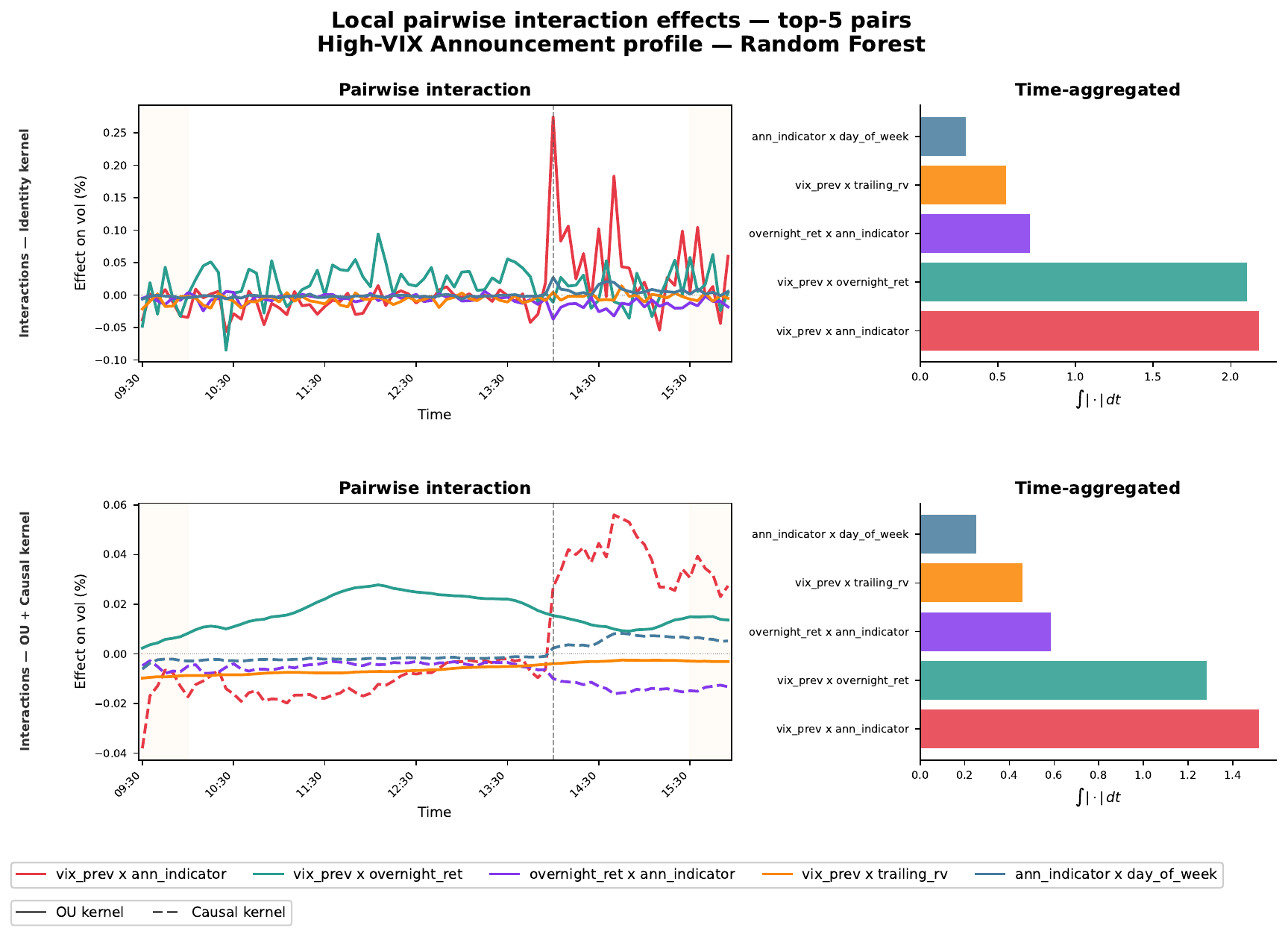}
  \caption{\textbf{Local pairwise interaction effects ---
    top-5 pairs.} Same instance as Fig.~\ref{fig:SPY_fig2}. Top-5 pairs
    are ranked by time-integrated OU-kernel importance. (\emph{row~1}) Identity kernel; (\emph{row~2}) Feature-specific kernel (causal for
    pairs involving \texttt{ann\_indicator}, shown dashed; OU otherwise).
    (\emph{col.~1}) Interaction trajectories over the trading day; (\emph{col.~2}) time-aggregated bar chart.
    The \texttt{vix\_prev}~$\times$~\texttt{ann\_indicator} pair
    is the strongest interaction under both kernels.}
  \label{fig:SPY_fig4}
\end{figure}

Fig.~\ref{fig:SPY_fig3} shows that \texttt{vix\_prev} exhibits a clear monotone relationship with predicted volatility deviations under the OU kernel: higher prior VIX is associated with larger positive effects across the full trading day, and this relationship is consistent across all five selected time points. Under the identity kernel the same
relationship is present but noisier, with greater variability across time points. For \texttt{overnight\_ret} the time-aggregated effect is near zero, suggesting that while the overnight return may matter for specific instances, its effect averages out across the training distribution.
 
\begin{figure}[ht]
  \centering
  \includegraphics[width=\textwidth]{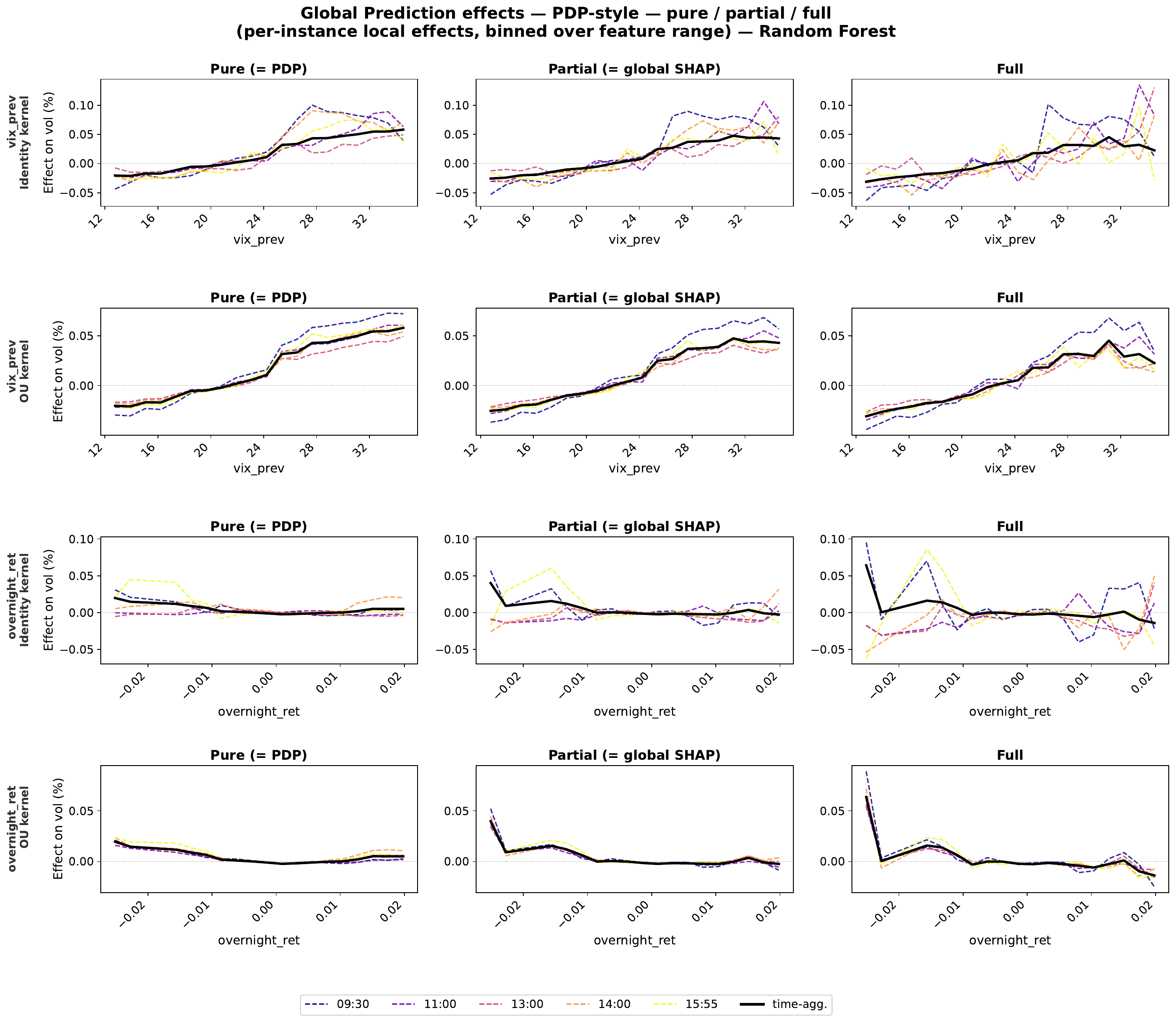}
  \caption{\textbf{Global prediction effects --- PDP-style.} Per-instance pure, partial, and full local prediction effects, computed for $N_{\mathrm{PDP}} = 120$ profiles drawn by stratified sampling from the training set, are binned along the feature axis ($N_{\mathrm{bins}} = 20$) for the two highest-importance features by time-aggregated OU-kernel global prediction behavior (\texttt{vix\_prev} and \texttt{overnight\_ret}); (\emph{rows~1--2}) \texttt{vix\_prev} under the identity (\emph{row~1}) and OU (\emph{row~2}) kernels; (\emph{rows~3--4}) \texttt{overnight\_ret} under identity (\emph{row~3}) and OU (\emph{row~4}) kernels. Coloured dashed lines show the mean effect at five selected time points (09:30, 11:00, 13:00, 14:00, 15:55); the solid black line is the time-aggregated mean.}
  \label{fig:SPY_fig3}
\end{figure}

\clearpage
\subsection{Electricity Demand Comparison}\label{app:energy_comparison}

This appendix provides full experimental details and supplementary figures for the energy demand case study summarized in the main text (Sec.~\ref{sec:real-world}, Fig.~\ref{fig:energy_main}).

\textbf{Data.} We use the UCI Individual Household Electric Power Consumption dataset (IHEPC) \citep{dua2019uci}, which contains measurements of global active power at one-minute resolution for a single French household. We aggregate to hourly means, retaining only days with all $T = 24$ hours
observed, yielding approximately $1{,}400$ complete days. The target trajectory is the hourly mean active power (kW). Additionally, we use the GB National Electricity System Operator historic demand dataset (NESO) \citep{neso2024demand}, which contains half-hourly settlement period demand figures (MW) for Great Britain. We use five years of data (2018--2022), retaining only days with all $T = 48$ half-hourly periods observed, yielding approximately $1{,}800$ complete days. The target trajectory is the half-hourly national demand in megawatts (MW), using the ND (national demand) column, falling back to TSD (transmission system demand) if unavailable. For both datasets the cross-day mean trajectory (the diurnal pattern) is estimated on the training split and subtracted from all splits, so the model is trained on day-to-day deviations from the typical daily profile.

\textbf{Features.} Six features are constructed for IHEPC and seven for NESO, all from information available at the start of the day:

\begin{center}
\small
\begin{tabularx}{\textwidth}{lXcc}
\toprule
Feature & Description & IHEPC & NESO \\
\midrule
\texttt{day\_of\_week}    & Integer $\in\{0,\ldots,6\}$        & \checkmark & \checkmark \\
\texttt{is\_weekend}      & Binary: 1 if Saturday or Sunday    & \checkmark & \checkmark \\
\texttt{month}            & Integer $\in\{1,\ldots,12\}$       & \checkmark & \checkmark \\
\texttt{season}           & Integer $\in\{1,\ldots,4\}$ (meteorological)
& \checkmark & \checkmark \\
\texttt{lag\_daily\_mean} & Mean power/demand of the prior day & \checkmark & \checkmark \\
\texttt{lag\_morning}     & Mean power/demand during prior-day AM peak
& \checkmark & \checkmark \\
\texttt{lag\_evening}     & Mean power/demand during prior-day PM peak
&            & \checkmark \\
\bottomrule
\end{tabularx}
\end{center}

The AM and PM peak windows are defined as 06:00--09:00 and 17:00--21:00 for IHEPC (hours 6--9 and 17--21) and 06:00--09:00 and 17:00--20:30 for NESO (half-hourly periods 12--18 and 34--41). Season is derived from month: winter $= \{12, 1, 2\}$, spring $= \{3, 4, 5\}$, summer $= \{6, 7, 8\}$, autumn $= \{9, 10, 11\}$.

\textbf{Model.} Both datasets use a chronological 80/20 train/test split.
 
\begin{center}
\begin{tabular}{llrr}
\toprule
Dataset & Split & \# Days & \% \\
\midrule
IHEPC & Train & $\approx 1{,}120$ & 80\% \\
      & Test  & $\approx 280$     & 20\% \\
\midrule
NESO  & Train & $\approx 1{,}440$ & 80\% \\
      & Test  & $\approx 360$     & 20\% \\
\bottomrule
\end{tabular}
\end{center}

Both datasets use the same multivariate random forest architecture as described in App.~\ref{app:intraday_spy_volatility}, with identical hyperparameters (300 trees, \texttt{max\_features="sqrt"}, \texttt{random\_state=42}, \texttt{n\_jobs=-1}). Each model predicts the full diurnal-adjusted trajectory jointly. The model achieves trajectory-level $R^2 = 0.191$ for IHEPC and $R^2 = 0.738$ for NESO, computed as $1 - \mathrm{SS}_\mathrm{res}/\mathrm{SS}_\mathrm{tot}$ over all test-split predictions jointly. The large difference reflects the fundamentally different nature of the two settings: NESO national demand is strongly driven by calendar and weather seasonality, which is well-captured by the available features, whereas single-household IHEPC demand is far more idiosyncratic, with individual behavior and appliance usage introducing substantial unexplained variability that no set of day-level features can capture. Both values are sufficient to produce stable and interpretable attributions.

\textbf{Cooperative game set-up.} The set-up follows App.~\ref{app:intraday_spy_volatility} in all respects except the kernel choice. The local prediction input behavior uses $n_{\mathrm{sample}} = 150$ background draws per coalition; PDP plots aggregate per-instance local effects over $N_{\mathrm{PDP}} = 120$ profiles drawn by stratified sampling from the training set. Global sensitivity and risk input behaviors are estimated by nested Monte Carlo with $n_{\mathrm{outer}} = n_{\mathrm{inner}} = 100$, using the same definitions as in App.~\ref{app:intraday_spy_volatility}. With $p = 6$ (IHEPC) and $p = 7$ (NESO) features, there are $2^6 = 64$ and $2^7 = 128$ coalitions per game instance respectively.

\textbf{Kernels.} Two kernel choices are compared throughout.
 
\begin{itemize}
  \item \textbf{Identity kernel} $K = I$.  Treats each time step independently; effects reduce to pointwise attributions.
  \item \textbf{Correlation kernel.}  The empirical Pearson correlation matrix of the raw (non-diurnal-adjusted) training trajectories,
    \[
      K_{st} = \frac{\mathrm{Cov}(Y_s, Y_t)}
                    {\sqrt{\mathrm{Var}(Y_s)\,\mathrm{Var}(Y_t)}},
      \qquad s, t \in [T],
    \]
    clipped to $[-1, 1]$.  This kernel encodes the empirical temporal dependence of demand: strongly correlated time steps receive similar attributions, so that a feature affecting the entire day is credited for its full coherent contribution rather than appearing as isolated
    hourly effects. Both kernels are row-normalized before application.
\end{itemize}

\textbf{Profile selection.} For the local explanations two profiles are studied, one per dataset. The IHEPC typical weekday profile is the median instance among 814 days satisfying \texttt{is\_weekend}~$= 0$ and \texttt{day\_of\_week}~$\in \{1,2,3,4\}$, with feature values \texttt{day\_of\_week}~$= 4$ (Friday), \texttt{month}~$= 11$ (November), \texttt{season}~$= 4$ (autumn), \texttt{lag\_daily\_mean}~$= 1.640$~kW, and \texttt{lag\_morning}~$= 2.324$~kW. The NESO winter weekday profile is the median instance among 322 days satisfying \texttt{is\_weekend}~$= 0$ and \texttt{season}~$= 1$ (winter), with feature values \texttt{day\_of\_week}~$= 0$ (Monday), \texttt{month}~$= 12$ (December), \texttt{lag\_daily\_mean}~$= 33{,}431$~MW, \texttt{lag\_morning}~$= 29{,}082$~MW, and \texttt{lag\_evening}~$= 41{,}313$~MW.

\textbf{Supplementary figures.} Figs.~\ref{fig:energy_fig1_ihepc}--\ref{fig:energy_fig4_neso} show the full set of results.  Color coding: \texttt{day\_of\_week} (blue), \texttt{is\_weekend} (orange), \texttt{month} (green), \texttt{season} (red), \texttt{lag\_daily\_mean} (purple), \texttt{lag\_morning} (brown), \texttt{lag\_evening} (pink, NESO only). Solid lines correspond to the identity kernel; dashed lines to the correlation kernel.  Blue and red shading marks AM and PM peak windows.
 
Figs.~\ref{fig:energy_fig1_ihepc}--\ref{fig:energy_fig1_neso} show global sensitivity and risk effects under both kernels. For IHEPC, \texttt{lag\_daily\_mean}, \texttt{lag\_morning}, and \texttt{month} dominate both sensitivity and risk across all effect types. Under the identity kernel their effects show separate morning and evening peaks, while the correlation kernel merges these into a single elevated region spanning both peaks, reflecting the block-diagonal covariance structure. Time-aggregated pure, partial and full time-aggregated effects diverge substantially, suggesting that forecast variance and error are sensitive to feature interactions. For NESO, \texttt{month} and \texttt{season} dominate, with substantially larger magnitudes than IHEPC consistent with national-level demand being driven by weather and calendar seasonality. Under the identity kernel effects
vary across the day, but the correlation kernel collapses them to near-flat trajectories, reflecting the near-uniform covariance structure of the NESO dataset. The gap between time-aggregated pure and full effects for both risk and sensitivity is large for \texttt{season} and \texttt{month} in particular, indicating strong, focused higher-order interactions for these two features in particular at the national grid level, while other features participate less in interactions.
 
\begin{figure}[ht]
  \centering
  \includegraphics[width=\textwidth]{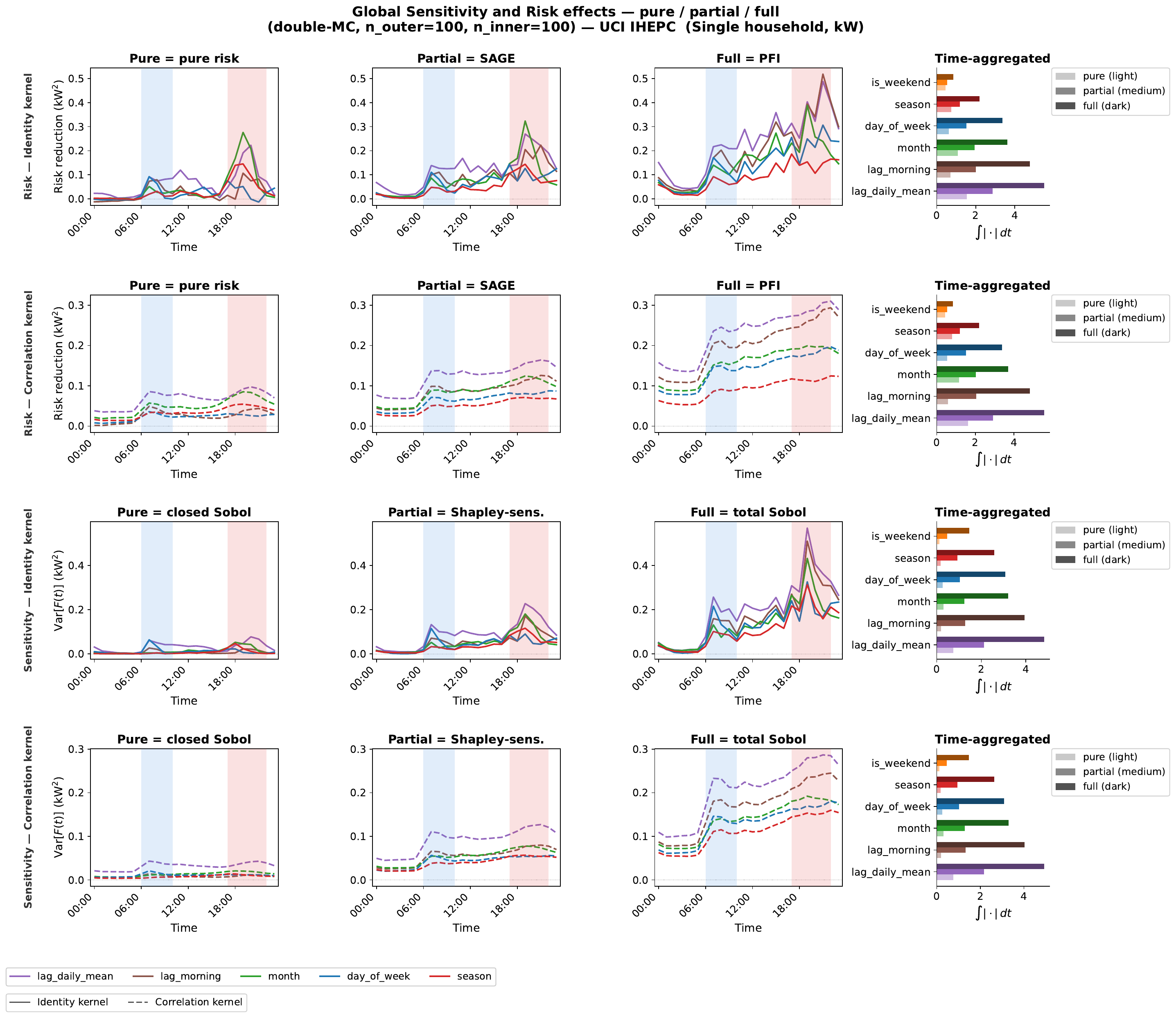}
  \caption{\textbf{Global sensitivity and risk effects --- IHEPC (single household, kW).} Effects are computed by nested Monte Carlo ($n_{\mathrm{outer}} = n_{\mathrm{inner}} = 100$) directly from the global sensitivity and risk games. (\emph{rows~1--2}) Risk game (loss reduction, kW$^2$) under the identity kernel (\emph{row~1}) and correlation kernel (\emph{row~2}); (\emph{rows~3--4}) sensitivity game ($\mathrm{Var}[F(t)]$, kW$^2$) under the identity (\emph{row~3}) and correlation (\emph{row~4}) kernels. (\emph{cols.~1--3}) Pure (closed Sobol / pure risk), partial (Shapley-sensitivity / SAGE), full (total Sobol / PFI); (\emph{col.~4}) time-aggregated bar chart. Blue and red shading marks AM and PM peak windows.}
  \label{fig:energy_fig1_ihepc}
\end{figure}
 
\begin{figure}[ht]
  \centering
  \includegraphics[width=\textwidth]{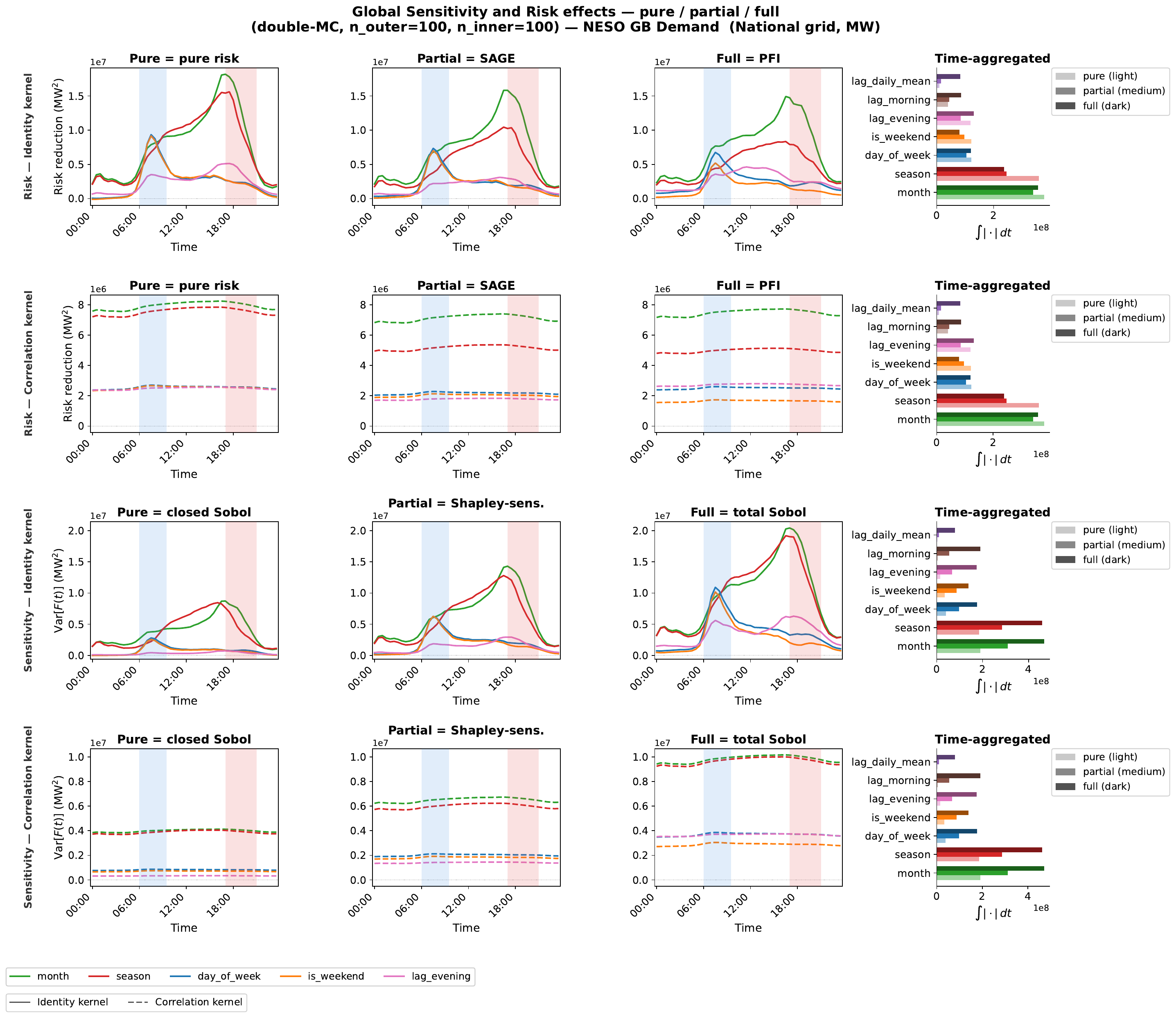}
    \caption{\textbf{Global sensitivity and risk effects --- NESO (national grid, MW).} Effects are computed by nested Monte Carlo ($n_{\mathrm{outer}} = n_{\mathrm{inner}} = 100$) directly from the global sensitivity and risk games. (\emph{rows~1--2}) Risk game (loss reduction, MW$^2 \times 10^7$) under the identity kernel (\emph{row~1}) and correlation kernel (\emph{row~2}); (\emph{rows~3--4}) sensitivity game ($\mathrm{Var}[F(t)]$, MW$^2 \times 10^7$) under the identity (\emph{row~3}) and correlation (\emph{row~4}) kernels. (\emph{Cols.~1--3}) Pure (closed Sobol / pure risk), partial (Shapley-sensitivity / SAGE), full (total Sobol / PFI); \emph{col.~4}: time-aggregated bar chart. Blue and red shading marks AM and PM peak windows.}
  \label{fig:energy_fig1_neso}
\end{figure}

\textbf{Sensitivity network plots.} Fig.~\ref{fig:energy_sens_networks} summarizes the global sensitivity input behavior as a network: node size encodes the time-aggregated attribution of each feature, edge width encodes pairwise interaction strength, and edge color encodes the sign of the underlying Möbius coefficient (teal for positive, red for negative). Three columns show the pure, partial, and full effects respectively. A salient pattern is that negative edges (red) appear almost exclusively in the pure-effect column and largely vanish in the partial and full columns. This is expected. Pure effects are raw Möbius coefficients, which under exact Sobol orthogonality (independent features) equal the non-negative variance contributions, but pick up signed contributions whenever features are empirically dependent or when sampling noise is present. Partial effects (Shapley sensitivity) and full effects (total Sobol) are non-negative weighted sums of Möbius coefficients at all coalitions containing the feature in question, and these aggregations average out the signed pure-effect components. In our setting, the visible negative edges in the IHEPC pure column connect calendar features to each other and lag features, and in the NESO pure column they connect lag features to one another and season and month, both consistent with the known empirical correlation between these features.

\begin{figure}[ht]
  \centering
  \includegraphics[width=\textwidth]{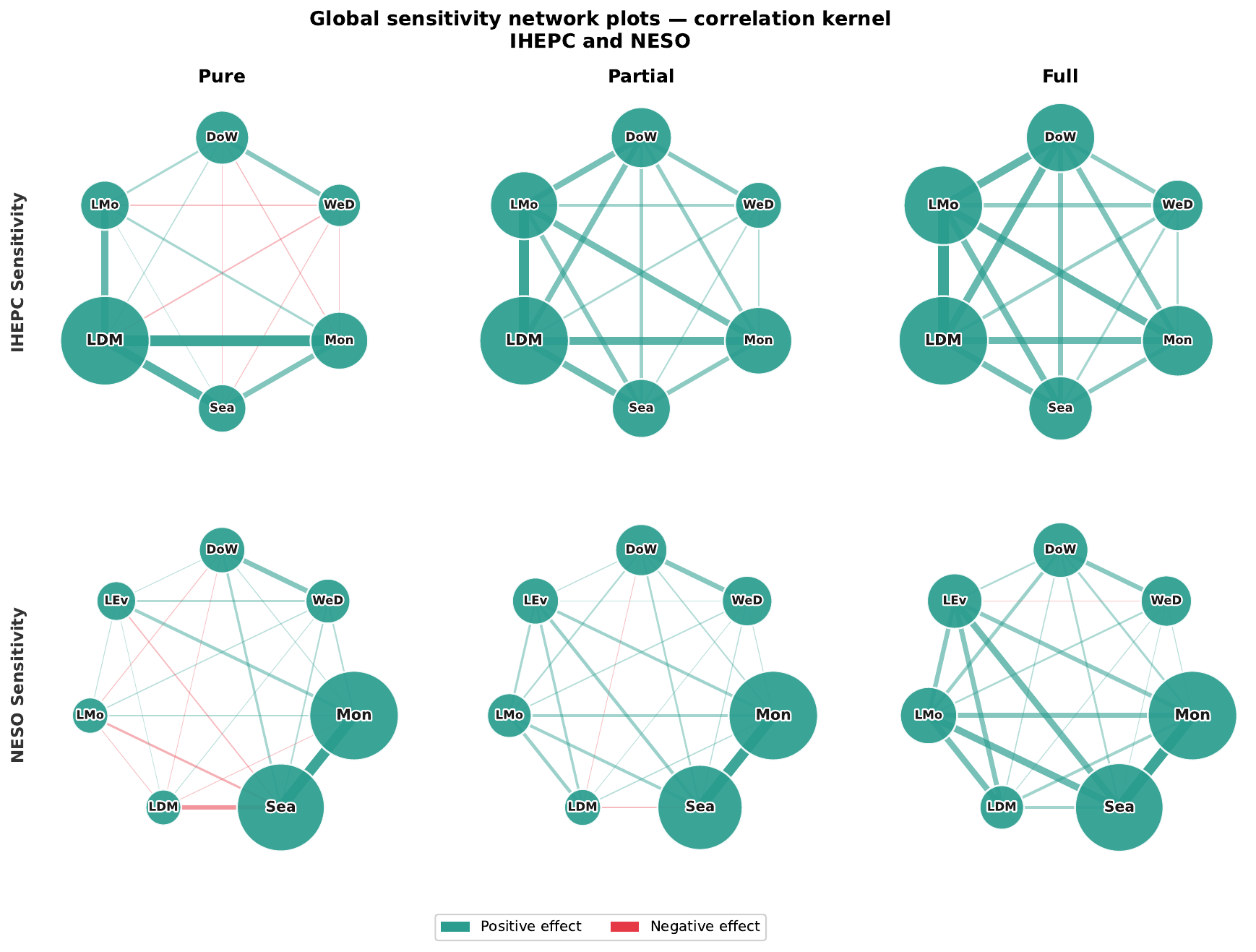}
  \caption{\textbf{Global sensitivity network plots --- correlation kernel.} (\emph{row~1}) IHEPC; (\emph{row~2}) NESO; (\emph{cols.~1--3}) pure, partial, and full effects. Node size encodes time-aggregated feature importance; edge width encodes pairwise interaction strength; edge color encodes the sign of the interaction (teal for positive, red for negative). Negative edges are concentrated in the pure-effect column and largely vanish after Shapley (partial) or total-Sobol (full) aggregation, which are non-negative weighted sums of Möbius coefficients. Feature abbreviations: \texttt{DoW} = \texttt{day\_of\_week}, \texttt{WeD} = \texttt{is\_weekend}, \texttt{Mon} = \texttt{month}, \texttt{Sea} = season, \texttt{LDM} = \texttt{lag\_daily\_mean}, \texttt{LMo} = \texttt{lag\_morning}, \texttt{LEv} = \texttt{lag\_evening} (NESO only).}
  \label{fig:energy_sens_networks}
\end{figure}

Figs.~\ref{fig:energy_local_pred_ihepc}--\ref{fig:energy_local_pred_neso} show local prediction effects for the selected profiles. For the IHEPC typical weekday, several features contribute comparably (\texttt{month}, \texttt{lag\_daily\_mean}, \texttt{day\_of\_week}, \texttt{lag\_morning}), with no single dominant driver. Under the identity kernel, effects are concentrated in the morning and evening peak windows, while the correlation kernel smooths these into broader positive contributions throughout the day with attenuated peaks. For the NESO winter weekday, contributions are similarly distributed across multiple features (notably \texttt{season}, \texttt{day\_of\_week}, \texttt{lag\_evening}, and \texttt{month}), with the correlation kernel producing roughly constant positive effects across the day, consistent with winter demand being uniformly elevated relative to the annual mean. The correlation kernel substantially reduces the apparent effect magnitude (note the compressed $y$-axis), reflecting the fact that coherent daily shifts are re-expressed as a single weighted attribution rather than summed over independent time steps.

\begin{figure}[ht]
  \centering
  \includegraphics[width=\textwidth]{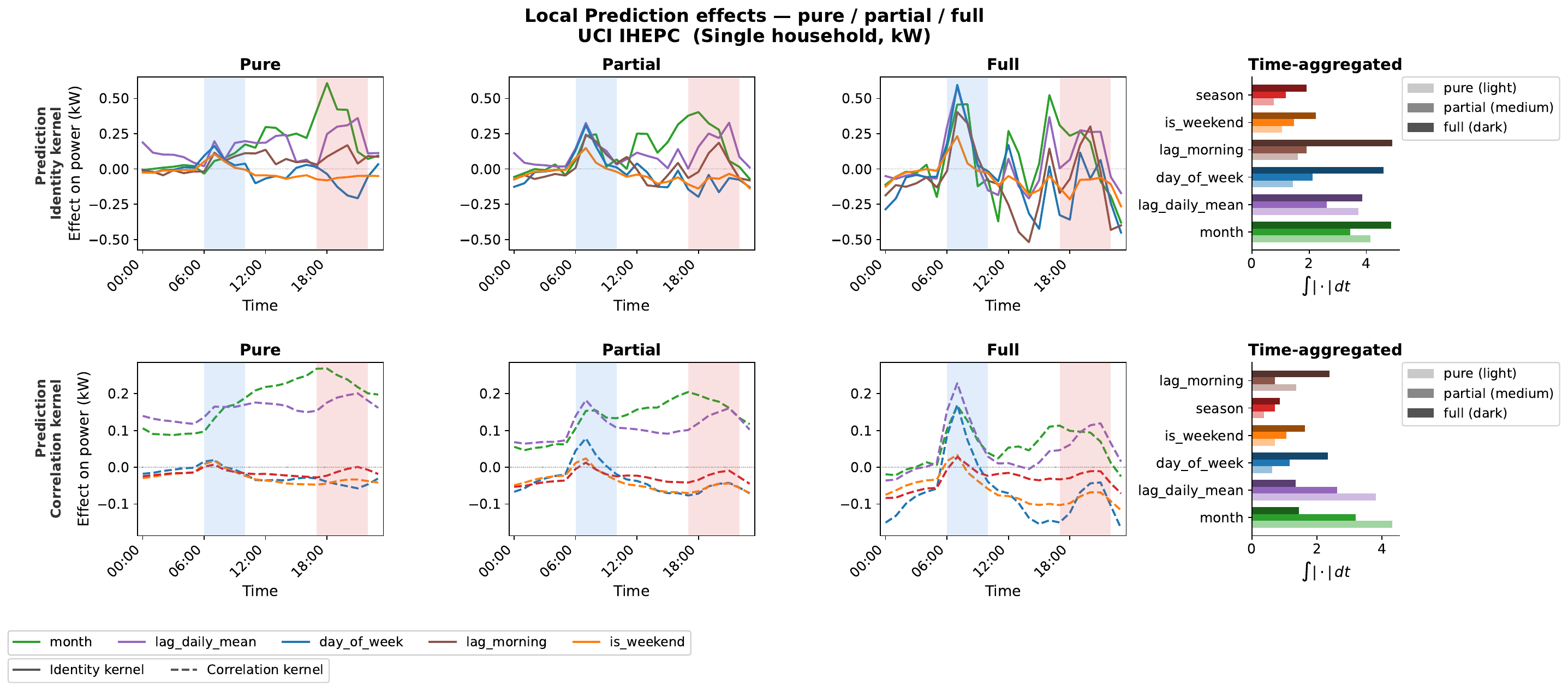}
  \caption{\textbf{Local prediction effects --- IHEPC typical weekday.} Instance: \texttt{day\_of\_week} $= 4$, \texttt{month} $= 11$, \texttt{lag\_daily\_mean} $= 1.640$\,kW, selected as the median among 814 days with \texttt{is\_weekend} $= 0$ and \texttt{day\_of\_week} $\in \{1,2,3,4\}$. (\emph{row~1}) Identity kernel; (\emph{row~2}) correlation kernel; (\emph{cols.~1--3}) pure, partial, full; (\emph{col.~4}) time-aggregated bar chart.}
  \label{fig:energy_local_pred_ihepc}
\end{figure}

\begin{figure}[ht]
  \centering
  \includegraphics[width=\textwidth]{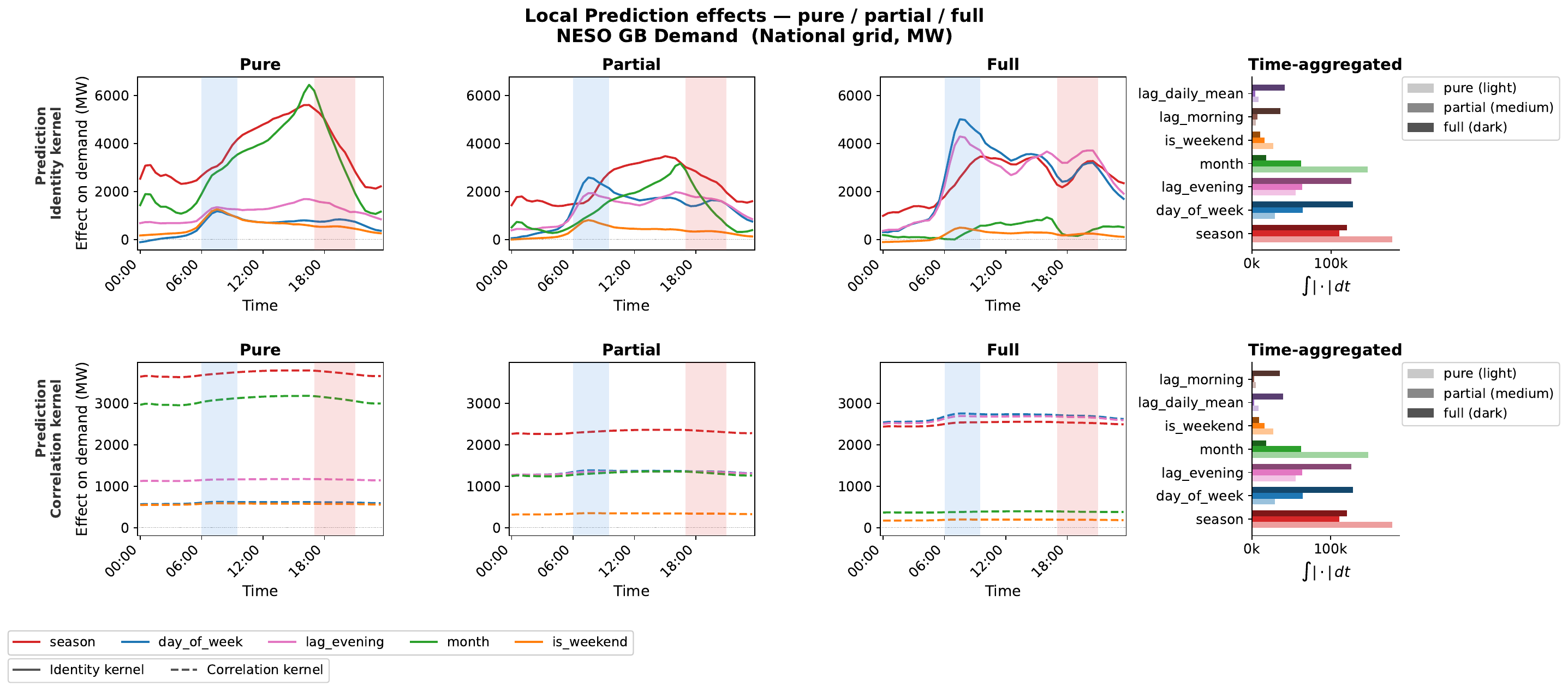}
  \caption{\textbf{Local prediction effects --- NESO winter weekday.} Instance: \texttt{day\_of\_week} $= 0$, \texttt{month} $= 12$, \texttt{lag\_daily\_mean} $= 33{,}431$\,MW, selected as the median among 322 days with \texttt{is\_weekend} $= 0$ and \texttt{season} $= 1$ (winter). (\emph{row~1}) Identity kernel; (\emph{row~2}) correlation kernel. (\emph{cols.~1--3}) Pure, partial, full; (\emph{col.~4}) time-aggregated bar chart.}
  \label{fig:energy_local_pred_neso}
\end{figure}

Figs.~\ref{fig:energy_fig4_ihepc}--\ref{fig:energy_fig4_neso} show local pairwise interaction effects. For IHEPC the depicted pairs, which are the top-5 pairs by time-aggregated importance, frequently involve lag features interacting with calendar features (e.g., \texttt{month} $\times$ \texttt{lag\_morning}, \texttt{day\_of\_week} $\times$ \texttt{lag\_morning}), suggesting that the strength of autoregressive momentum depends on the time of year or day. Under the identity kernel interaction trajectories are noisy, but exhibit morning and evening peaks; the correlation kernel smooths these into coherent daily profiles. For NESO the \texttt{month} $\times$ \texttt{season} pair dominates by a wide margin, reflecting the near-collinearity of these two features and their joint encoding of seasonal demand patterns. Under the correlation kernel all interaction trajectories collapse to flat lines, consistent with the uniform covariance structure.
 
\begin{figure}[p]
  \centering
  \includegraphics[width=\textwidth]{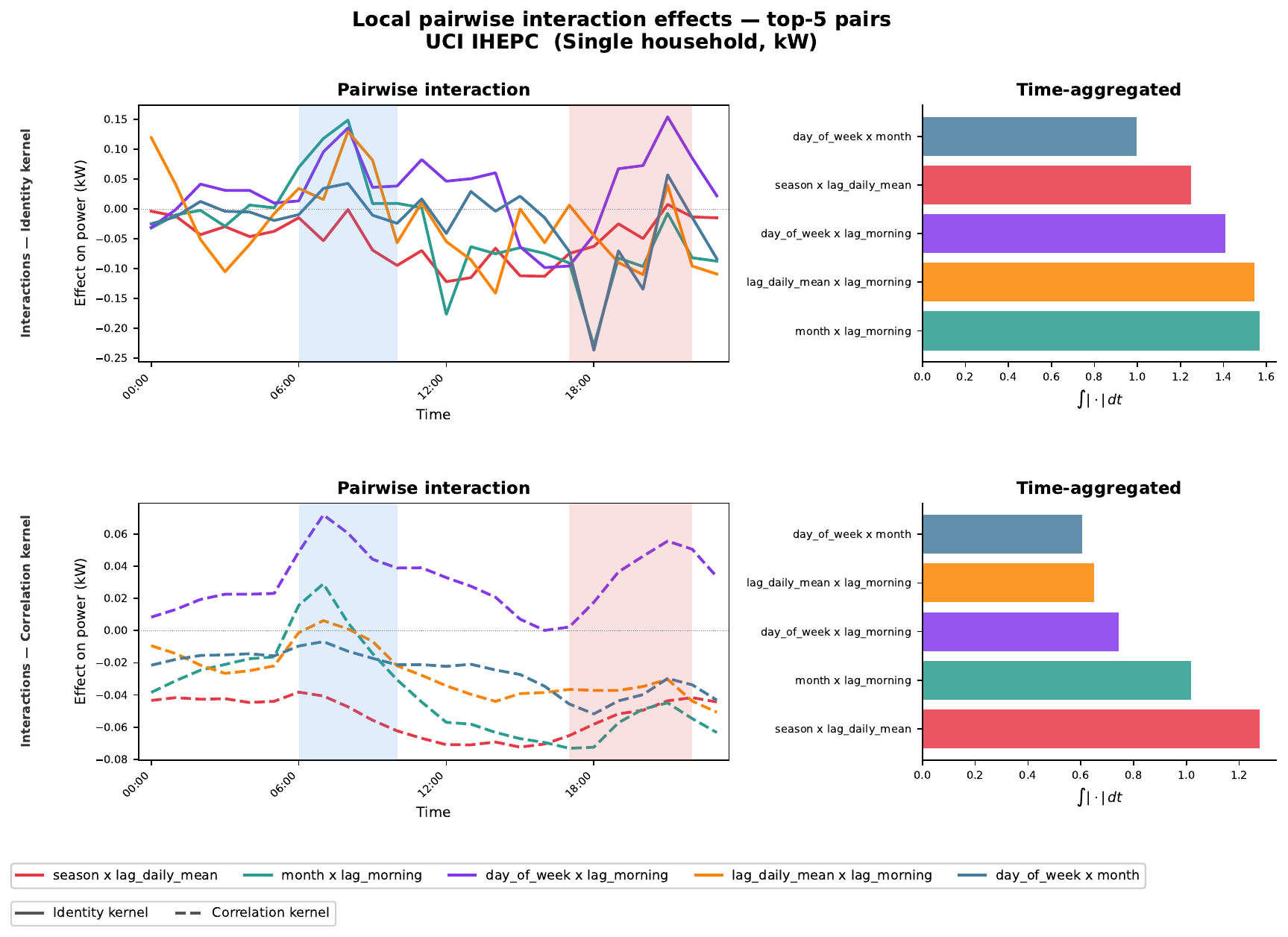}
  \caption{%
    \textbf{Local pairwise interaction effects --- top-5 pairs.} Instance: \texttt{day\_of\_week} $= 4$, \texttt{month} $= 11$, \texttt{lag\_daily\_mean} $= 1.640$\,kW, selected as the median among 814 days with \texttt{is\_weekend} $= 0$ and \texttt{day\_of\_week} $\in \{1,2,3,4\}$.  Top-5 pairs are selected by time-integrated correlation-kernel importance. (\emph{row~1}) Identity kernel; (\emph{row~2}) correlation kernel (dashed lines). (\emph{col.~1}) Interaction trajectories over the 24-hour day; (\emph{col.~2}) time-aggregated bar chart. Blue and red shading marks AM and PM peak windows.
  }
  \label{fig:energy_fig4_ihepc}
\end{figure}
 
\begin{figure}[p]
  \centering
  \includegraphics[width=\textwidth]{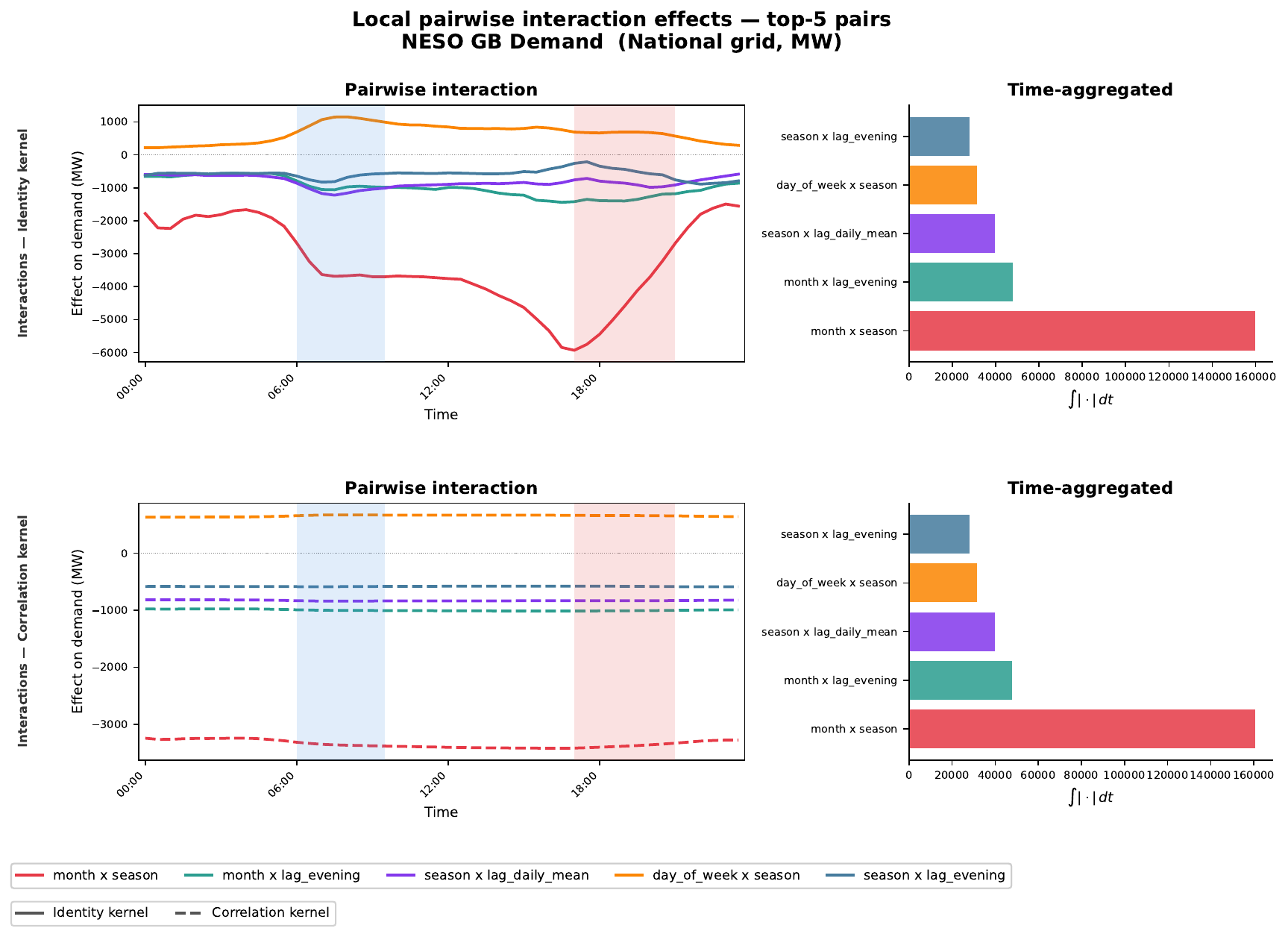}
  \caption{%
    \textbf{Local pairwise interaction effects --- top-5
    pairs.} Instance: \texttt{day\_of\_week} $= 0$, \texttt{month} $= 12$, \texttt{lag\_daily\_mean} $= 33{,}431$\,MW, selected as the median among 322 days with \texttt{is\_weekend} $= 0$ and \texttt{season} $= 1$ (winter). Top-5 pairs are ranked by time-integrated correlation-kernel importance.
    (\emph{row~1}) Identity kernel; (\emph{row~2}) correlation kernel (dashed lines). (\emph{col.~1}) Interaction trajectories over the 48 half-hourly periods; (\emph{col.~2}) time-aggregated bar chart. The \texttt{month}~$\times$~\texttt{season} pair dominates by a wide margin under both kernels. Blue and red shading marks AM and PM peak windows.
  }
  \label{fig:energy_fig4_neso}
\end{figure}

Figs.~\ref{fig:energy_fig3_ihepc}--\ref{fig:energy_fig3_neso} show global prediction effects in PDP style. For IHEPC the top-2 features by identity-kernel SHAP importance are \texttt{lag\_daily\_mean} and \texttt{month}, with \texttt{lag\_daily\_mean} showing a clear monotone relationship: higher prior demand predicts higher demand deviations. The correlation kernel preserves this monotone shape but tightens the effect and time point spread, indicating that the relationship is more coherent across the day rather than time-step-specific. \texttt{month} exhibits a U-shaped curve, capturing the January peak, summer trough, and second winter peak. For NESO the top-2 features are \texttt{month} and \texttt{lag\_evening}. \texttt{month} exhibits a U-shaped curve again. The correlation kernel suppresses time point variability and cleanly reveals these seasonal patterns for both features.
 
\begin{figure}[ht]
  \centering
  \includegraphics[width=\textwidth]{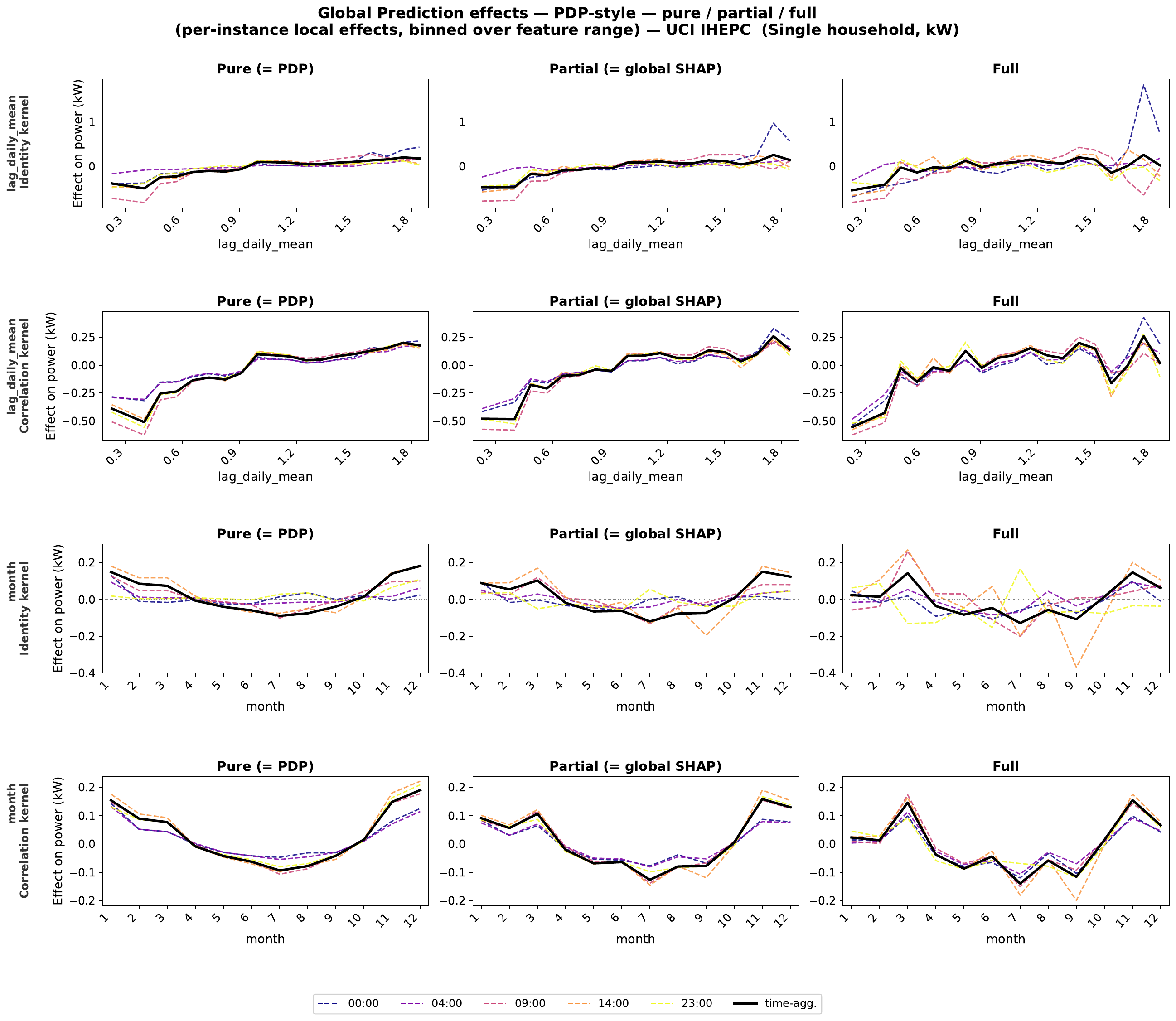}
  \caption{\textbf{Global prediction effects --- PDP-style --- IHEPC.} Per-instance pure, partial, and full local prediction effects, computed for $N_{\mathrm{PDP}} = 120$ profiles drawn by stratified sampling from the training set, are binned over the feature range ($N_{\mathrm{bins}} = 20$) for the two highest-importance features by identity-kernel SHAP importance. (\emph{rows~1--2}) \texttt{lag\_daily\_mean} under the identity (\emph{row~1}) and correlation (\emph{row~2}) kernels; (\emph{rows~3--4}) \texttt{month} under the identity (\emph{row~3}) and correlation (\emph{row~4}) kernels. Coloured dashed lines show mean effects at five selected time points (00:00, 04:00, 09:00, 14:00, 23:00); the solid black line is the time-aggregated mean.}
  \label{fig:energy_fig3_ihepc}
\end{figure}
 
\begin{figure}[ht]
  \centering
  \includegraphics[width=\textwidth]{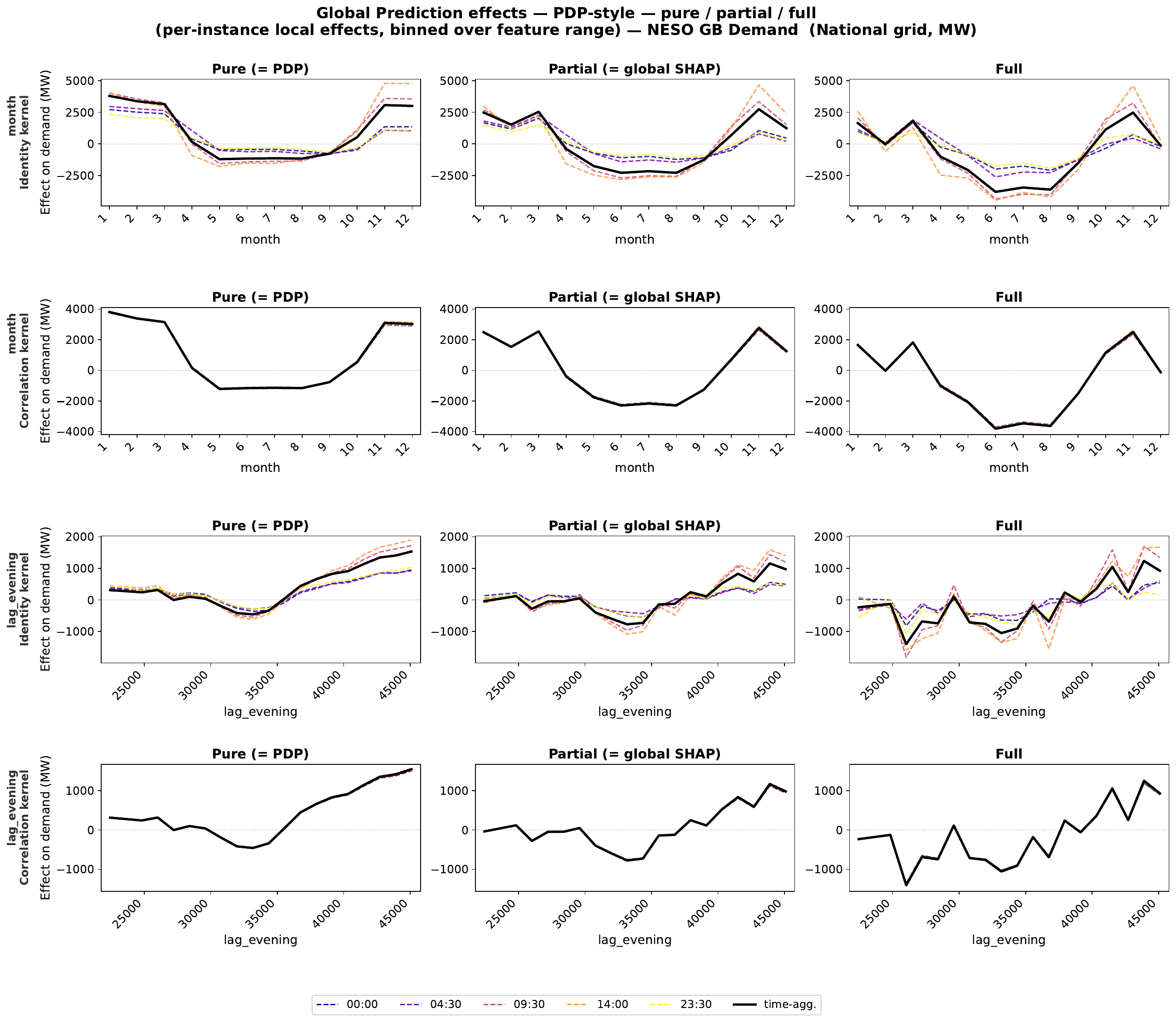}
  \caption{%
     \textbf{Global prediction effects --- PDP-style --- pure / partial / full - NESO.} Per-instance pure, partial, and full local prediction effects, computed for $N_{\mathrm{PDP}} = 120$ profiles drawn by stratified sampling from the training set, are binned over the feature range ($N_{\mathrm{bins}} = 20$) for the two highest-importance features by identity-kernel SHAP importance; (\emph{rows~1--2}) \texttt{month} under the identity (\emph{row~1}) and correlation (\emph{row~2}) kernels; (\emph{rows~3--4}) \texttt{lag\_evening} under the identity (\emph{row~3}) and correlation (\emph{row~4}) kernels. Coloured dashed lines show mean effects at five selected time points (00:00, 04:00, 09:00, 14:00, 23:00); the solid black line is the time-aggregated mean.
  }
  \label{fig:energy_fig3_neso}
\end{figure}

\clearpage
\subsection{Data Availability}\label{app:data_available}
Intraday five-minute OHLCV bar data for SPY were obtained from Polygon.io (\url{https://polygon.io}) covering January~2022 to April~2024. Access requires a Polygon.io subscription. The derived feature matrix and diurnal-adjusted volatility trajectories used in the experiments are available in the supplementary material. The IHEPC dataset is publicly available from the UCI Machine Learning Repository~\cite{dua2019uci}. GB historic demand data are publicly available from the National Energy System Operator data portal~\cite{neso2024demand}. All analysis code is available at \href{https://github.com/bips-hb/Hilbert_explanation_framework}{GitHub}.

\subsection{Computational Details}\label{app:computation}

\textbf{Computing environment.} All experiments were run on a 64-bit Linux platform running Ubuntu 22.04 LTS with two AMD EPYC Genoa 9534 64-core processors (128 cores, 256 threads total), 1.5~TB of RAM, and eight NVIDIA RTX 6000 Ada Generation GPUs (each with 48~GB memory). GPUs were not used: all model training and cooperative-game computations rely on CPU-only \texttt{scikit-learn} routines, parallelized over cores via \texttt{n\_jobs=-1}.

\textbf{Approximate wall-clock times.}
The synthetic ground-truth recovery experiment (Sec.~\ref{app:ground_truth}), comprising $30$ Monte Carlo runs $\times\,8$ sample sizes $\times\,4$ model classes plus the oracle, with exact $2^p$ Möbius computation at each setting, completes in approximately $30$--$45$~minutes. The SPY case study (Sec.~\ref{app:intraday_spy_volatility}) completes in approximately $15$~minutes for a first run, dominated by the global-game cache ($30$ instances $\times\,3$ game types). The IHEPC energy experiment (Sec.~\ref{app:energy_comparison}) takes approximately $20$~minutes for a first run, and the NESO experiment approximately $45$~minutes; NESO is slower due to its larger coalition space ($2^7=128$ vs.~$2^6=64$ for IHEPC) and its finer half-hourly resolution ($T=48$ vs.~$T=24$).

\section{Broader Impact and Ethics Statement}\label{app:ethics}

This work introduces a methodological framework for explaining machine learning models with time-dependent outputs. As a post-hoc, model-agnostic explanation method, it does not itself produce predictions or decisions, but rather aids interpretation of existing models. We see two main avenues of positive impact. First, in high-stakes domains where models with functional outputs are increasingly deployed — clinical trajectory prediction, energy demand forecasting, financial risk modeling — our framework supports more transparent model auditing by making temporal dependencies in feature attributions explicit, which scalar pointwise methods obscure. Second, by unifying existing approaches under a common operator view, the framework lowers the barrier to comparing explanations across methods and may reduce ad-hoc method selection in applied work.

Limitations and risks of explanation methods generally apply here. Explanations can create unwarranted confidence in model behavior, particularly when the underlying model is itself unreliable or trained on biased data; our framework does not address these upstream issues and should not be treated as a substitute for model validation, fairness analysis, or domain expertise. The kernel choice introduced by our framework adds an additional modeling decision: a misspecified kernel can produce attributions that are technically valid but misleading for the question at hand (for example, applying a symmetric kernel to a causal process, as we illustrate in our synthetic experiments). We address this through explicit kernel-selection guidance (Tab.~\ref{tab:kernel_guidance}), but practitioners retain responsibility for justifying their choices. As with all Shapley-based methods, our explanations rely on a reference distribution and inherit known concerns about marginal versus conditional masking under feature dependence; users should select the masking strategy appropriate to their interpretive goal.

Our experiments use publicly available datasets (UCI IHEPC, GB NESO historic demand) and a standard licensed financial dataset (Polygon.io intraday equity bars). None involve human subjects, personally identifiable information, or sensitive demographic attributes. The IHEPC dataset contains household-level electricity consumption from a single anonymized household; we use it solely as a benchmark for the methodological framework and draw no inferences about individuals. We do not foresee direct dual-use concerns: the framework is a diagnostic tool for existing models rather than an enabling capability for new ones.

%%%%%%%%%%%%%%%%%%%%%%%%%%%%%%%%%%%%%%%%%%%%%%%%%%%%%%%%%%%%

%\clearpage
%\input{checklist.tex}

\end{document}